%% file: neurips_2026.tex
\PassOptionsToPackage{bookmarks=true}{hyperref}
\documentclass[]{fairmeta}

\usepackage{helvet}

\DeclareTextFontCommand{\textbf}{\fontseries{bx}\selectfont}

\title{Unleashing the Potential of Diffusion Models\\[0.2em] for End-to-End Autonomous Driving}

\usepackage{url}
\usepackage{amsmath}
\usepackage{amsfonts}
\usepackage{amssymb}
\usepackage{amsthm}
\usepackage{mathtools}
\usepackage{nicefrac}
\usepackage{bbm}
\usepackage{algorithm}
\usepackage{algpseudocode}
\usepackage{float}
\usepackage{wrapfig}
\usepackage{listings}
\usepackage{diagbox}
\usepackage{capt-of}
\usepackage{makecell}
\usepackage{pifont}
\usepackage{multicol}
\usepackage{array}
\usepackage{thmtools}
\usepackage{enumitem}

\declaretheorem[name=Theorem, numberwithin=section]{theorem}

\declaretheorem[name=Remark, numberwithin=section, sibling=theorem]{remark}

\definecolor{mine}{RGB}{205, 232, 248}

\newcommand{\sd}[1]{\textcolor{gray}{#1}}
\definecolor{ego_car_yellow}{RGB}{255, 188, 50}
\definecolor{ego_pred_blue}{RGB}{45, 133, 240}
\definecolor{fut_gt_red}{RGB}{244, 67, 60}

\newcommand{\name}{HDP}

\author[1,*]{Yinan Zheng}
\author[1,*]{Tianyi Tan}
\author[2,*]{Bin Huang}
\author[2]{Enguang Liu}
\author[1]{Ruiming Liang}
\author[2]{Jianlin Zhang}
\author[2]{\\[0.15em]Jianwei Cui}
\author[2]{Guang Chen}
\author[2]{Kun Ma}
\author[2]{Hangjun Ye}
\author[2]{Long Chen}
\author[1]{\\[0.15em]Ya-Qin Zhang}

\author[1,\dagger]{Xianyuan Zhan}
\author[1,\dagger]{Jingjing Liu}

\affiliation[1]{Institute for AI Industry Research (AIR), Tsinghua University}
\affiliation[2]{Xiaomi EV}

\contribution[*]{Equal contribution}
\contribution[\dagger]{Corresponding authors}

\abstract{Diffusion models have become a popular choice for decision-making tasks in robotics, and more recently, are also being considered for solving autonomous driving tasks. However, their applications and evaluations in autonomous driving remain limited to simulation-based or laboratory settings. The full strength of diffusion models for large-scale, complex real-world settings, such as End-to-End Autonomous Driving (E2E AD), remains underexplored. In this study, we conducted a systematic and large-scale investigation to unleash the potential of the diffusion models as planners for E2E AD, based on a tremendous amount of real-vehicle data and road testing. Through comprehensive and carefully controlled studies, we identify key insights into the diffusion loss space, trajectory representation, and data scaling that significantly impact E2E planning performance. Moreover, we also provide an effective reinforcement learning post-training strategy to further enhance the safety and robustness of the learned planner. The resulting diffusion-based learning framework, \textit{\textbf{H}yper \textbf{D}iffusion \textbf{P}lanner} (\textit{\name{}}), is deployed on a real-vehicle platform and evaluated across 6 urban driving scenarios and 200 km of real-world testing, achieving a notable 10x performance improvement over the base model. Our work demonstrates that diffusion models, when properly designed and trained, can serve as effective and scalable E2E AD planners for complex, real-world autonomous driving tasks. Project website: \url{https://zhengyinan-air.github.io/Hyper-Diffusion-Planner/}.}

\metadata[Emails]{\texttt{zhengyn23@mails.tsinghua.edu.cn}, \texttt{zhanxianyuan@air.tsinghua.edu.cn}}

\begin{document}

\newcommand{\zyn}[1]{\textcolor{orange}{\small{\bf [zyn: #1]}}}
\newcommand{\tanty}[1]{\textcolor{blue}{\small{\bf [tanty: #1]}}}
\newcommand{\zh}[1]{\textcolor{cyan}{\small{\bf [zh: #1]}}}

\newtcolorbox{conclusionbox}{
  colback=blue!5,        % Background color
  colframe=blue!75!black, % Border color
  coltitle=black,        % Title text color
  fonttitle=\bfseries,   % Title font style
  boxrule=1pt,           % Border thickness
  arc=1mm,               % Rounded corners
  left=2mm,              % Left padding
  right=2mm,             % Right padding
  top=1mm,               % Top padding
  bottom=1mm,            % Bottom padding
}

\maketitle

\section{Introduction}

Diffusion models~\citep{ho2020denoising, sohldickstein2015deep} have demonstrated 
remarkable capabilities in image and video generation~\citep{betker2023improving, 
croitoru2023diffusion, liu2024sora, rombach2022high} and are becoming increasingly 
popular in robotics~\citep{black2024pi0visionlanguageactionflowmodel, 
chi2023diffusion, intelligence2025pi_, liurdt}. Their generative, multimodal 
nature also makes them a natural fit for autonomous driving 
(AD), where human driving behavior is inherently diverse and stochastic~\citep{zheng2025diffusionplanner, tan2025flow}. Among AD 
paradigms, End-to-End Autonomous Driving (E2E AD)~\citep{bojarski2016end, 
chen2023end, hu2023planning} has emerged as one of the most practically viable 
directions toward production-grade autonomy~\citep{TeslaOpenDay2023Video}, and 
a growing body of work has accordingly explored diffusion-based planners in 
this setting~\citep{diffusiondrive,  
li2025discrete, wang2026meanfuser}, with encouraging benchmark results.

However, what these benchmarks actually reveal about diffusion-based planning is less clear than it appears, for two reasons. First, they offer only a narrow view of real-world driving. As illustrated in Figure~\ref{fig:teaser}, mainstream open-source datasets~\citep{caesar2021nuplan} and simulation environments~\citep{Dauner2024NEURIPS} differ substantially from real-world deployment in data scale, trajectory diversity, and evaluation protocols. In particular, their open-loop and pseudo-closed-loop metrics are widely acknowledged to be poor proxies for true closed-loop performance on the road~\citep{Dauner2023CORL, li2024ego, zheng2025diffusionplanner}, so planners that excel in simulation may degrade substantially in real-world deployment, as our experiments confirm. Second, even on these benchmarks, top-performing planners often rely on 
rule-based post-processing~\citep{fan2018baidu} or incorporate strong 
assistive designs—such as pre-defined anchor trajectories~\citep{li2024hydra} 
or explicit goal conditions~\citep{albrecht2021interpretable, gu2021densetnt}—
to reduce the learning burden, leaving it unclear how much of the gain comes 
from the model itself rather than from these hand-crafted priors. Taken together, the true ceiling of diffusion models as AD planners in real world remains 
unknown, raising a critical question:
\textit{Are we fully exploiting the potential of diffusion models as AD 
planners?}

\begin{figure}[t]
\begin{center}
% \vspace{-14pt}
\includegraphics[width=1\textwidth]{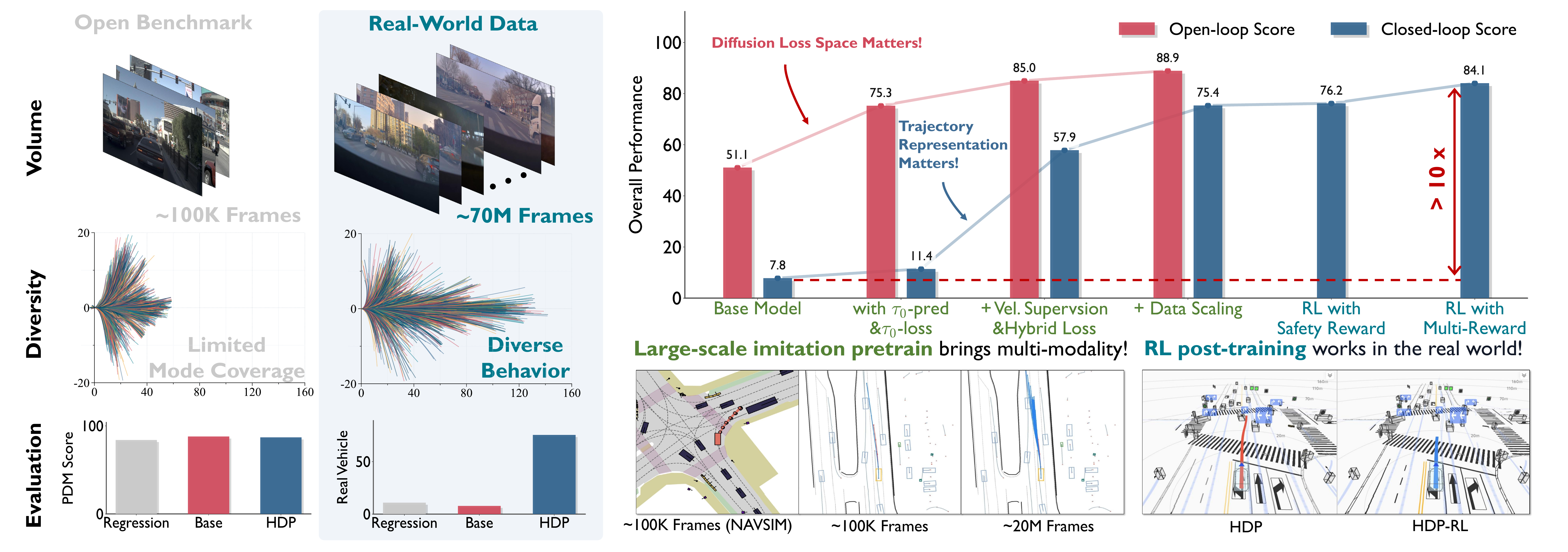}
\end{center}
% \vspace{-15pt}
\caption{\small Simulation-based benchmarks differ substantially from real-world deployment, leaving the true capability of diffusion-based planners largely unverified. Through a detailed investigation of the IL/RL training recipe, we unleash the potential of diffusion-based planning, achieving a 10x improvement in real world.}
\vspace{-15pt}
\label{fig:teaser}
\end{figure}

To answer this question, we conduct a systematic investigation of diffusion-based E2E AD, leveraging massive real-vehicle data and rigorous on-road testing. We adopt a vanilla diffusion-based planning head\citep{zheng2025diffusionplanner} for trajectory generation. We deliberately keep the diffusion head free of hand-crafted priors, so that any performance gain we observe can be attributed to diffusion modeling itself. Building on this base model, we identify a set of key design principles, each supported by real-vehicle evidence:

\begin{itemize}[leftmargin=*] 
\item \textbf{Diffusion loss space matters}. Our key insight stems from the observation that planning trajectory lives in a low-dimensional manifold, distinct from image generation. Consequently, we re-examine the diffusion loss space design~\citep{li2025back} and find that data ($\tau_0$)-prediction combined with diffusion loss directly supervised on data ($\tau_0$-loss)
% $\tau_0$-prediction combined with $\tau_0$-loss 
best captures the trajectory manifold, enabling better learning performance and high-quality trajectory generation.
% generation of high-quality trajectories.
\item \textbf{Trajectory representation matters}. In our $\tau_0$-prediction setup, predicting waypoints yields better spatial awareness, while predicting velocity gives smoother trajectories. Our model therefore outputs velocity but is supervised on both, achieving substantial closed-loop gains. Although this hybrid loss resembles 
auxiliary losses~\citep{bansal2018chauffeurnet, hu2023planning, jiang2023vad} 
routinely used in AD, we \textit{theoretically prove} that it leaves the 
diffusion training optimum unchanged, thereby exposing an often-overlooked 
pitfall: widely used auxiliary losses such as collision penalties or $L_{1}$-norm 
loss would distort it.

% us to leverage the merits of both waypoints and velocity supervision in practice.
% us to leverage both advantages in practice.
 \item \textbf{Emergence of data scaling}. We further show that, without extra auxiliary losses or hand-crafted 
priors, our diffusion framework effectively benefits from data scaling in real-world testing. Our model captures richer multimodal driving behaviors, an effect not observed when training on commonly used 
benchmarks~\citep{caesar2021nuplan, Dauner2024NEURIPS}, whose data are too small to expose it. The same scaling also yields stronger 
open-loop and closed-loop performance, making our framework a viable planner.
 % In contrast, commonly used AD benchmarks fail to capture such a scaling property and result in poor multimodal modeling capabilities.
\end{itemize}

While the imitation pretraining establishes a strong diffusion planner prior, it lacks explicit optimization for safety-critical scenarios. To close this gap, we 
further introduce Reinforcement Learning (RL) post-training that 
re-weights the diffusion objective with safety-aware advantage 
estimates~\citep{peng2019advantage, zheng2024safe}, aligning generated 
trajectories with safety constraints while preserving training stability. We \textit{theoretically prove} that this reweighting is naturally compatible with our hybrid loss, enabling a clean implementation. The same formulation extends to multi-reward settings and further advances the Pareto frontier, showing that \textbf{RL post-training is effective in real-world closed-loop deployment}.
% can naturally combine the hybrid loss introduced earlier, with a simple implementation.

Finally, we integrate the above innovations into a complete framework, \textit{\textbf{H}yper \textbf{D}iffusion \textbf{P}lanner} (\textit{\name{}}), and its RL post-trained variant \textit{\name{}-RL}. With only minimal smoothness post-processing, both are deployed on a real vehicle. We evaluate them across 6 urban driving scenarios over 200~km of road testing under comprehensive metrics: \textit{\name{}}~/~\textit{\name{}-RL} achieve a 10x improvement over the base model. We further analyze the characteristics of the framework and the impact of key components, showing that diffusion models, when properly designed and trained, can serve as effective and scalable planners for real-world autonomous driving.

\section{Preliminaries}
\label{sec:pre}

Our work focuses on the planning module of E2E AD systems, where the planner receives the latent representation $C$ from the scene encoder and generates a trajectory $\tau_0$ for downstream control systems. Diffusion models~\citep{sohldickstein2015deep} define a forward process that transforms the conditional trajectory distribution $q_0(\tau_0 \mid C)$ into a noised distribution $q_{t0}(\tau_t \mid \tau_0) = \mathcal{N}(\tau_t \mid \alpha_t \tau_0,\, \sigma_t^2 \mathbf{I})$ for $t \in [0,1]$, where $\alpha_t, \sigma_t$ define a pre-defined noise schedule. As $t \rightarrow 1$, this schedule ensures that the marginal distribution $q(\tau_1)$ approaches $\mathcal{N}(\tau_1 \mid 0, \mathbf{I})$. The reverse denoising process can be expressed as a diffusion ODE~\citep{song2021scorebased}:
\begin{equation}
\label{eq:reverse}
{\rm d}\tau_t = \left[f(t)\tau_t - \tfrac{1}{2}g^2(t)\,\nabla_{\tau_t}\log q_t(\tau_t)\right]{\rm d}t,
\end{equation}
where $f(t)=\tfrac{{\rm d}\log \alpha_t}{{\rm d}t}$ and $g^2(t)=\tfrac{{\rm d}\sigma_t^2}{{\rm d}t}-2\tfrac{{\rm d}\log \alpha_t}{{\rm d}t}\sigma_t^2$. A commonly used approach for learning diffusion models~\citep{ho2020denoising, rombach2022high} is to train a neural network $\epsilon_{\theta}(\tau_t, t, C)$ to fit the Gaussian noise $\epsilon$:
\begin{equation}
\label{eq:diffusion_loss}
\mathcal{L} = \mathbb{E}_{t,\, \tau_0,\, \tau_t,\, \epsilon}\,\|\epsilon_{\theta}(\tau_t, t, C) - \epsilon\|_2^2,
\end{equation}
where $t\sim \mathbb{U}(0,1), \tau_0 \sim q_0(\tau_0 | C),\tau_t \sim q_{t0}(\tau_t| \tau_0)$ and $\epsilon \sim \mathcal{N}(\tau_{1}\mid0,\mathbf{I})$. Then, we can estimate the score function $\nabla_{\tau_{t}}\log q_t(\tau_{t})$ in Eq.~(\ref{eq:reverse}) using $s_{\theta}(\tau_t, t, C) = - \epsilon_{\theta}(\tau_t, t, C)/\sigma_t$, and use an ODE solver to generate the data.

\section{Investigation Roadmap}

In this section, we first introduce the base model and evaluation metrics for assessing model performance. 
% With the base model and evaluation metrics ready, we will start our journey 
Subsequently, we will briefly outline our investigation roadmap aimed at fully unleashing the potential of diffusion models for E2E AD.

\subsection{Base Model}
\label{sec:basemodel}

\begin{wrapfigure}{r}{0.58\textwidth} 
\vspace{-15pt}
\includegraphics[width=0.56\textwidth]{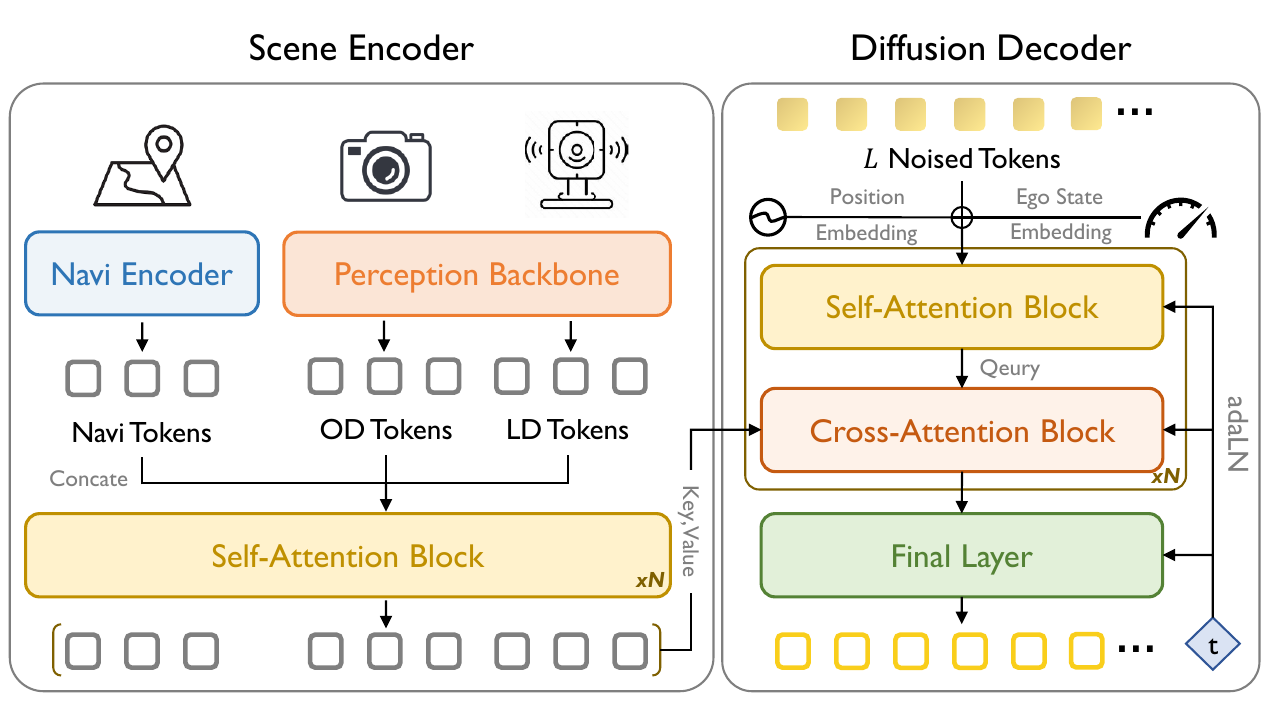}
\vspace{-2pt}
\caption{\small Model architecture of the base model.}
\label{fig:architecture}
\vspace{-10pt}
\end{wrapfigure}

\textbf{Scene Encoder.} 
Following standard E2E AD practice~\citep{chen2023end, hu2023planning}, we adopt a perception backbone that ingests multi-modal inputs (camera and LiDAR) and produces two sets of perception tokens---Object Detection (OD) and Lane Detection (LD). Together with Navi tokens encoding high-level navigation commands, they are concatenated and fused via self-attention into the latent condition $C := \mathrm{Enc}(s)$ consumed by the planner. More architecture and training details are deferred to Appendix~\ref{app:scene_encoder}.

\textbf{Diffusion Decoder.} 
We use a vanilla transformer-based diffusion model~\citep{peebles2023scalable} as the decoder. Conditioned on $C$ and the current ego state, it generates the trajectory $\tau_0 \in \mathbb{R}^{L \times 4}$, where each of the $L$ timesteps contains the ego-centric waypoint coordinates and the cosine/sine of the heading. As illustrated in Figure~\ref{fig:architecture}, the noised trajectory $\tau_t$ is first split and projected into $L$ tokens, with position and velocity embeddings added, fused via self-attention, and then attended to $C$ through cross-attention. The diffusion timestep $t$ is injected via adaptive layer normalization~\citep{peebles2023scalable}. After several such blocks, an MLP head~\citep{liurdt, zheng2025diffusionplanner} produces the prediction, and the model is trained using Eq.~(\ref{eq:diffusion_loss}).

\subsection{Evaluation Metrics}
\label{subsec:metrics}

We adopt two complementary evaluation protocols. \emph{Open-Loop} metrics evaluate trajectory quality via data replay, while \emph{Closed-Loop} metrics measure performance in real-vehicle testing on a fixed route; full definitions and hyperparameters are in Appendix~\ref{app:metrics}.

\textbf{Open-Loop.} Following nuPlan~\citep{caesar2021nuplan}, we report \emph{Average Displacement Error} (ADE), \emph{Final Displacement Error} (FDE), \emph{Comfort}, and \emph{Collision Rate} (CR), and aggregate them into the \emph{Open-Loop Score} $(1-CR)\sum_{m\in\{\text{ADE, FDE, Comfort}\}} \omega_m S_m$. To probe multimodality, we additionally report \emph{Trajectory Divergence}, the average pairwise Euclidean distance among samples drawn for the same scene.

\textbf{Closed-Loop.} For real vehicle testing, we use a fixed route to conduct controlled experiments. We log the \textit{Success Rate} on six common driving scenarios, with 
any takeover counted as a failure: starting maneuvers, car-following with stopping, navigational lane changes, yielding to VRUs, yielding to cross traffic at intersections, and left and right turns. Per-scenario rates are aggregated by 
frequency-weighted averaging for a more representative evaluation. We 
further compute a \textit{Stability Score} as the mean of 
\textit{Centering Performance} and \textit{Speed Compliance}, capturing 
off-center driving or improper speed even without 
takeover. The overall \textit{Closed-Loop Score} is the mean of the two.

\subsection{Roadmap Overview}

With the base model and evaluation metrics ready, we start our journey to unleash the potential of diffusion models for E2E AD. As the AD planning task is obviously different from image generation tasks, with its output trajectories residing on a relatively low-dimensional manifold, need to satisfy hard constraints like collision avoidance, and is evaluated in a closed-loop setting that easily suffers from error accumulation. To address these challenges, we structure our exploration into two phases: 1) \textbf{Imitation Learning Pre-Training}, examining how diffusion loss and trajectory representation shape planning quality and validating closed-loop data scaling; and 2) \textbf{Reinforcement Learning Post-Training}, where a compatible RL algorithm enables stable, efficient enhancement of the pretrained model.

% With the base model and evaluation metrics ready, we start our journey to unleash the potential of diffusion models for E2E AD. As the AD planning task is obviously different from image generation tasks, with its output trajectories residing on a relatively low-dimensional manifold, need to satisfy hard constraints like collision avoidance, and is evaluated in a closed-loop setting that easily suffers from error accumulation. Hence, very different design considerations could apply. 
% Our starting point is that the AD planning task is obviously different from image generation tasks. This is because planning trajectories reside on a relatively low-dimensional manifold, need to satisfy hard constraints like collision avoidance, and are evaluated in a closed-loop setting, which will amplify the error. 
% To address these challenges, we structure our exploration into two separate phases: 1) \textbf{Imitation Learning Pre-Training}, where we study how diffusion loss and trajectory representation influence planning trajectory quality, and validate data scaling in a closed-loop setting; and 2) \textbf{Reinforcement Learning Post-Training}, where we use RL to further enhance the safety of the pre-trained model by developing a compatible RL algorithm for stable and efficient post-training.

\section{Imitation Learning Pre-training}

\subsection{Diffusion Loss Space}
\label{sec:diffusion_loss}

As the score function in the denoising process (Eq.~(\ref{eq:reverse})) is generally intractable, in practice, the diffusion model is typically trained to predict 
% A diffusion model predicts the score function in Eq.~(\ref{eq:reverse}). However, the marginal score function itself is intractable. In practice, the model is trained to predict
one of the three conditioned quantities: the noise $\epsilon$~\citep{ho2020denoising}, the flow velocity $v_t$~\citep{ho2022imagen}, or the clean data $\tau_0$~\citep{ramesh2022hierarchical}. These quantities are mutually convertible, allowing for various loss space designs (see Appendix~\ref{app:b} for more details). 
For instance, the diffusion model can be parameterized to output $\tau_0$, while being supervised with $\epsilon$-loss. 
However, models trained in different loss spaces can exhibit distinct learning dynamics~\citep{li2025back} and planning behaviors. To investigate the impact of loss space design on the planning task, we trained our model with all 9 prediction-loss combinations and conducted open-loop evaluations. The results are shown in Table~\ref{tab:overall_aggregated_score} and Figure~\ref{fig:training_curve}, \ref{fig:grid_openloop}. Most models achieved competitive performance (except $\epsilon$-pred with $\tau_0$- and $v$-loss), successfully capturing the expert policy in the training data while demonstrating multimodal generation capability.
% , with two distinct modes (\textit{proceed} or \textit{yield}) in trajectory predictions. 
In addition, among these models, the $\tau_0$-prediction model trained with $\tau_0$-loss stands out prominently. To study the reasons for this advantage, we provide further investigation from the following perspectives.

\textbf{Fast Convergence}. Figure~\ref{fig:training_curve} displays the aggregated scores of models at various training stages. While the model utilizing $\tau_0$-prediction converged rapidly with increased training steps, the other two approaches experienced notable instability. This disparity stems from differences in the inherent dimensionality of the target manifold~\citep{li2025back}. Because the trajectory $\tau_0$ resides in a low-dimensional manifold, the neural network can capture it more easily. Conversely, the $\epsilon$ and $v$ targets are supported on much higher-dimensional spaces and therefore require greater model capacity. Furthermore, $\tau_0$-loss works best for the $\tau_0$-prediction model compared with the other two choices.
% , which introduce no noise in its learning objective. 

\textbf{High Generation Quality}. Furthermore, we visualized the generated trajectories of different models, as shown in Figure~\ref{fig:grid_openloop}. Although most models generate trajectories that resemble the ground truth, the quality varies across different loss space designs. Some models (especially $\epsilon$-pred) generate trajectories with noticeable non-smoothness and irregular jitters, leading to abrupt changes in heading direction and velocity,
% The generations of most models exhibit irregular jitters of different extents, leading to abrupt changes in velocity and acceleration, 
while the $\tau_0$-prediction models generate trajectories with better kinematic coherence. This disparity likely stems from denoising dynamics during the final, low-noise steps~\citep{ning2025dctdiff}. Unlike $\epsilon$- and $v$-prediction models, which struggle to estimate faint noise signals and consequently generate high-frequency artifacts, data prediction demonstrates superior stability. By directly predicting the trajectory, it effectively suppresses noise to yield smoother, kinematically consistent trajectories. Another thing worth noting is that the $\epsilon$-prediction models trained with $\tau_0$- and $v$-loss suffered a complete breakdown. The failure of these two modes can be attributed to the extremely high variance of training objectives in which the noise target is scaled by $1 / \alpha_t$.
In conclusion, the $\tau_0$-prediction model with $\tau_0$-loss yields both fast convergence and high-quality generation,
% Therefore, we choose this design as the default for the following investigative experiments and discussions.
making it a suitable choice for further investigation.
% % the discussion below. 
Therefore, we choose this design as the default for the following investigative experiments and discussions.

\begin{table*}[t]
    \centering
    \captionsetup[table]{name=Table, labelsep=colon}
    \begin{minipage}{\linewidth}
        \centering
        \begin{minipage}[p]{0.53\linewidth}
            \centering
            \small
            \captionof{table}{\small Aggregated open-loop score. The models are trained for $2\times10^4$ steps in total, and the results are averaged over 3 evaluations. Averaged open-loop score in black and standard variance in \sd{gray}.} % 标题放在表格上方
            \label{tab:overall_aggregated_score}
            \resizebox{\linewidth}{!}{
                \begin{tabular}{l |c c c}
                    \toprule[1.pt]
                      & $\tau_0$-pred & $v$-pred & $\epsilon$-pred \\
                    \midrule
                    $\tau_0$-loss: $\mathbb{E}[||\tau_\theta - \tau_0||^2_2]$      & \colorbox{mine}{75.27} $\pm$ \sd{8.53}  & 35.64 $\pm$ \sd{0.97}  & 11.43 $\pm$ \sd{0.45}   \\
                    $v$-loss: $~\mathbb{E}[||v_{\theta;t} - v_t||^2_2]$          & 63.47 $\pm$ \sd{8.58}  & 53.91 $\pm$ \sd{0.87}   & 0.66  $\pm$ \sd{0.40} \\
                    $\epsilon$-loss: $~\mathbb{E}[||\epsilon_{\theta} - \epsilon||^2_2]$  & 63.78 $\pm$ \sd{8.59}  & 45.24 $\pm$ \sd{2.84}   & 51.07 $\pm$ \sd{4.67} \\
                    \bottomrule[1.pt]
                \end{tabular}
            }
            
            \vspace{0.5em}
            \includegraphics[width=\linewidth]{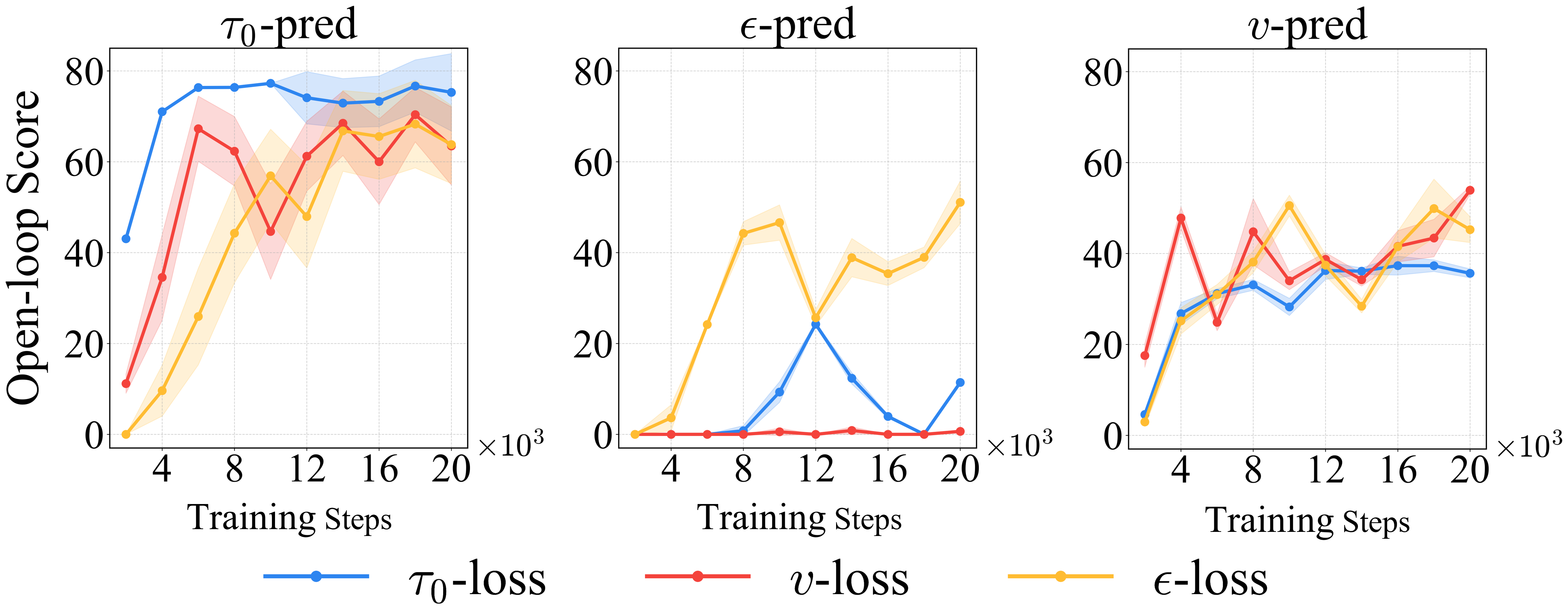}
            % \vspace{-1pt}
            \captionof{figure}{\small The learning curve of models with different loss designs. Results are averaged over three evaluations.}
            \label{fig:training_curve}
        \end{minipage}
        \hfill
        \begin{minipage}[p]{0.45\linewidth}
            \centering
            % 图片填满右侧 minipage 宽度
            \includegraphics[width=\linewidth]{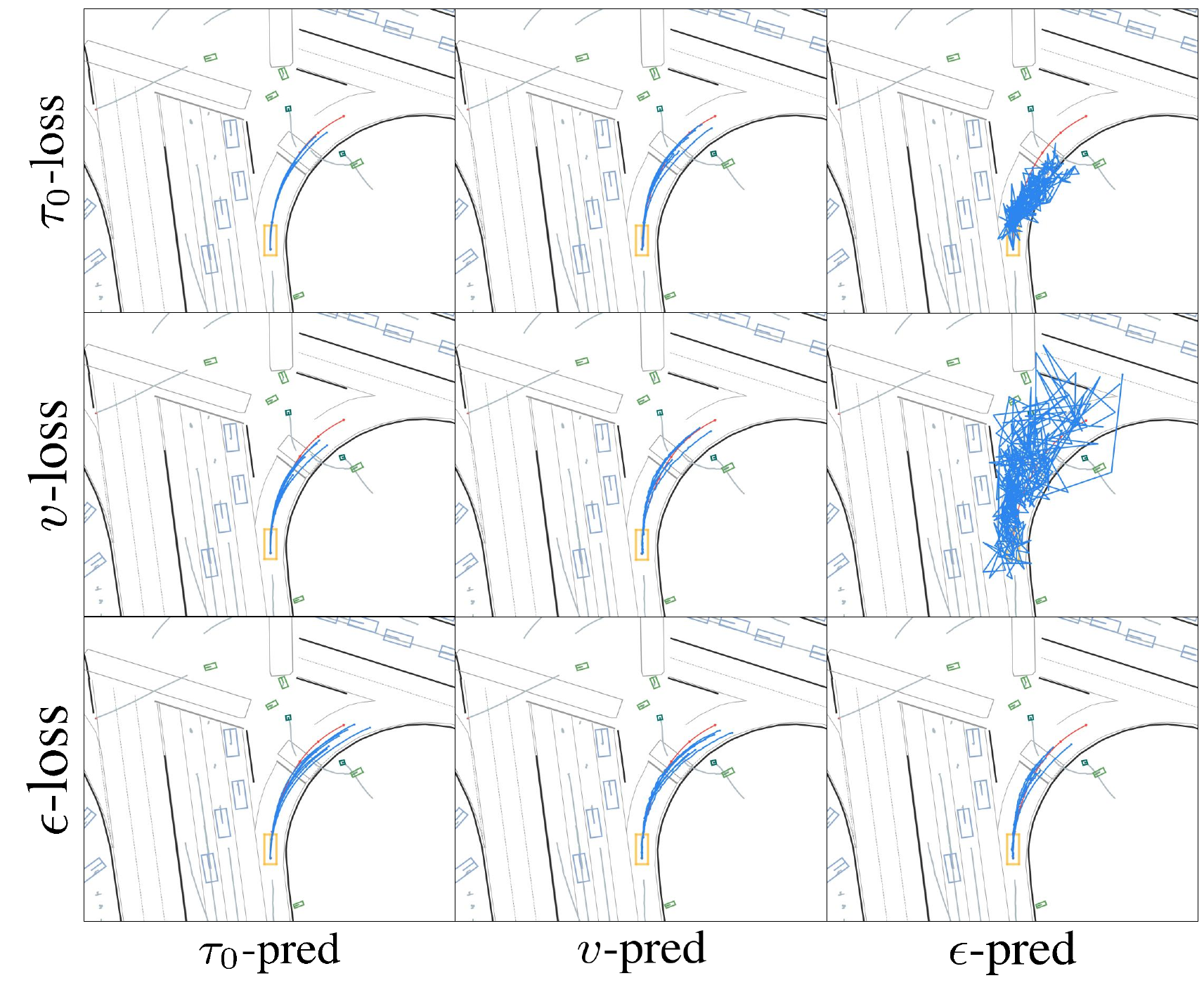}
            % 强制标记为 Figure
            % \vspace{0.3em}
            \captionof{figure}{\small The open-loop visualization of planning trajectories. 6 generations are plotted for each scene. Ego vehicle in \textcolor{ego_car_yellow}{yellow}, model predictions in \textcolor{ego_pred_blue}{blue} and ground truth trajectory in \textcolor{fut_gt_red}{red}.}
            \label{fig:grid_openloop}
        \end{minipage}
    \end{minipage}
    \vspace{-15pt}
\end{table*}

\vspace{-2pt}
\subsection{Trajectory Representation}
\label{sec:trajectoryrepresentation}
% \tanty{need rewrite}
% current logic
% 1. despite fig 3, wpt with good structure but poor smooth (comfort)
% 2. turn to velocity for comfort, but accumulating error -> highlight real vehicle problem
% 3. so propose hybrid loss: each timestep should consider error before
% 4. real vehicle improvement

\begin{wrapfigure}{r}{0.5\linewidth}
    \centering
    \vspace{-7pt}
    \includegraphics[width=\linewidth]{assets/vel_wpt_horizontal_combo.png}
    \caption{\small Comparison of prediction quality with waypoint and velocity representations.}
    \label{fig:vel_wpt_relative}
    \vspace{-7pt}
\end{wrapfigure}

In the previous section, we identified $\tau_0$-prediction with $\tau_0$-loss as a suitable diffusion loss space for planning tasks, which achieves much better learning and open-loop performance. However, when taking a finer-grained inspection on higher-order statistics of generated trajectories, we find that directly using trajectory waypoints as $\tau_0$ could easily result in noticeable jerky movements on the velocity\footnote{Note that the term "velocity" in this section refers to physical kinematic velocity rather than the diffusion velocity.} curve, as shown in Figure~\ref{fig:vel_wpt_relative}.
% In Section~\ref{sec:diffusion_loss}, we conducted an in-depth analysis to identify a suitable diffusion loss space. Although we achieved relatively high-quality trajectory generation, the results still fail to capture fine-grained details. As shown in Figure~{\ref{fig:kinematic_hybrid_loss}}, the velocity curve exhibits jitter. 
This indicates that while the model captures the global geometric structure of the trajectory, it fails to enforce local temporal coherence, which could be highly detrimental to closed-loop real-vehicle performance.

A natural fix is to use a delta representation of the trajectory for higher-order supervision: enforcing the model to predict the velocity $\tau_0^\mathbf{v} = \{(v^l_x, v^l_y)\}_{l=1}^L$ instead of absolute waypoints $\tau_0^\mathbf{x} = \{(x^l, y^l)\}_{l=1}^L$, with the final trajectory recovered via integration at inference. Interestingly, we find empirically that these two trajectory representations have a striking impact on the trajectories generated by diffusion planners. As shown by the $v$-$t$ curves and decomposed metrics in Figure~\ref{fig:vel_wpt_relative}, velocity-represented trajectories demonstrate smoothness and stability similar to human driving, enjoying a much higher comfort score. By contrast, waypoint-represented trajectories suffer from severe jerky motion, but achieve a superior ADE thanks to better modeling of global geometric structure.

\textbf{Hybrid Loss}. An intuitive idea is to supervise the model with both waypoints and velocity representation simultaneously. However, we find that the magnitude of waypoint coordinates in a trajectories increase greatly along the time-axis, while the velocity representations are more concentratedly distributed, resulting in better numerical stability when learning with a diffusion model.

% and making it a better objective for diffusion generation. 
Therefore, we retain the skeleton of velocity representation, but also incorporate the waypoint supervision through a carefully designed hybrid loss.
% Therefore, we retain the design of velocity prediction, opting instead to incorporate the waypoints representation into the model through a designed hybrid loss. 
Specifically, the model outputs the velocity of the planned trajectory, and we compute the $L_2$ loss on both the directly output velocity and the integrated waypoints:
\begin{equation}
    \begin{aligned}
        \mathcal{L}_{velocity}  &= \mathbb{E}_{\tau^\mathbf{v}_0, \epsilon, t}\lVert\tau^\mathbf{v}_\theta - \tau^\mathbf{v}_0\lVert_2^2 \\
        \mathcal{L}_{waypoints} &= \mathbb{E}_{\tau^\mathbf{x}_0,\epsilon, t}\lVert M\tau^\mathbf{v}_\theta \cdot \Delta t - \tau^\mathbf{x}_0\lVert_2^2\\ 
        &= \mathbb{E}_{\tau^\mathbf{v}_0,\epsilon, t}\lVert M\tau^\mathbf{v}_\theta \cdot \Delta t - M\tau^\mathbf{v}_0\cdot \Delta t\lVert_2^2,
        % &= \mathbb{E}_{\mathbf{x},\epsilon,t} \sum\limits_{l=1}^L(\Delta t \cdot (\sum\limits_{i=1}^l \mathbf{v}^i_\theta - \sum\limits_{i=1}^l \mathbf{v}^i))^2
    \end{aligned}
    \label{eq:kinematic_loss}
\end{equation}
where $\Delta t$ is the time interval of neighboring frames, and $M$ is a lower triangular matrix of ones that integrates the velocity into waypoints. The final hybrid loss is a weighted sum of these two losses:
\begin{wrapfigure}{r}{0.45\linewidth}
    \vspace{-7pt}
    \centering
    \includegraphics[width=\linewidth]{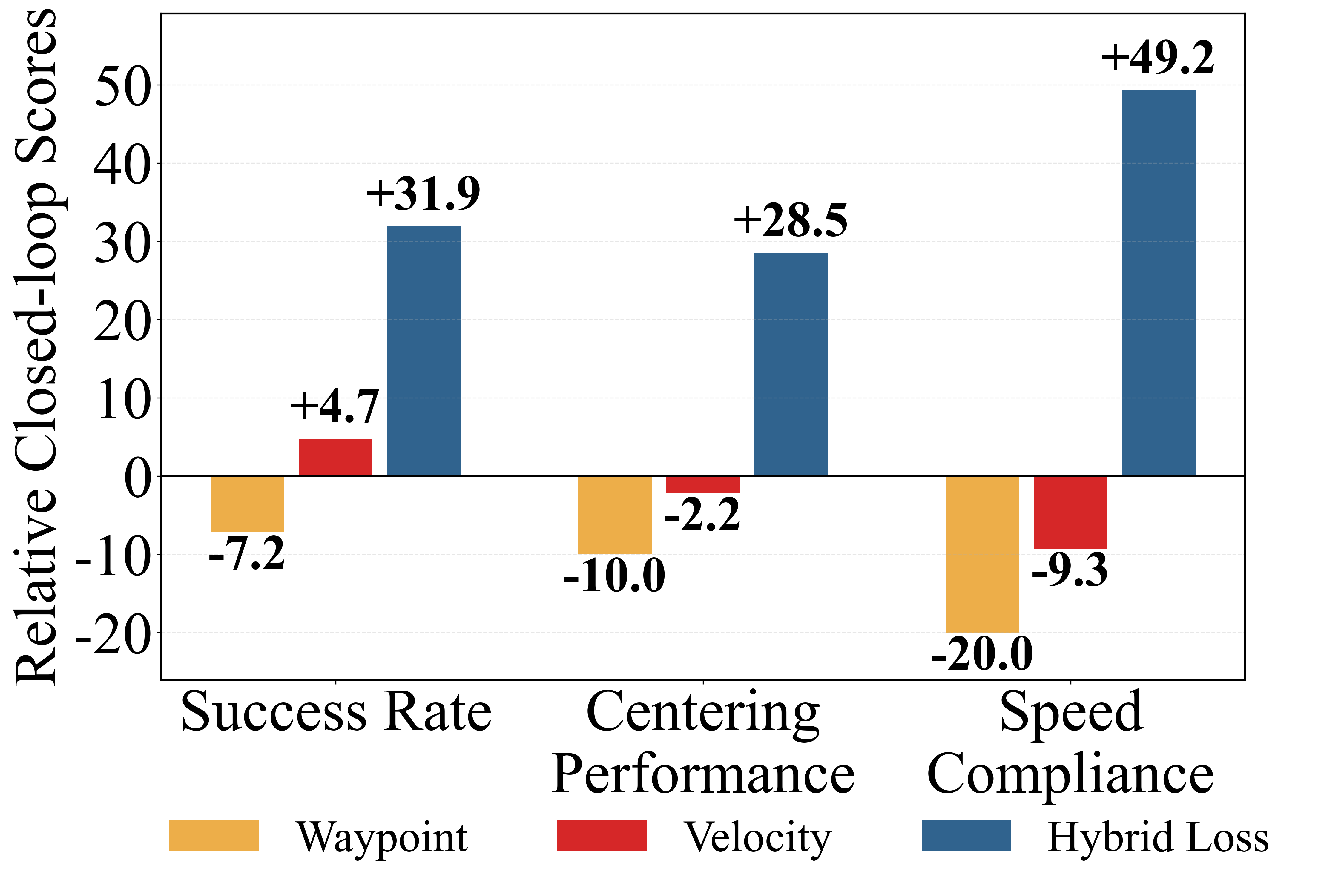}
    \caption{\small Comparison of waypoint, velocity representations and velocity with hybrid loss, which yields both steadiness and smoothness.}
    \label{fig:vel_wpt_vel_hybrid}
    \vspace{-15pt}
\end{wrapfigure}
\begin{equation}
    \begin{aligned}
        \mathcal{L}_{hybrid} &= \mathcal{L}_{velocity} + \omega \cdot \mathcal{L}_{waypoints},
        % &= \mathbb{E}_{\mathbf{v}, \epsilon, t}\lVert\mathbf{v}_\theta - \mathbf{v}\lVert_2^2 + \omega \cdot \mathbb{E}_{\mathbf{x},\epsilon, t}\lVert M\mathbf{v}_\theta \cdot \Delta t - \mathbf{x}\lVert_2^2 \\
        % &= \mathbb{E}_{\mathbf{v},\epsilon, t} \left[(\mathbf{v}_\theta - \mathbf{v})^T(\mathbf{v}_\theta - \mathbf{v})\right. \\
        % &~~~~~~~~~~ + \left.\omega\Delta t^2(\mathbf{v}_\theta - \mathbf{v})^TM^TM(\mathbf{v}_\theta - \mathbf{v})\right] \\
        % &= \mathbb{E}_{\mathbf{v},\epsilon, t} \left[(\mathbf{v}_\theta - \mathbf{v})^T(I + \omega \Delta t^2M^TM)(\mathbf{v}_\theta - \mathbf{v})\right] \\
    \end{aligned}
    \label{eq:hybrid_loss}
\end{equation}
where $\omega$ is a balancing weight. Moreover, we can theoretically show that this Eq.~(\ref{eq:hybrid_loss}) is also a valid diffusion loss to obtain the correct marginal score function of the data distribution. We first recall the following property of Bregman divergences~\citep{banerjee2005clustering}. Please see Appendix~\ref{app:proof_bregman} for proof of Lemma~\ref{lem:bregman}.

\begin{restatable}[\citet{lipman2024flowmatchingguidecode}]{lemma}{BregmanLemma}
    \label{lem:bregman}
    For any positive-definite $P\succ 0$, the quadratic form $D_{P}(u,v)=(u-v)^{\top}P\,(u-v)$ is a Bregman divergence. The unique minimizer of the conditional regression objective $\mathbb{E}_{\tau^{\mathbf{v}}_{0},\,\epsilon,\,t}[D_{P}(\tau^{\mathbf{v}}_{\theta},\tau^{\mathbf{v}}_{0})]$ is the conditional expectation $\tau^{\mathbf{v},\,\star}_{\theta}(\tau^{\mathbf{v}}_{t},t) \;=\; \mathbb{E}\!\left[\tau^{\mathbf{v}}_{0}\,\big|\,\tau^{\mathbf{v}}_{t}\right]$, which by Tweedie's formula recovers the marginal score function in Eq.~(\ref{eq:reverse}).
\end{restatable}

With Lemma~\ref{lem:bregman} in hand, it suffices to verify that the hybrid loss in Eq.~(\ref{eq:hybrid_loss}) takes the form of a positive-definite quadratic Bregman divergence. Please see Appendix~\ref{app:proof_general_score} for proof of Theorem~\ref{thm:general_score_matching}.
\begin{conclusionbox}
\begin{restatable}{theorem}{GeneralScoreMatching}
    \label{thm:general_score_matching}
    The hybrid loss in Eq.~(\ref{eq:hybrid_loss}) is equivalent to a score matching loss under $\mathit{P}$-norm:
    \begin{equation}
        \mathcal{L}_{hybrid} = \mathbb{E}_{\tau^\mathbf{v}_0, \epsilon, t}[||\tau^\mathbf{v}_\theta - \tau^\mathbf{v}_0||_\mathit{P}^2],
    \end{equation}
    where $\mathit{P} = I + \Delta t^2 \cdot \omega M^TM$ is positive-definite. The minimizer of the loss recovers the marginal score function in Eq.~(\ref{eq:reverse}).
\end{restatable}
\end{conclusionbox}

\begin{remark}[Generality of $M$]
\label{rem:generality_M}
Theorem~\ref{thm:general_score_matching} only requires $M^{\top}\!M\succeq 0$, which holds for any matrix $M$, so $\mathcal{L}_{hybrid}$ remains a valid score matching objective under any choice of kinematic-coupling matrix. The lower-triangular integration matrix in Eq.~(\ref{eq:kinematic_loss}) is the specific instantiation we adopt for its direct physical interpretation as velocity-to-waypoint integration.
\end{remark}

\begin{remark}[Bias of non-Bregman losses]
\label{rem:non_bregman_bias}
The $L_{1}$-norm losses~\citep{diffusiondrive} (not strictly convex) and auxiliary planning losses~\citep{jiang2023vad} (depend on the prediction) used in AD are not Bregman divergences. By Lemma~\ref{lem:bregman}, their minimizers do not coincide with $\mathbb{E}[\tau^{\mathbf{v}}_{0}\mid\tau^{\mathbf{v}}_{t}]$, yielding biased score estimators.
\end{remark}

% It is worth noting that many existing studies adopt L1-norm loss ~\citep{diffusiondrive} or auxiliary planning loss~\citep{jiang2023vad} (e.g., collision loss) in diffusion-based AD tasks, which will lead to a biased score function that does not faithfully reflect the data distribution. (\zyn{add appendix discussion})

% directly applying the widely adopted 1-norm loss~\citep{diffusiondrive} or auxiliary planning loss~\citep{jiang2023vad} (e.g., collision loss) in autonomous driving will lead to a biased score function that does not faithfully reflect the data distribution See Appendix for the proof and more discussion \tanty{appendix}. 

% In practice, the integration in $\mathcal{L}_{waypoint}$ could result in gradient accumulation with future timesteps. To avoid the imbalanced gradient distribution over future predictions, we limit the gradient backpropagation to a temporal window of size $W$ by detaching the trajectory history beyond this horizon (see Algorithm~\ref{alg:hybrid_loss} in Appendix~\ref{app:c} for detailed implementation). This stop-gradient operation preserves the forward value of $\mathcal{L}_{hybrid}$ and therefore does not alter the loss landscape characterized by Theorem~\ref{thm:general_score_matching}; it serves purely as a training-stabilization.

% as illustrated in a simple implementation in Algorithm~\ref{alg:hybrid_loss}.

\textbf{Detached Integral.} The integration in $\mathcal{L}_{waypoint}$ accumulates gradients along the temporal axis, causing imbalanced supervision across future timesteps. We mitigate this by restricting gradient backpropagation through the integration to a temporal window of $W$ steps, achieved via stop-gradient on the trajectory history beyond this horizon. This preserves the forward value of $\mathcal{L}_{hybrid}$ and does not alter the minimizer characterized by Theorem~\ref{thm:general_score_matching}; it serves purely as training stabilization. The matrix-form expression $\hat{\tau}^{\mathbf{x}}_{\theta} = M_W \tau^{\mathbf{v}}_{\theta}\Delta t + \mathrm{sg}((M-M_W)\tau^{\mathbf{v}}_{\theta}\Delta t)$ and full implementation are deferred to Appendix~\ref{app:hybrid loss}.

Our proposed hybrid loss enables substantial performance improvement in closed-loop real-vehicle testing. As shown in Figure~\ref{fig:vel_wpt_vel_hybrid}, training the diffusion model with the hybrid loss improves all closed-loop metrics, outperforming solely using waypoint and velocity representation by a large margin. This shows that the hybrid loss indeed effectively combines the merits of both representations, capturing the overall vehicle motion trend while preserving the kinematic coherence.

% We further plotted the $v$-$t$ curve of a trajectory generated by a model trained with the hybrid loss in Figure~\ref{fig:vel_wpt_vel_hybrid}. The predicted trajectory effectively combines the strengths of both representations, capturing the overall motion trend while preserving the dynamical properties.
% \tanty{need a subsection conclusion}
% \vspace{-10pt}
\subsection{Multimodal Capability and Data Scaling}
\label{sec:datascaling}

% To further study the generative capability of diffusion models, we perform an additional analysis on training data scaling. Starting from hundreds of thousands of frames, which is a typical capacity of existing benchmarks~\citep{Dauner2024NEURIPS}, we enlarge the data size to a maximum of approximately 70 million frames \tanty{ref Figure 5}. \tanty{to be determined}
Diffusion models are renowned for their multimodal generation capabilities. However, existing diffusion-based planning models~\citep{zheng2025diffusionplanner, tan2025flow} often suffer from severe mode collapse on AD benchmarks~\citep{caesar2021nuplan,Dauner2024NEURIPS}. To investigate the reasons for this discrepancy, as well as examine the multimodal capability and scalability of our proposed framework, we conduct a series of controlled data scaling experiments, spanning from 100K to over 70M real-vehicle training frames. By comparison, existing mainstream E2E AD benchmarks like NAVSIM~\citep{Dauner2024NEURIPS} only contain 100K training data. We evaluate our model's multimodal generation capability in Figure~\ref{fig:divergence_scaling_curve}, \ref{fig:multimodal_vis}, as well as its open- and closed-loop scaling performance in Figure~\ref{fig:data_scaling}.

% In contrast, our model demonstrates high-quality multimodal generation. To investigate the reasons for this discrepancy, we performed a data scaling experiment. Starting from hundreds of thousands of frames, which is a typical capacity of existing benchmarks~\citep{Dauner2024NEURIPS}, we gradually enlarge the data size to a maximum of approximately 20 million frames. We evaluated the model's multimodal generation capability, and the results are shown in Figure~\ref{fig:data_scaling_curve}, and the open-loop visualization in Figure~\ref{fig:data_scaling_vis}. 
\begin{figure*}[t]
    \centering

    \begin{minipage}[t]{0.415\linewidth}
        \centering
        \includegraphics[width=\linewidth]{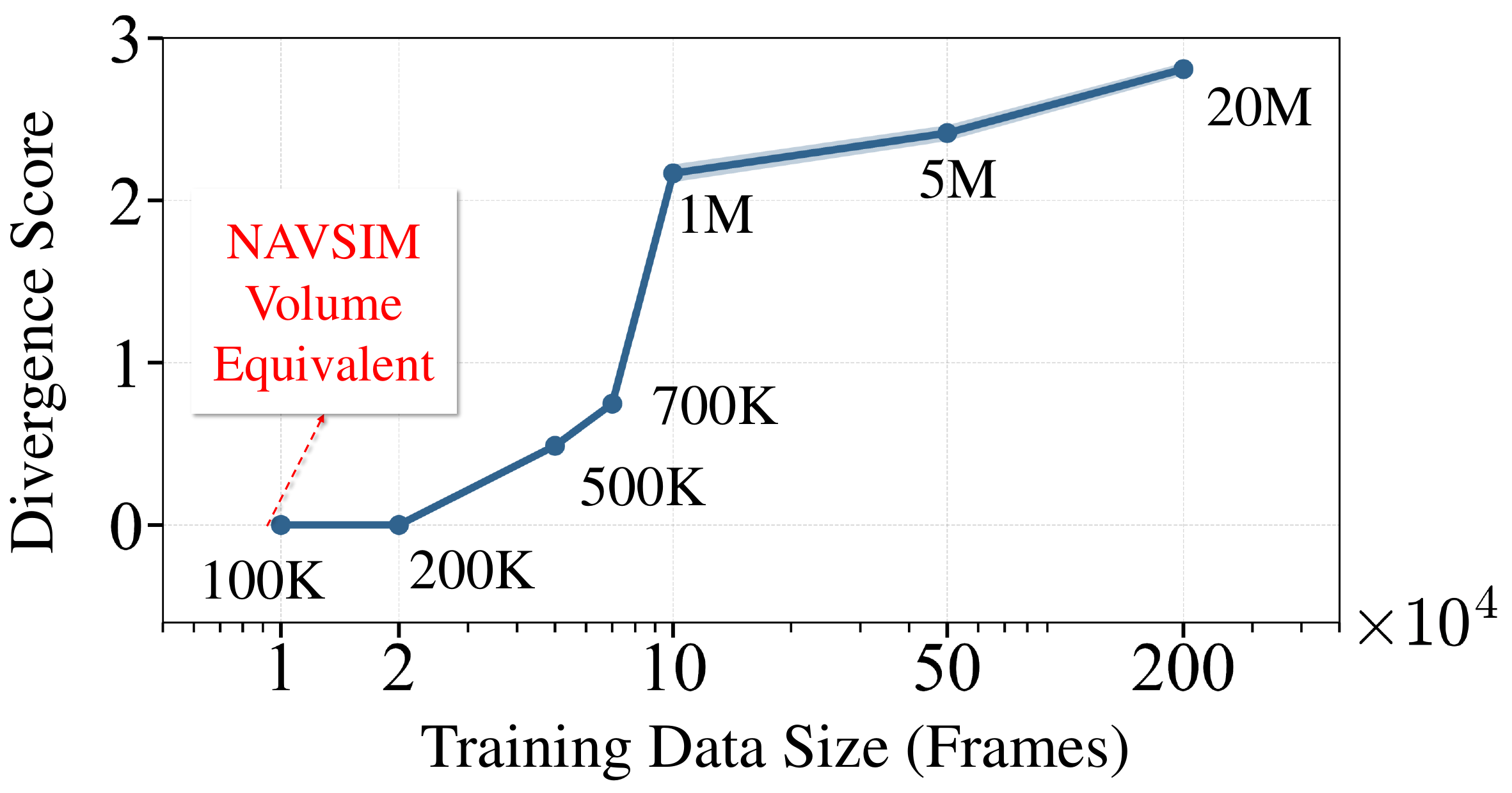}
        \captionof{figure}{\small Divergence score evaluated under different training data sizes.}
        \label{fig:divergence_scaling_curve}
    \end{minipage}
    \hspace{0.02\linewidth}
    \begin{minipage}[t]{0.535\linewidth}
        \centering
        \includegraphics[width=\linewidth]{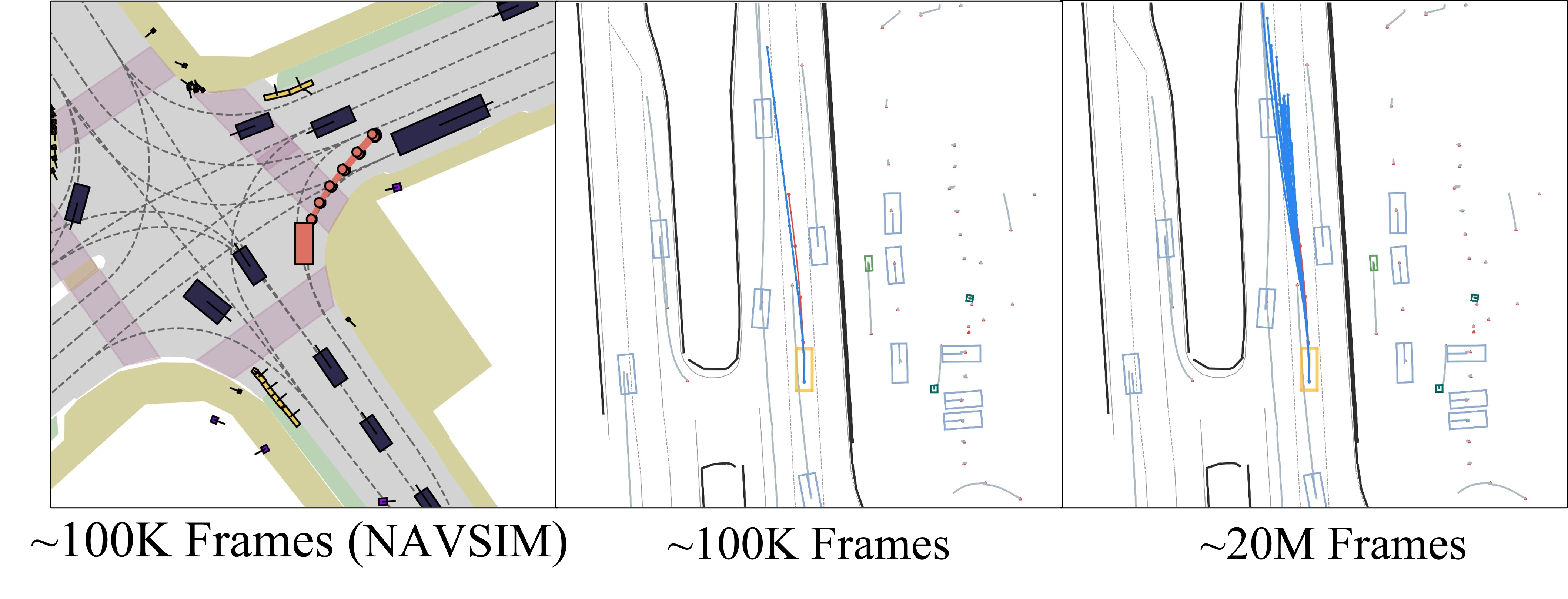}
        \captionof{figure}{\small Visualization of planning trajectories obtained under different data sources and dataset sizes.}
        \label{fig:multimodal_vis}
    \end{minipage}

    \vspace{-10pt}
\end{figure*}

\textbf{Multimodal Generation Capability}. We train our model with data from 100K (NAVSIM equivalent) to 20M frames, and use the \textit{Trajectory Divergence} introduced in Section~\ref{subsec:metrics} to measure multimodal generation.
\begin{wrapfigure}{r}{0.35\textwidth}
    \centering
    \vspace{-6pt}
    \includegraphics[width=\linewidth]{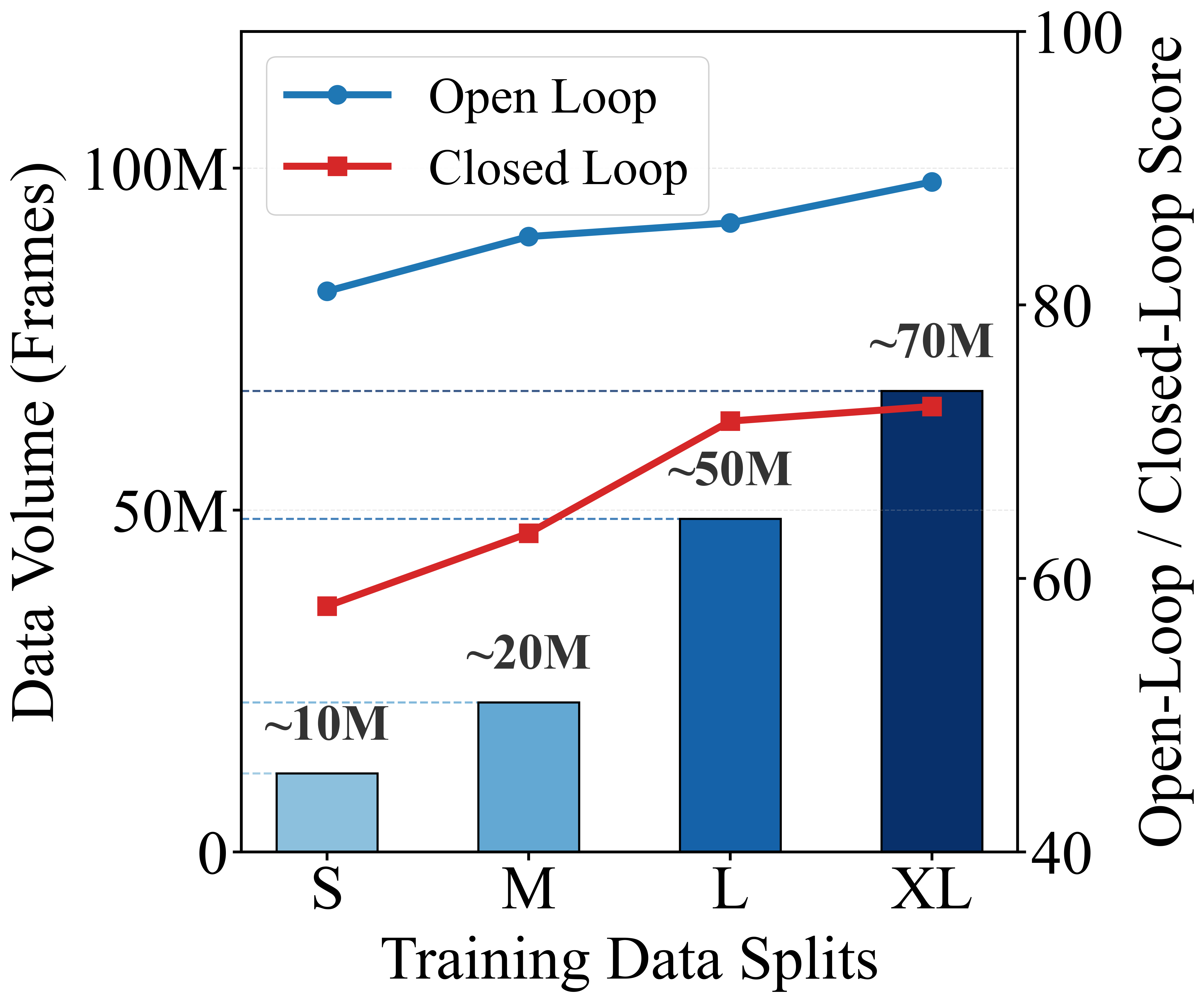}
    \caption{\small Data scaling experiments. Both open- and closed-loop performance improve as data scales up.}
    \label{fig:data_scaling}
    \vspace{-15pt}
\end{wrapfigure}
The results are shown in Figure~\ref{fig:divergence_scaling_curve}. It is observed that the model exhibits negligible multimodal capability when trained on 100K frames of data, consistent with the mode collapse observation in existing AD benchmarks. However, as the training frames increase, the divergence score grows rapidly, suggesting the emergence of multimodal behavior and enhanced generalization performance.
% and continues to increase even after the training data exceeds 1M. 
This can also be verified by inspecting the generated planning trajectories in Figure~\ref{fig:multimodal_vis}, that the generated planning trajectories exhibit clear multimodal behavior when trained with 20M frames of data, whereas all trajectories collapse to a single mode when trained on only 100K frames. Our finding is consistent with the theoretical results in \citet{zhang2023emergence}, that diffusion models need sufficient training data for generalization. It also demonstrates that diffusion models can capture multimodal behavior in diverse driving scenarios with proper scaling of training data, even without prior knowledge or bias, such as anchor~\citep{diffusiondrive} or goal conditioning~\citep{xing2025goalflow}.

% At the beginning, the model exhibits negligible multi-modal generation capability when trained on 100,000 frames of data, which is consistent with the phenomenon of mode collapse in open benchmarks. As the number of frames increases, however, the divergence grows rapidly, and continues to increase even after the number of frames of training data exceeds 1 million. These results demonstrate that the diffusion models can capture multi-modal behavior in diverse driving scenarios with proper scaling of training data, even without any prior knowledge or bias, such as anchor~\citep{diffusiondrive} or goal~\citep{xing2025goalflow}.

% \begin{figure}[h]
% \centering
%     \includegraphics[width=0.40\textwidth]{assets/data_scaling_curve.pdf}
%     \caption{The model divergence against the number of frames of training data. The divergence is computed using a group size of 64.}
%     \label{fig:data_scaling_curve}
% \end{figure}

\textbf{Performance Scaling}. We also observe a continuous improvement in both open- and closed-loop performance of our model as the data increases, as shown in Figure~\ref{fig:data_scaling}. By simply increasing the number of training data from 10M to 70M frames, the model's closed-loop performance increased by more than $20\%$, and open-loop increased by $10\%$, indicating a data scaling property on real vehicles. This demonstrates the huge potential of our methods for large-scale industrial-level applications. However we do not find prominent closed-loop improvement for regression model with data scaling, as shown in Appendix~\ref{app:ilresults}.
% verifying the scaling law of real vehicle performance. 

% \begin{figure}[h]
% \centering
%     \includegraphics[width=0.40\textwidth]{assets/data_distribution.png}
%     \caption{Numbers of frames of training data splits.}
%     \label{fig:train_data_splits}
% \end{figure}

% \begin{figure*}[t]
% \centering
%     \begin{minipage}[t]{\linewidth}
%         \includegraphics[width=0.24\textwidth]{assets/real/real_9a.png}
%         \includegraphics[width=0.24\textwidth]{assets/real/real_9b.png}
%         \includegraphics[width=0.24\textwidth]{assets/real/real_9c.png}
%         \includegraphics[width=0.24\textwidth]{assets/real/real_9d.png}
%     \end{minipage}
%     \caption{Place holder.}
%     \label{fig:real_il_vs_rl}
% \end{figure*}

\section{Reinforcement Learning Post-training}
\label{sec:rl}

% In this section, we introduce a Reinforcement Learning (RL) stage that complements imitation learning with reward signals to further enhance the model's behavioral performance.

% \subsection{Diffusion-Based Reinforcement Learning}
\textbf{KL-Regularized RL}. We adopt standard RL notation conventions, formulating the diffusion-based planner as the policy $\pi(a|s)$. Specifically, the action $a$ corresponds to the generated trajectory $\tau^{\mathbf{v}}_0$, and the state $s$ is the driving scene. The diffusion model takes as input the encoded scene $C := \mathrm{Enc}(s)$ from the scene encoder of Section~\ref{sec:basemodel}. We consider an fine-tuning setting where, at iteration $k$, we aim to optimize the policy $\pi^k$ to maximize the expected reward $r(s,a)$, starting from the previous policy $\pi^{k-1}$:
\begin{align}
\label{eq:optim}
\max_{\pi^k} \; \mathbb{E}_{s\sim \mathcal{D}}\!\left[\,\mathbb{E}_{a \sim \pi^k(\cdot\mid s)}\!\left[r(s, a)\right] - \frac{1}{\beta}\, D_\text{KL}\!\big(\pi^k(\cdot\!\mid\! s)\,\big\|\,\pi^{k-1}(\cdot\!\mid\! s)\big)\,\right],
\end{align}
where $\mathcal{D}$ is the replay buffer, $\beta > 0$ is the temperature, and $D_\text{KL}\left(p \| q\right) = \mathbb{E}_{x \sim p}\left[\log \left(p(x)/q(x) \right) \right]$. The KL-regularized objective in Eq.~(\ref{eq:optim}) provides a closed-form solution for $\pi^k$ as follows~\citep{nair2020awac}:
\begin{equation}
\label{eq:closedform}
    \pi^{k^\star}(a\mid s) \propto \pi^{k-1}(a\mid s)\cdot \exp\!\left(\beta r(s,a) \right).
\end{equation}
To extract the optimal policy in Eq.~(\ref{eq:closedform}), one approach is to use classifier guidance to steer the diffusion process toward generating high-reward actions during inference~\citep{lu2023contrastive, zheng2025diffusionplanner}. However, this method requires additional inference-time gradient computation, which is very costly and difficult to implement on real vehicles. An alternative approach is to employ a weighted regression loss based on the diffusion imitation loss~\citep{black2023training, liang2026dipole, zheng2024safe}. Specialized to our setting—$\tau_0$-prediction loss on the velocity-based trajectory representation (so that $a \equiv \tau^{\mathbf{v}}_0$ and $\tau^{\mathbf{v}}_t = \alpha_t \tau^{\mathbf{v}}_0 + \sigma_t \epsilon$)—it reads
\begin{align}
\label{eq:awr}
\mathcal{L}_{RL}=\mathbb{E}_{t,\,\epsilon,\,(s,\,\tau^{\mathbf{v}}_0)\sim \mathcal{D}}\!\left[ \exp\!\left(\beta\, r(s,\tau^{\mathbf{v}}_0) \right) \big\lVert\tau^{\mathbf{v};k}_\theta(\tau^{\mathbf{v}}_t, t, C) - \tau^{\mathbf{v}}_0 \big\rVert_2^2\right],
\end{align}
where $\tau^{\mathbf{v};k}_\theta$ denotes the parameterized diffusion model corresponding to the policy $\pi^{k}$, and the actions $\tau^{\mathbf{v}}_0$ in $\mathcal{D}$ are drawn from the previous-iteration policy $\pi^{k-1}$. The weighted regression loss in Eq.~(\ref{eq:awr}) only modifies the imitation loss with a weight term, maintaining almost the same computational cost as IL. In contrast, other methods model the denoising process as a multi-step MDP with Gaussian transitions to estimate intermediate log-likelihoods~\citep{black2023training, ren2024diffusion}, and use RL algorithms like PPO~\citep{schulman2017proximal} for policy optimization. However, these approaches require storing gradients for all denoising steps during inference and assume a large number of steps to ensure Gaussian transition validity, leading to significantly increased computational cost. We further provide a detailed discussion in Appendix~\ref{app:rlresults}.

\textbf{RL-Hybrid Loss}. To maintain consistency with the hybrid loss defined in Eq.~(\ref{eq:hybrid_loss}) used during imitation pretraining, we introduce the RL-hybrid loss for the post-training phase. This consistency loss design helps mitigate distribution shift, which is a common problem in the offline-to-online fine-tuning setting~\citep{levine2020offline}. We provide further results in Section~\ref{sec:rl_results}.
\begin{align}
\label{eq:awr_hybrid}
    \mathcal{L}_{RL\text{-}hybrid} = \mathbb{E}_{t,\,\epsilon,\,(s,\,\tau^{\mathbf{v}}_0)\sim\mathcal{D}}\!\left[\exp\!\left(\beta\, r(s,\tau^{\mathbf{v}}_0) \right) \big\lVert \tau^{\mathbf{v};k}_\theta(\tau^{\mathbf{v}}_t, t, C) - \tau^{\mathbf{v}}_0 \big\rVert_{\mathit{P}}^{2}\right].
\end{align}
Besides, we prove that the hybrid loss can be naturally combined during the RL post-training procedure to optimize the policy, due to its simple formulation as a weighted regression, as shown in Theorem~\ref{theorem:fisor}. Proof see Appendix~\ref{app:proof_awr_hybrid}.

\begin{conclusionbox}
\begin{restatable}{theorem}{OptimalPolicy}
    \label{theorem:fisor}
    Optimal action $a \sim \pi^{k^\star}(a|s)$ in Eq.~(\ref{eq:closedform}) can be generated by optimizing the weighted diffusion loss in Eq.~(\ref{eq:awr_hybrid}) and solving the diffusion reverse process with the learned $\tau^{\mathbf{v};k^\star}_\theta$.
\end{restatable}
\end{conclusionbox}
% \subsection{Practical Implementation}

In practice, we initialize the policy $\pi^0$ for RL post-training using an imitation model pretrained with the hybrid loss in Eq.~(\ref{eq:hybrid_loss}). We first consider the safety reward $r_\text{safety}$, which captures collision risk and yields a safety-aware policy. To further enhance performance, our framework naturally extends to a multi-reward setting: we refine the safety signal into a continuous risk score $r_\text{risk}$, and additionally include $r_\text{follow}$ for car-following comfort and $r_\text{lane}$ for lane-keeping robustness. The training reward $r$ used in Eq.~(\ref{eq:awr_hybrid}) is then a weighted sum of these components. See Appendix~\ref{ap:implementation}, \ref{app:rewards} for more implementation details.

\begin{table*}[t]
\centering
\caption{\small Main results. \colorbox{mine}{\color{mine}{a}} represents the highest score. The open-loop score is evaluated through data replay on test datasets, while the closed-loop score is obtained from real-world road testing on a real-vehicle platform.}
\vspace{-5pt}
\resizebox{\linewidth}{!}{\scriptsize
\begin{tabular}{l c c c c c} \toprule 
\multirow{2}{*}[-0.8ex]{\makecell[l]{\textbf{Model Name}}} & \multirow{2}{*}[-0.8ex]{\makecell[l]{\textbf{Data Size}}} & \multirow{2}{*}[-0.8ex]{\makecell[l]{\textbf{Open-Loop Score}}} & \multicolumn{3}{c}{\textbf{Closed-Loop Score}} \\ \cmidrule(lr){4-6} 
& & & \textbf{Success Rate} & \textbf{Stability Score} & \textbf{Overall Score} \\ \midrule
Base Model & M& 51.07 & 15.67 & 0.00 & 7.83 \\
with $\tau_0$-loss \& $\tau_0$-pred & M& 75.27 & 22.84 & 0.00 & 11.42 \\
+ Velocity Supervision & M& 84.38 & 34.72 & 9.24 & 21.98 \\
+ Hybrid Loss & M & 85.05 & 61.88 & 53.88 & 57.88 \\ \midrule
+ Data Scaling & L& 86.07 & 70.59 & 59.00 &64.79 \\
+ Data Scaling (HDP) & XL& \colorbox{mine}{88.94} & 71.24 & {79.53} & 75.38 \\ \midrule
+ RL with safety reward (HDP-RL$^\dagger$) & -& - & {72.89} & {79.53} & {76.20}\\
+ RL with multi-rewards (HDP-RL)  & -& - & \colorbox{mine}{83.49} & \colorbox{mine}{84.65} & \colorbox{mine}{84.07}\\
\bottomrule \end{tabular}}
\label{tab:mainresults}
\vspace{-5pt}
\end{table*}

\begin{figure}
    \centering
    \includegraphics[width=\linewidth]{assets/scaling_closed_loop.png}
    \vspace{-10pt}
    \caption{\small (a) Relative success rate vs.\ dataset volume in frequently-occurring scenarios; (b) stability score under the same scaling; (c) relative success rate before and after RL fine-tuning in safety-critical scenarios.}
    \label{fig:closed_loop_scale_and_rl}
    \vspace{-5pt}
\end{figure}

\section{Real-Vehicle Testing Results}
\label{sec:exp}

Given the above findings and designs for diffusion-based planning methods for E2E AD, we incorporate all the aforementioned innovations into a complete framework, \textit{\textbf{H}yper \textbf{D}iffusion \textbf{P}lanner}  (\textit{\name{}}). We begin with the base model introduced in Section~\ref{sec:basemodel}, which uses $\epsilon$-loss and $\epsilon$-pred (Base Model).  In Section~\ref{sec:diffusion_loss}, we find that using $\tau_0$-pred and $\tau_0$-loss achieves the best trajectory quality among other diffusion loss variants (with $\tau_0$-loss \& $\tau_0$-pred). Afterwards, in Section~\ref{sec:trajectoryrepresentation}, we discover that using velocity as a supervision signal performs better than using waypoints (+ Velocity Supervision). Combining both improvements, we introduce a hybrid loss function (+ Hybrid Loss). Furthermore, we scale up the dataset in Section~\ref{sec:datascaling} from the original 20M samples to 50M (+ Data Scaling / L) and 70M (+ Data Scaling / XL), resulting in the final version of \textit{\name{}}. Applying the RL method in Section~\ref{sec:rl} with only $r_{\text{safety}}$ gives \textit{\name{}-RL$^\dagger$}, while extending to the multi-reward setting gives \textit{\name{}-RL}. Implementation details are deferred to Appendix~\ref{ap:implementation}. As shown in Figure~\ref{fig:real_case}, our model handles complex real-world urban driving scenarios well; more cases are provided in Appendix~\ref{app:a}.

\subsection{Imitation Learning Pretraining Lays a Strong Foundation}

As shown in Table~\ref{tab:mainresults}. \textit{\name{}} achieves nearly a 10x improvement in closed-loop performance compared to the base model. For the open-loop setting, during imitation pretraining, it shows that a well-designed loss function and data scaling can steadily improve performance. However, a significant improvement in the closed-loop score is observed only after applying the hybrid loss, highlighting the difference between open-loop and closed-loop metrics. The key insight is that the hybrid loss greatly enhances stability, allowing the model to have a higher probability of completing each task, thereby achieving an overall noticeable improvement. In addition, when scaling up the data, we show the relative success rate on frequent scenarios, as illustrated in Figure~\ref{fig:closed_loop_scale_and_rl}. 
On XL, we observe a trade-off: ``Navigational lane change'' improves by +18.9 while ``Car-following with stopping'' drops by 6.2, suggesting that the model reallocates capacity toward more complex behaviors as data scales up.
As shown in Figure~\ref{fig:closed_loop_scale_and_rl}, we also observe a significant gain in the stability score, including both centering performance and speed compliance, indicating that the model better captures the underlying data distribution when trained on larger datasets. To further demonstrate the effectiveness of our method, we provide more comparison against baseline methods under real-vehicle setting in Appendix~\ref{app:ilresults}.

\begin{figure*}[t]
\centering
    \begin{minipage}[t]{0.49\linewidth}
        \includegraphics[width=0.49\textwidth]{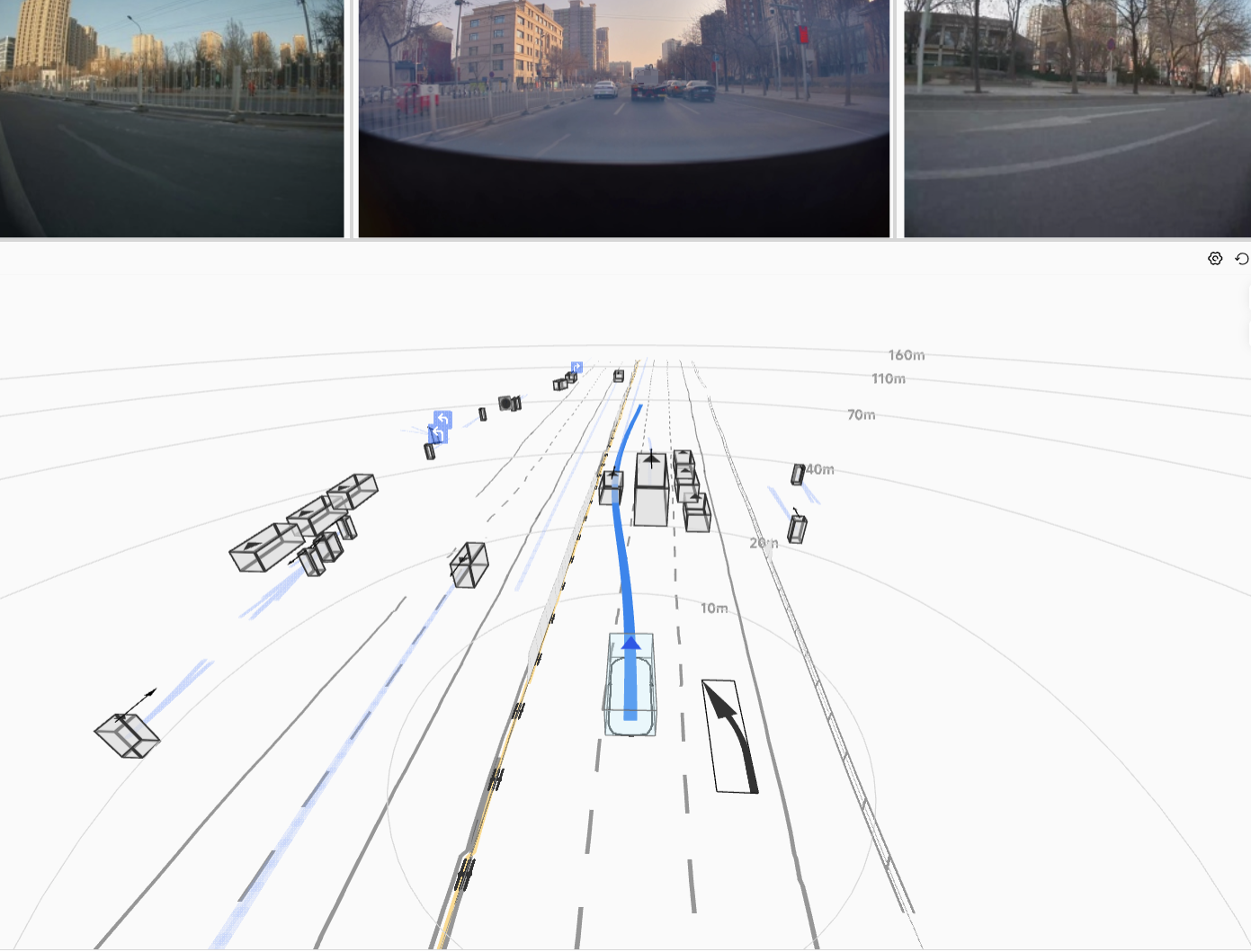}
        \includegraphics[width=0.49\textwidth]{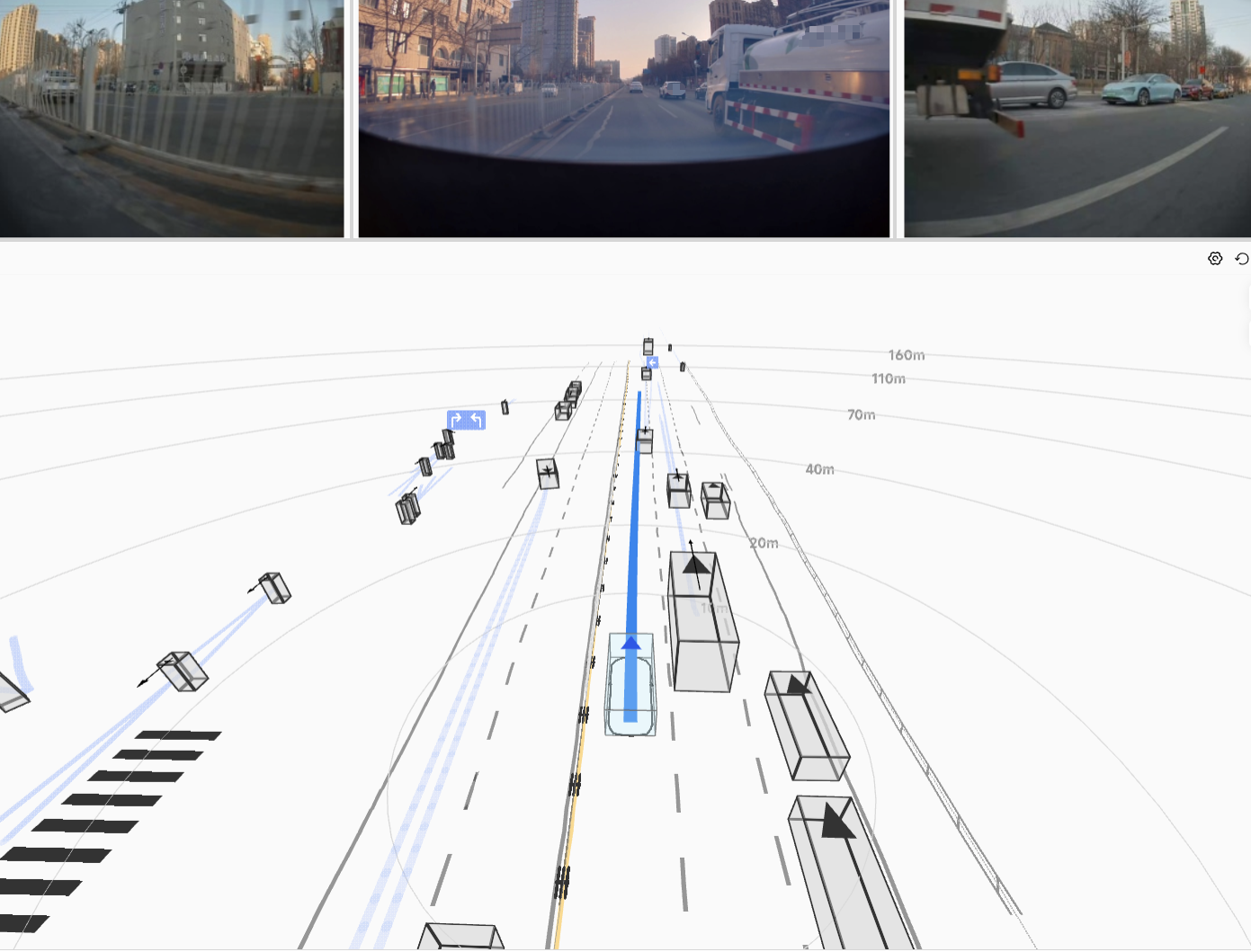}
        \subcaption{Efficient Lane Change.} \label{fig:efficientchange} 
    \end{minipage}
    \begin{minipage}[t]{0.49\linewidth}
        \includegraphics[width=0.49\textwidth]{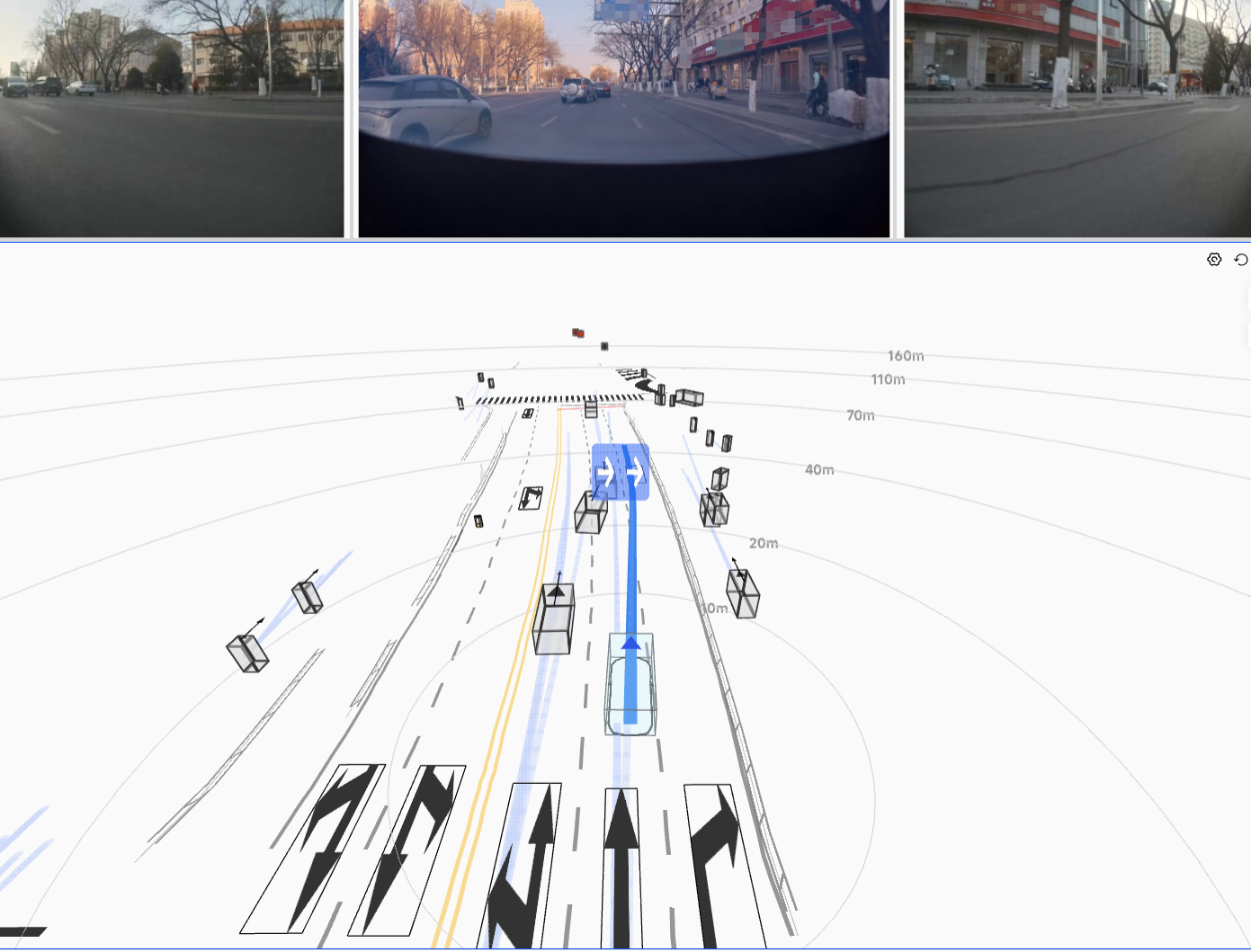}
        \includegraphics[width=0.49\textwidth]{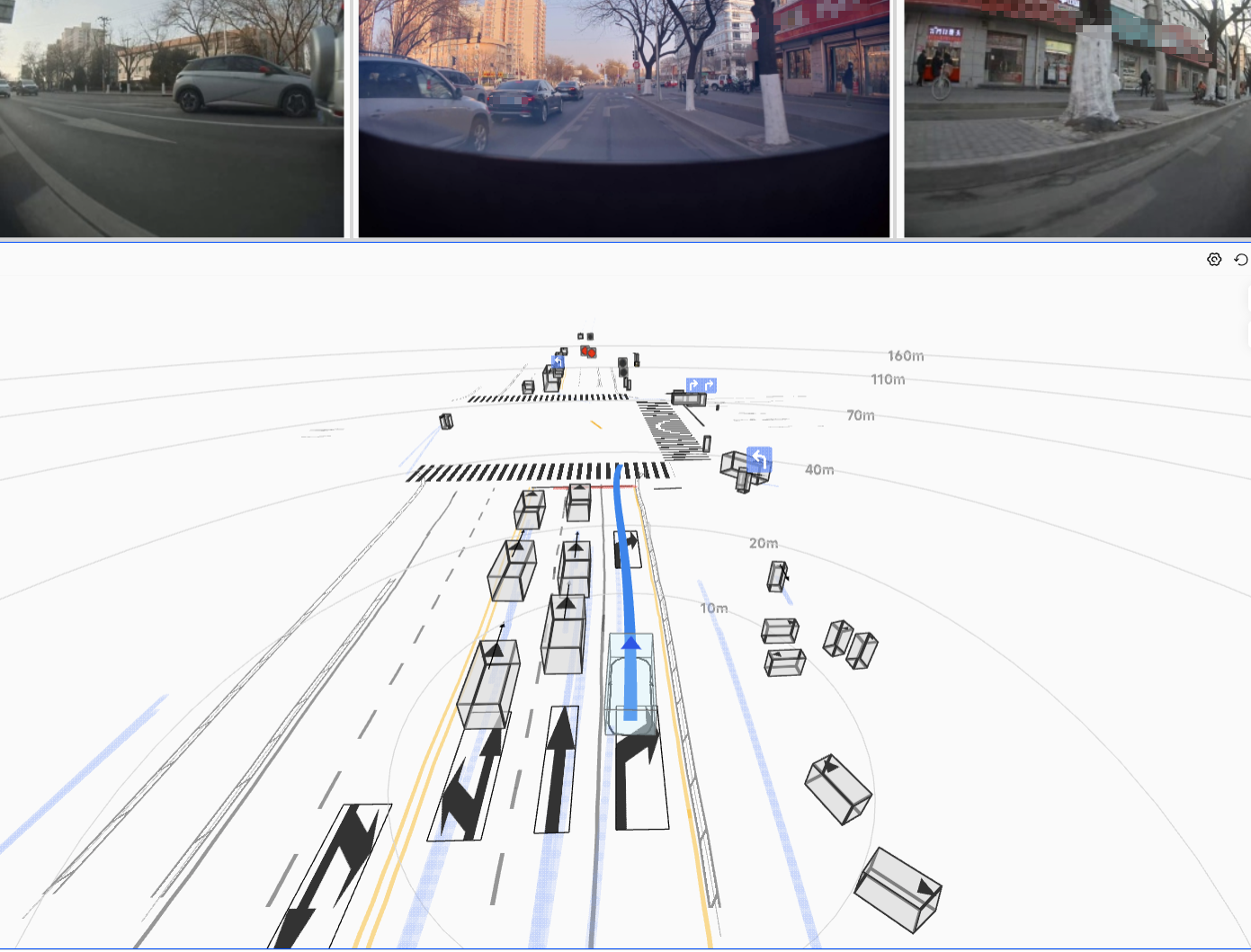}
        \subcaption{Navigational Lane Change.} \label{fig:navichange} 
    \end{minipage}

    \begin{minipage}[t]{0.49\linewidth}
        \includegraphics[width=0.49\textwidth]{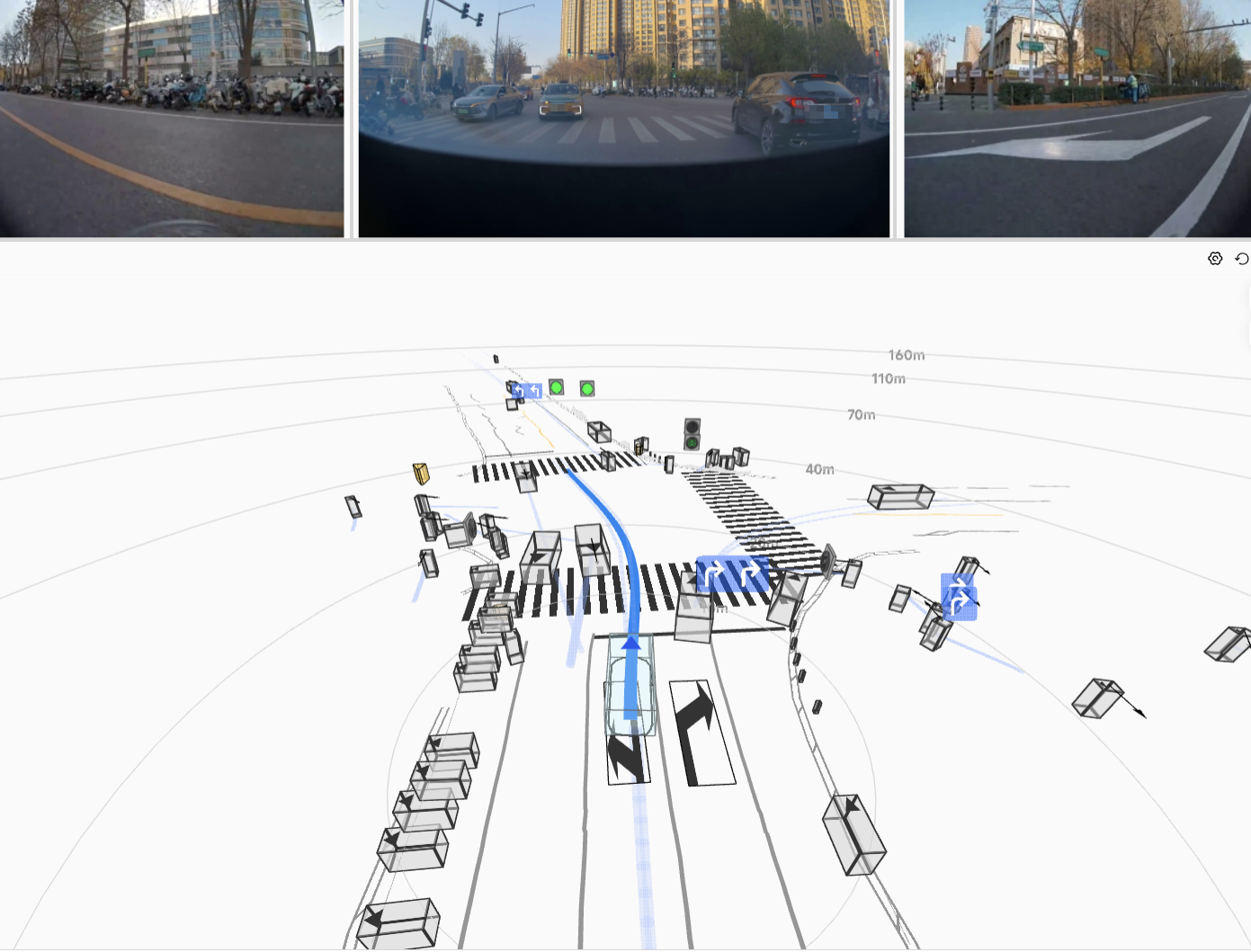}
        \includegraphics[width=0.49\textwidth]{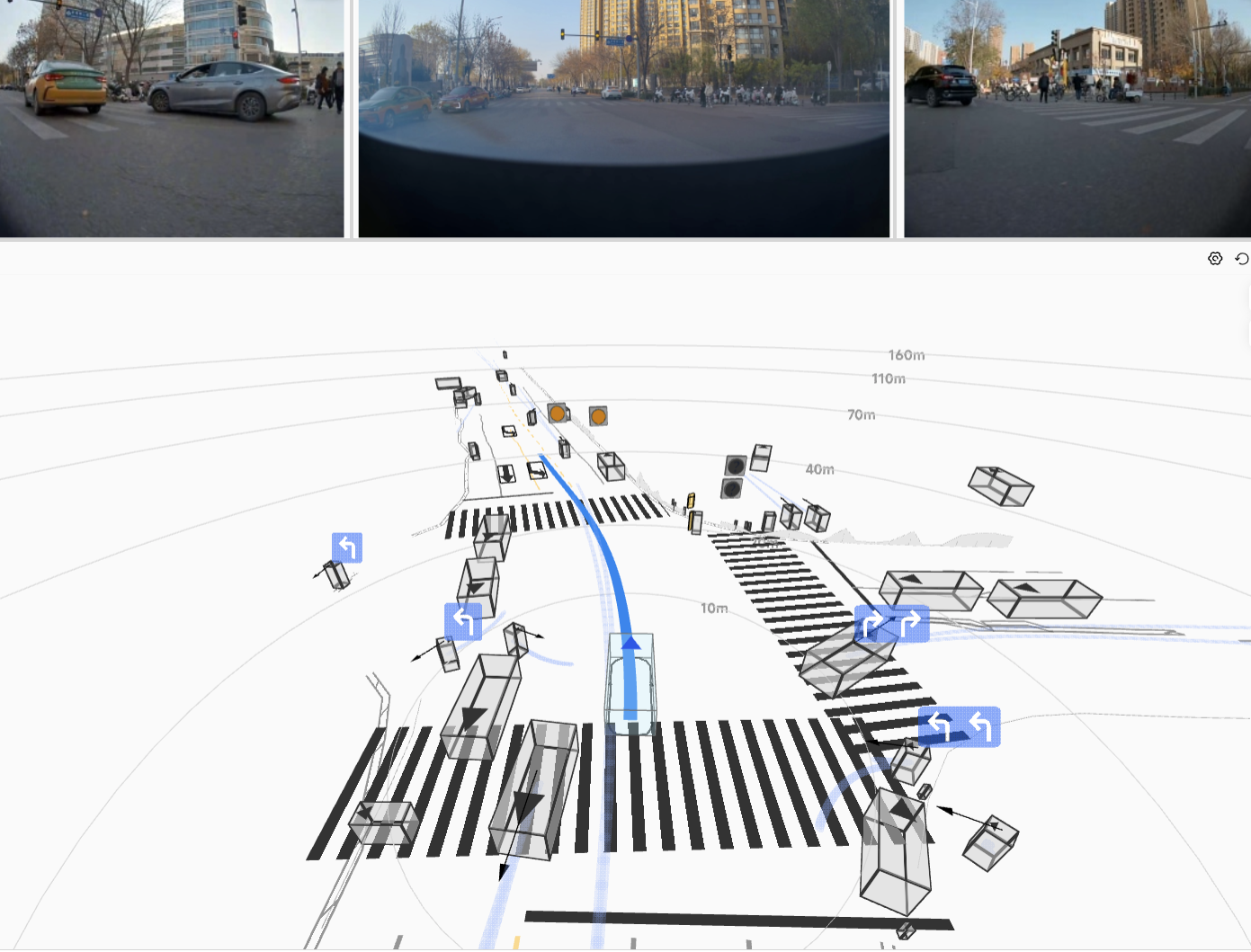}
        \subcaption{Vehicle avoidance at intersection.} \label{fig:avoidance} 
    \end{minipage}
    \begin{minipage}[t]{0.49\linewidth}
        \includegraphics[width=0.49\textwidth]{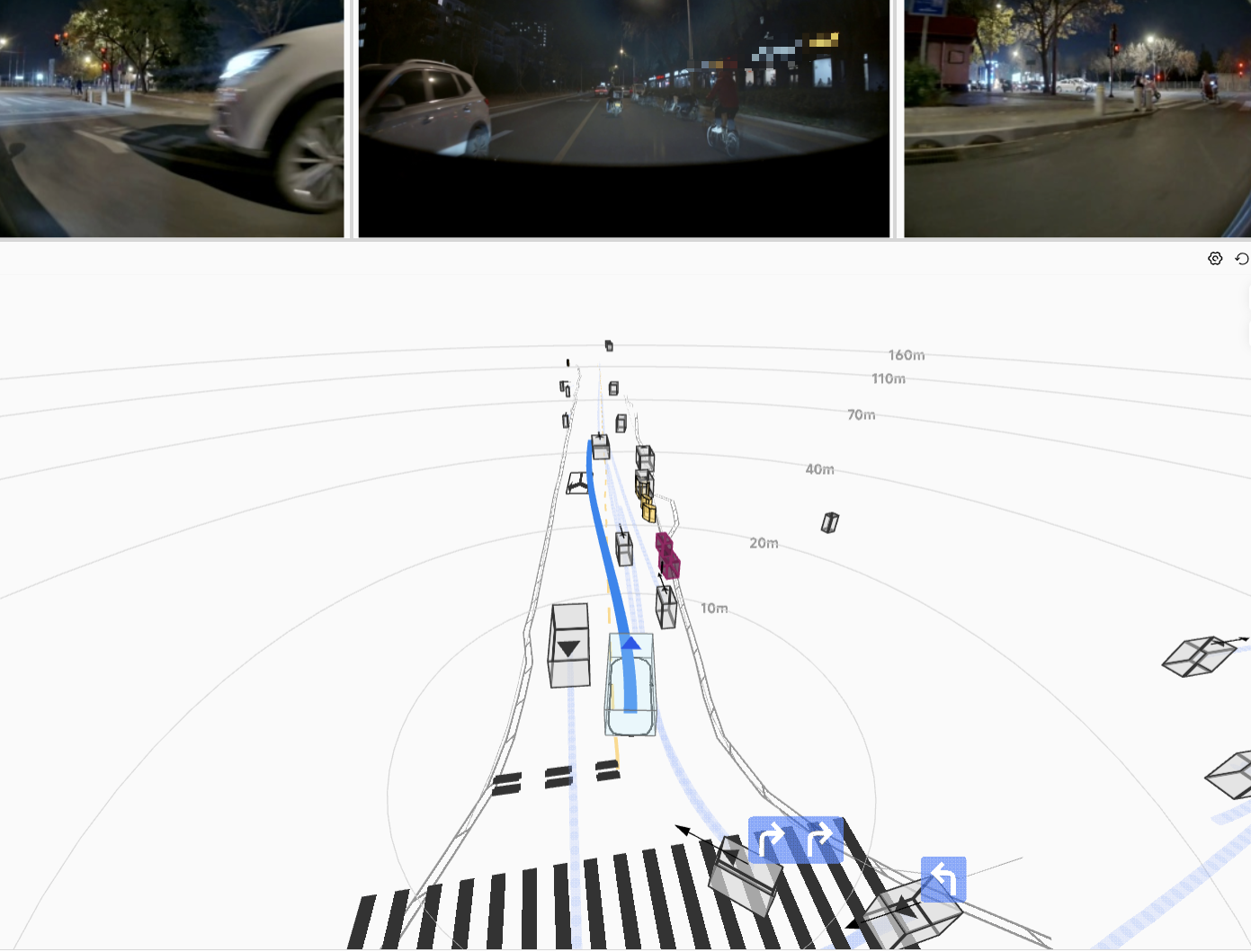}
        \includegraphics[width=0.49\textwidth]{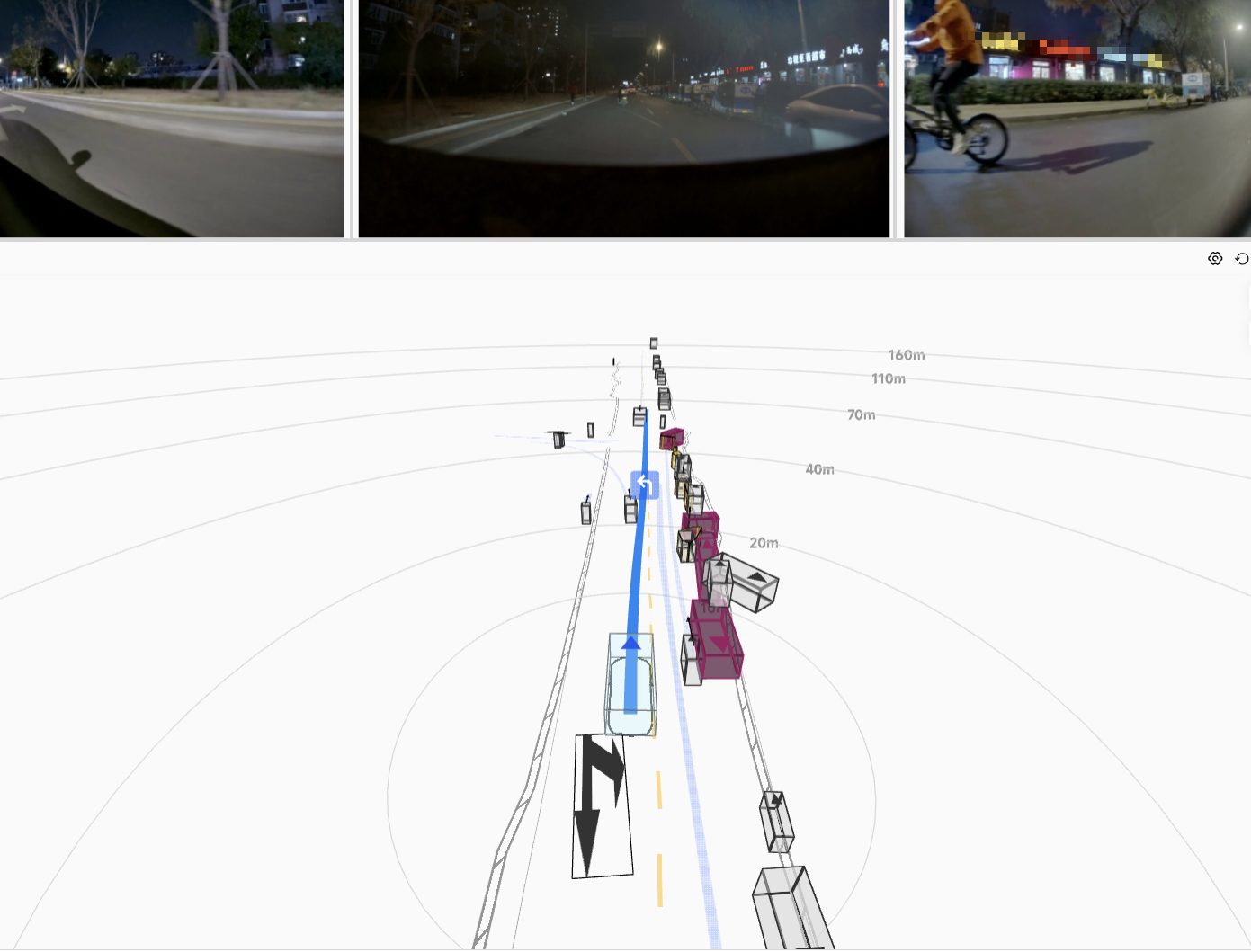}
        \subcaption{VRU avoidance.} \label{fig:vru} 
    \end{minipage}
    \caption{\small Closed-loop real-vehicle testing, illustrated with two representative frames.}
    \label{fig:real_case}
    \vspace{-10pt}
\end{figure*}

\subsection{Reinforcement Learning Post-Training Yields a Reliable Planner}
\label{sec:rl_results}

\begin{wrapfigure}{r}{0.45\linewidth}
    \centering
    \vspace{-8pt}
    % 第一张图
    \begin{minipage}[b]{\linewidth}
        \centering
        \includegraphics[width=\linewidth]{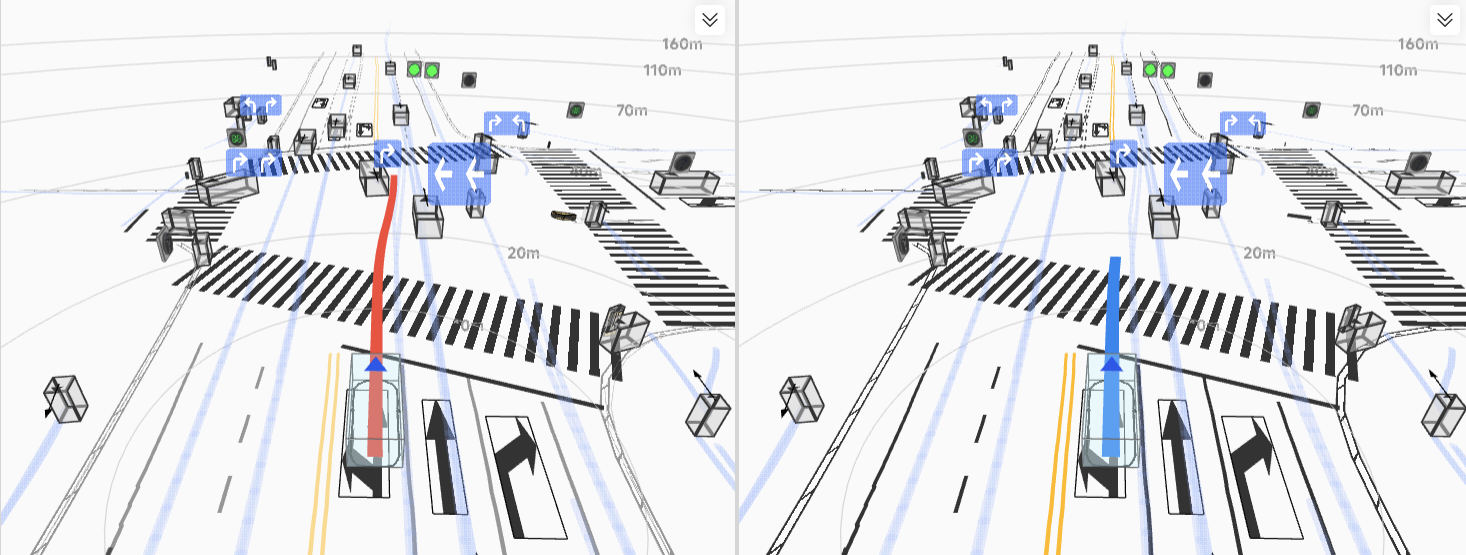}
        \subcaption{Avoid oncoming vehicles.}
        \label{fig:avoidoncoming}
    \end{minipage}
    \vspace{2pt}
    % 第二张图
    \begin{minipage}[b]{\linewidth}
        \centering
        \includegraphics[width=\linewidth]{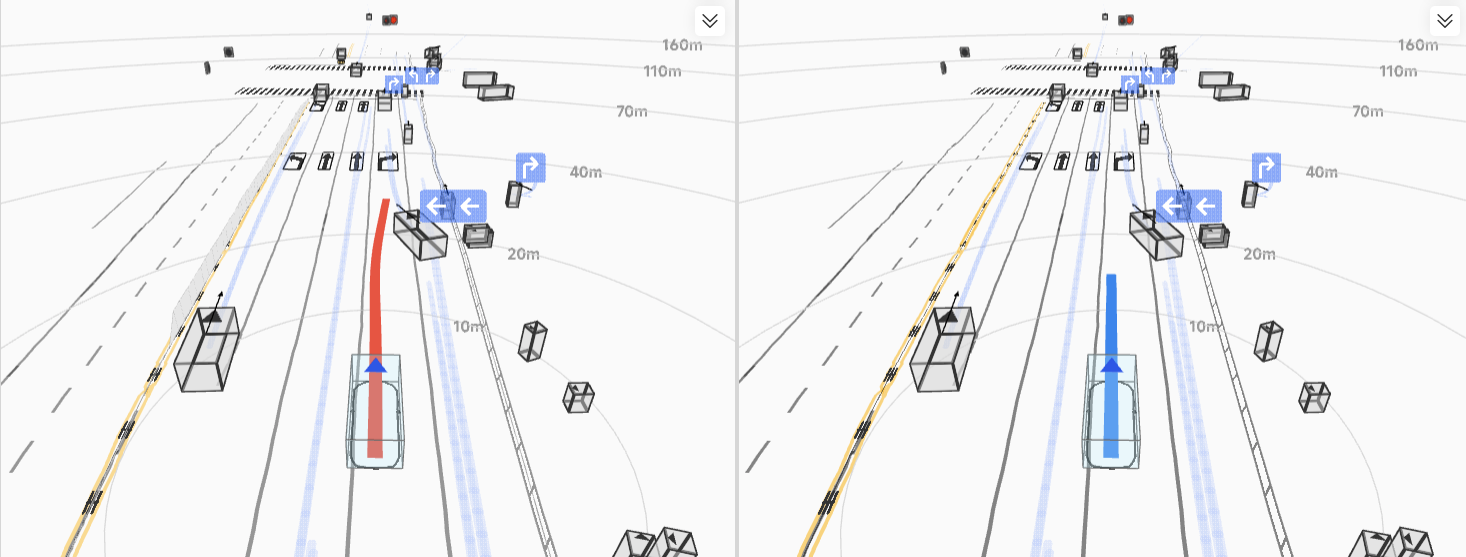}
        \subcaption{Avoid cutting-in vehicles.}
        \label{fig:avoidcutin}
    \end{minipage}
    \caption{\small Replay of bad cases before and after post-training. \textit{\name{}-RL$^\dagger$} in 
    \textcolor{ego_pred_blue}{blue}, \textit{\name{}} in 
    \textcolor{fut_gt_red}{red}.}
    \label{fig:rlcase}
    \vspace{-10pt}
\end{wrapfigure}

\textbf{RL with Safety Reward}. As shown in Figure~\ref{fig:closed_loop_scale_and_rl} and 
Table~\ref{tab:mainresults}, \textit{\name{}-RL$^\dagger$} substantially 
improves safety-related metrics. 
Figure~\ref{fig:rlcase} further illustrates the effect of RL post-training through two representative cases. In Figure~\ref{fig:rlcase}(a), when an oncoming vehicle slightly intrudes into the ego lane, \textit{\name{}-RL$^\dagger$} proactively steers laterally away to enlarge the safety margin, whereas \textit{\name{}} maintains a centered trajectory and leaves much narrower clearance. In Figure~\ref{fig:rlcase}(b), when a vehicle abruptly cuts in from the adjacent lane, \textit{\name{}-RL$^\dagger$} smoothly yields space to the cut-in vehicle, while \textit{\name{}} stays closer to the original path and passes the intruder at a notably smaller distance. These two cases consistently show that \textit{\name{}-RL$^\dagger$} produces safer and more defensive maneuvers around surrounding traffic participants, confirming the effectiveness of our RL algorithm. 

\textbf{RL with Multi-Rewards}. However, with only $r_\text{safety}$, the overall 
performance does not improve substantially. Since real-world AD is 
inherently a multi-objective problem, we extend the RL stage to the 
multi-reward setting: our best model \textit{\name{}-RL} achieves a 
10-point improvement (Table~\ref{tab:mainresults}) over \textit{\name{}}, showing that RL post-training 
truly works for real-world AD. Moreover, under the same challenging setting, our RL fine-tuning approach serves as 
an effective alternative to baseline RL methods, as shown in the 
Appendix~\ref{app:rlresults}.

\textbf{RL-Hybrid Loss Matters}. We further show that maintaining alignment with imitation pretraining is necessary when using the hybrid loss, as illustrated in Figure~\ref{fig:multi_rewards_hybrid}. In this setup, our \textit{\name{}-RL} employs the RL-hybrid loss, whereas the baseline without RL-hybrid loss uses vanilla weighted regression, as in Eq.~(\ref{eq:awr}). Because the policy is imitation-pretrained with the hybrid loss, preserving it during RL is critical: without it, we observe a substantial performance drop at the beginning of RL, which we attribute to distribution shift. By contrast, \textit{\name{}-RL} exhibits steadily increasing reward, with multiple reward components improving and advancing the Pareto frontier, while closed-loop performance improves as well. We emphasize that forgoing the hybrid loss may be acceptable when imitation pretraining uses the standard diffusion loss; however, as discussed in Section~\ref{sec:trajectoryrepresentation}, the hybrid loss is important for further improving closed-loop performance.

\begin{figure}[t]
    \centering
    \includegraphics[width=\textwidth]{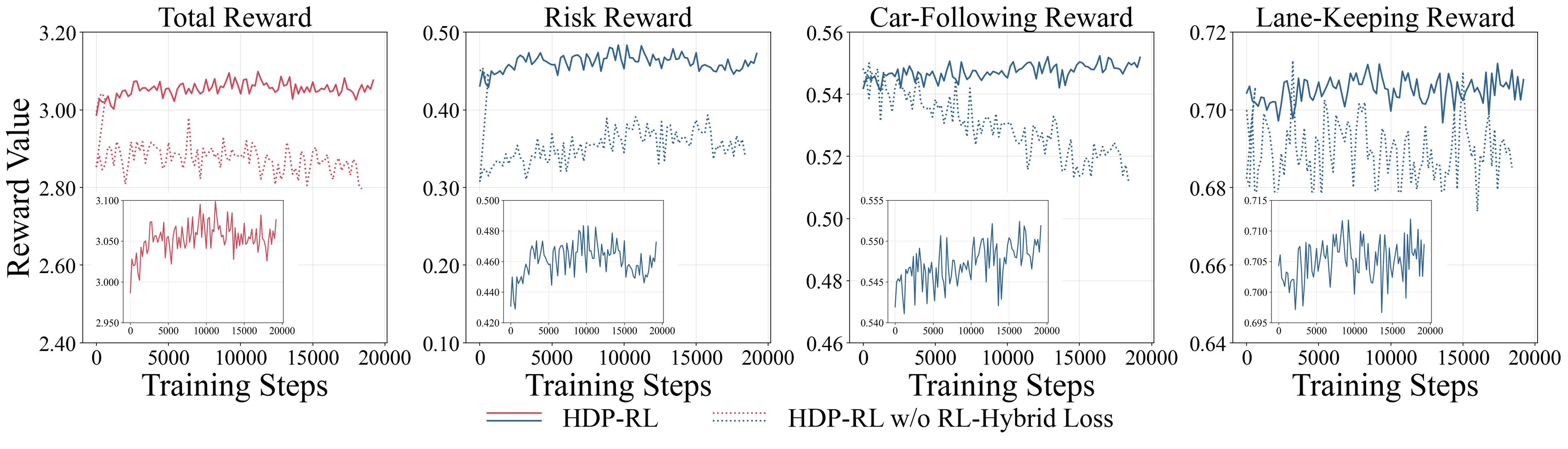}
    \caption{\small Reward versus training steps: HDP-RL with vs. without the RL-hybrid loss in Eq.~(\ref{eq:awr_hybrid}). The subgraph is the image after rescaling the coordinates.}
    \label{fig:multi_rewards_hybrid}
\end{figure}

\section{Conclusion} 
\label{sec:conclusion}

In this paper, we introduce the \textit{\textbf{H}yper \textbf{D}iffusion \textbf{P}lanner} (\textit{\name{}}), a novel framework that effectively harnesses the generative capabilities of diffusion models for E2E AD. Through comprehensive and controlled studies, we identify key insights into the diffusion loss space, trajectory representation, and data scaling, revealing their critical impact on E2E planning performance. Furthermore, we integrate an effective RL post-training strategy to enhance the safety and robustness of the learned planner. \textit{\name{}} is deployed on a real-vehicle platform and validated across 6 urban driving scenarios and 200 km of real-world testing, achieving a notable 10x performance improvement over the base diffusion planner. These results demonstrate that diffusion models, when properly designed and trained, serve as effective and scalable solutions for complex, real-world autonomous driving tasks. Due to space limit, more discussion on limitations and future direction can be found in Appendix \ref{ap:limitations}.

\section*{Acknowledgments}
This work is supported by Xiaomi EV and funding from Wuxi Research Institute of Applied Technologies, Tsinghua University under Grant 20242001120 and the Xiongan AI Institute. Furthermore, we would like to thank Zhiming Li, Huahang Liu, Yan Wang, Xuhui Lu, Xiaojun Ni and Guang Li from Xiaomi EV for their resource support and real vehicle deployment support. We would like to express our gratitude to Quanyun Zhou, Qi Tang, Cheng Chen, Xibin Yue, Qing Li from Xiaomi EV for their valuable discussion. In addition, we thank Jianxiong Li and Zhihao Wang from AIR, Tsinghua University, and Kexin Zheng from The University of Hong Kong for their helpful discussions.

\newpage
\bibliographystyle{assets/plainnat}
\bibliography{neurips_2026}
%%%%%%%%%%%%%%%%%%%%%%%%%%%%%%%%%%%%%%%%%%%%%%%%%%%%%%%%%%%%

\appendix

\newpage

\section{Visualization of Real-Vehicle Testing Results}
\label{app:a}
Below we provide more visualizations of real-vehicle testing besides Section~\ref{sec:exp}.

\begin{figure}[H]
\centering
\includegraphics[width=0.24\textwidth]{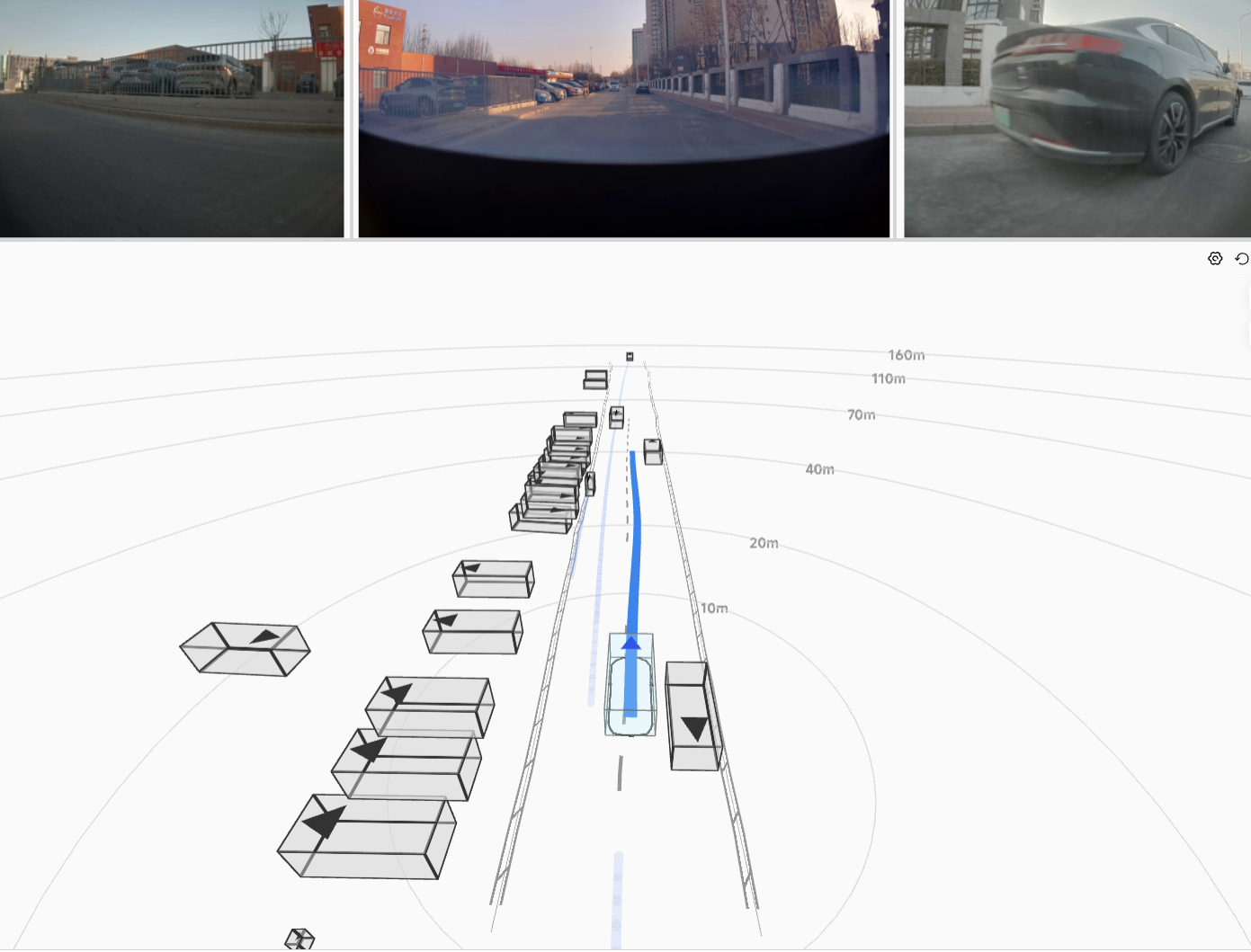}\hfill
\includegraphics[width=0.24\textwidth]{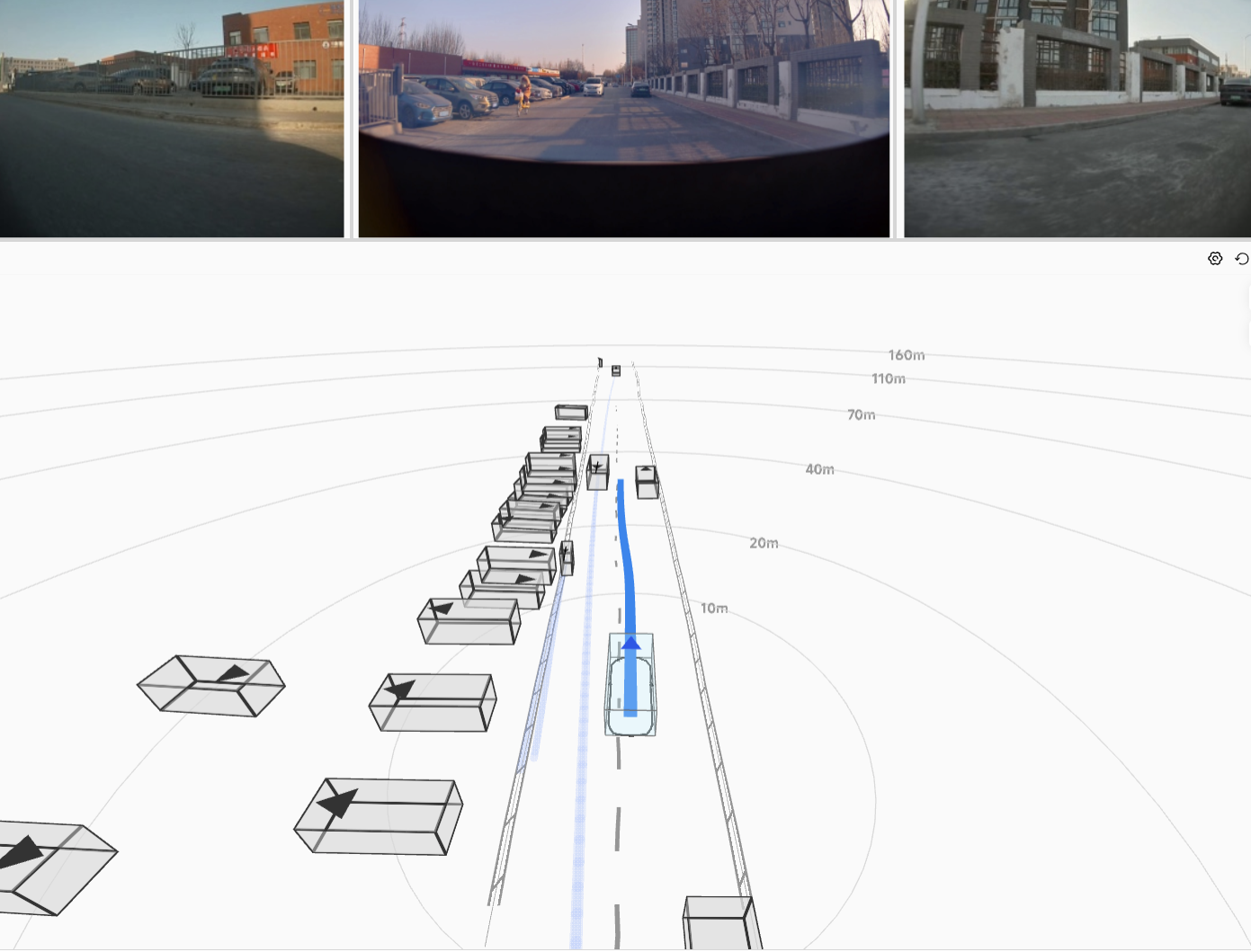}\hfill
\includegraphics[width=0.24\textwidth]{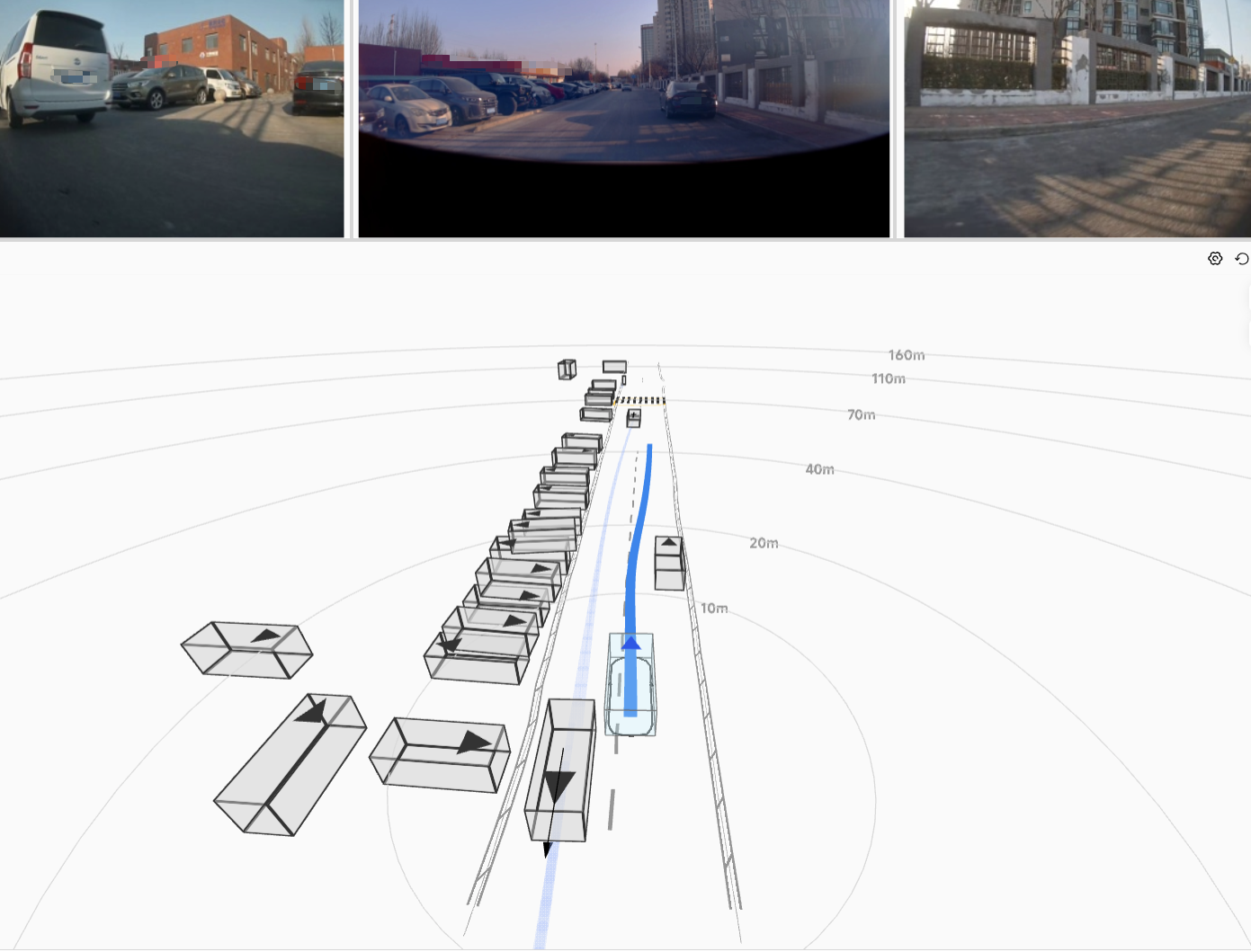}\hfill
\includegraphics[width=0.24\textwidth]{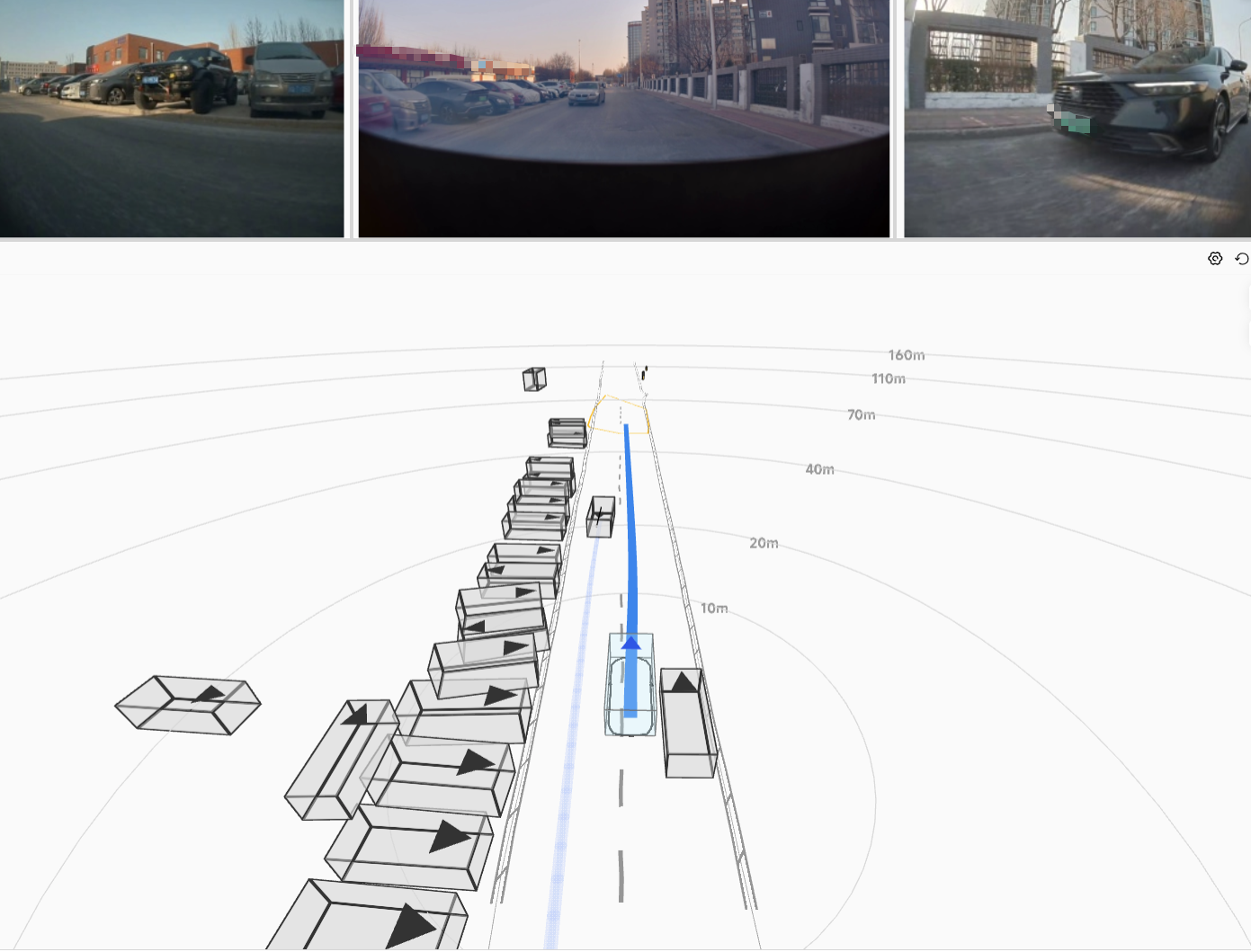}\\[2pt]
\includegraphics[width=0.24\textwidth]{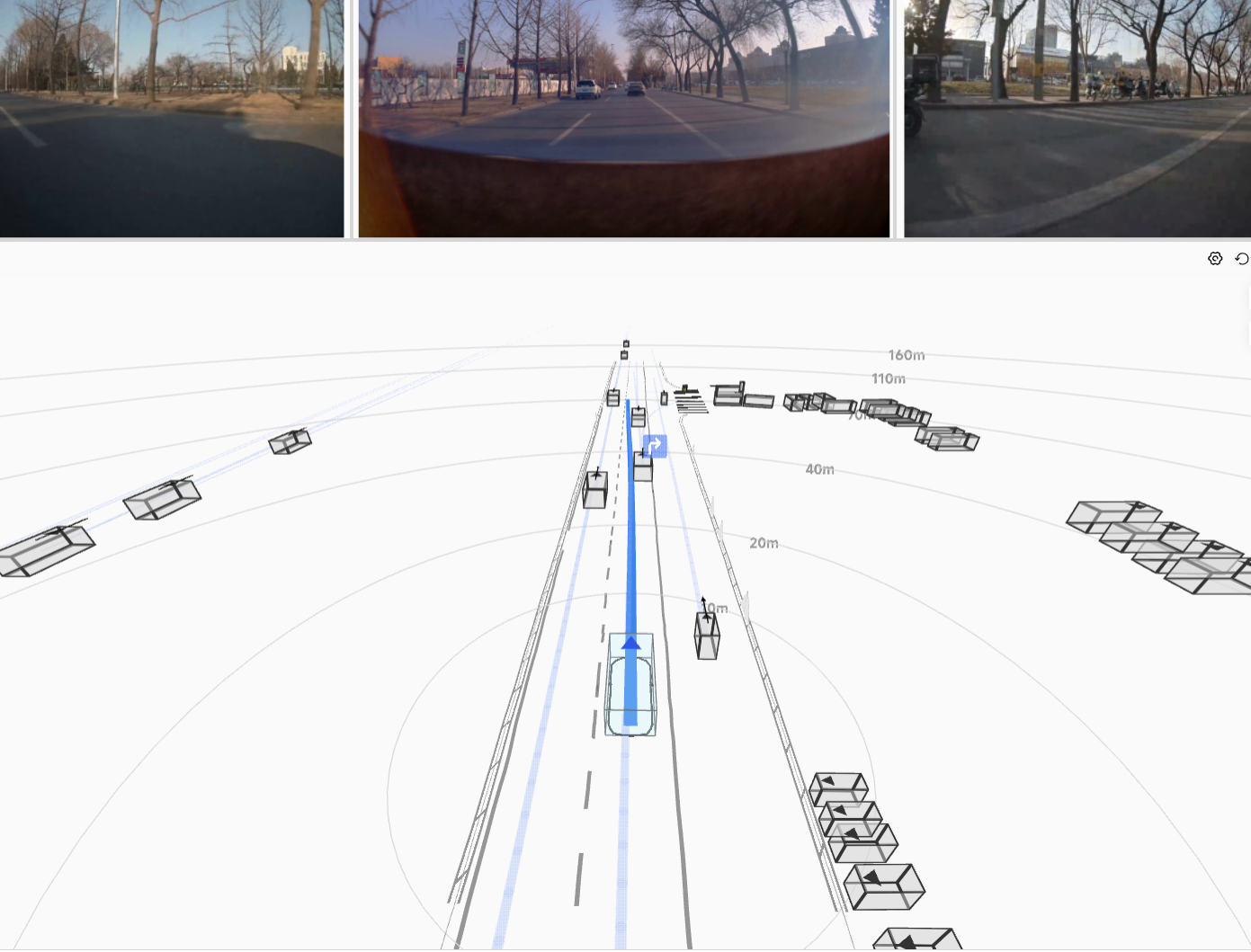}\hfill
\includegraphics[width=0.24\textwidth]{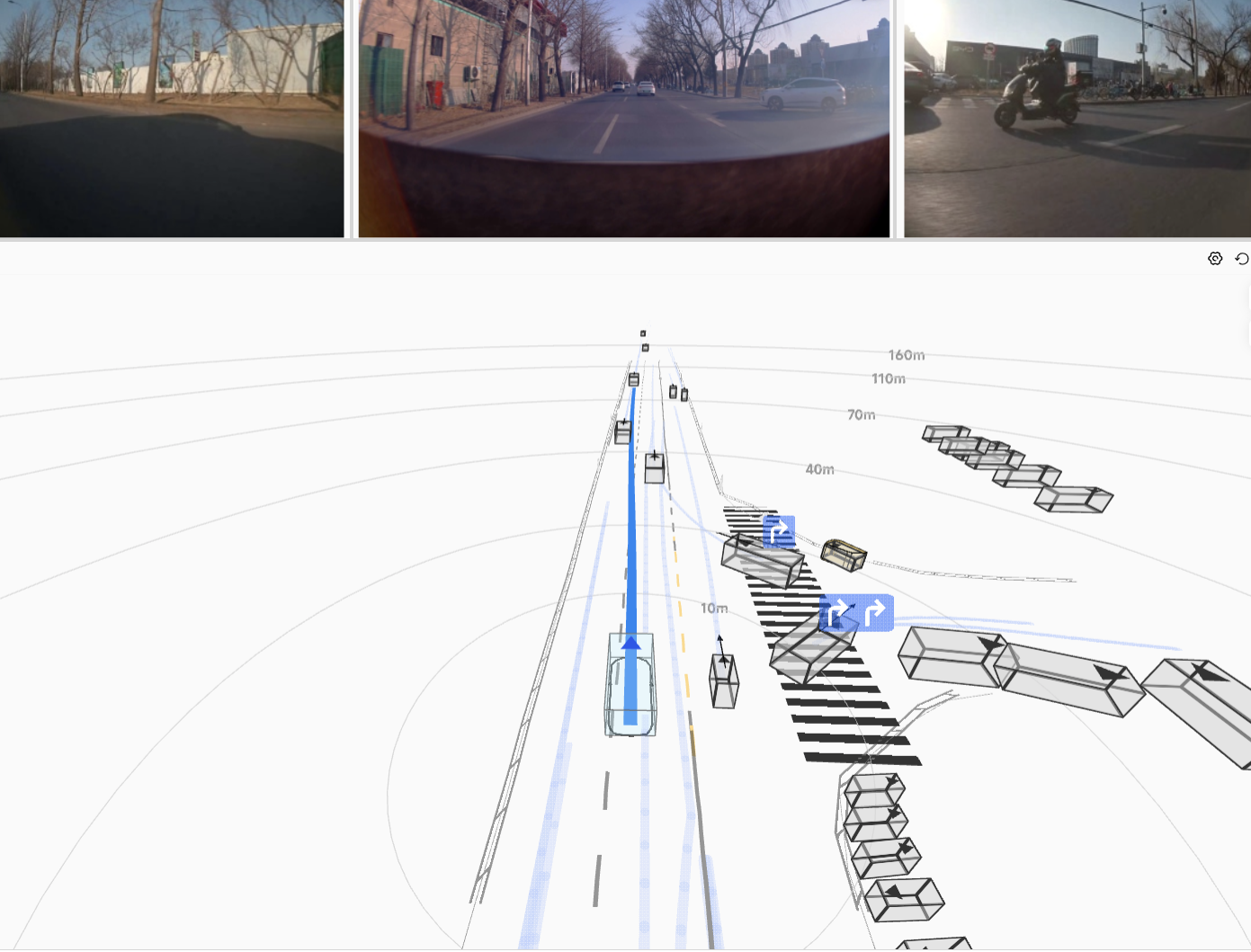}\hfill
\includegraphics[width=0.24\textwidth]{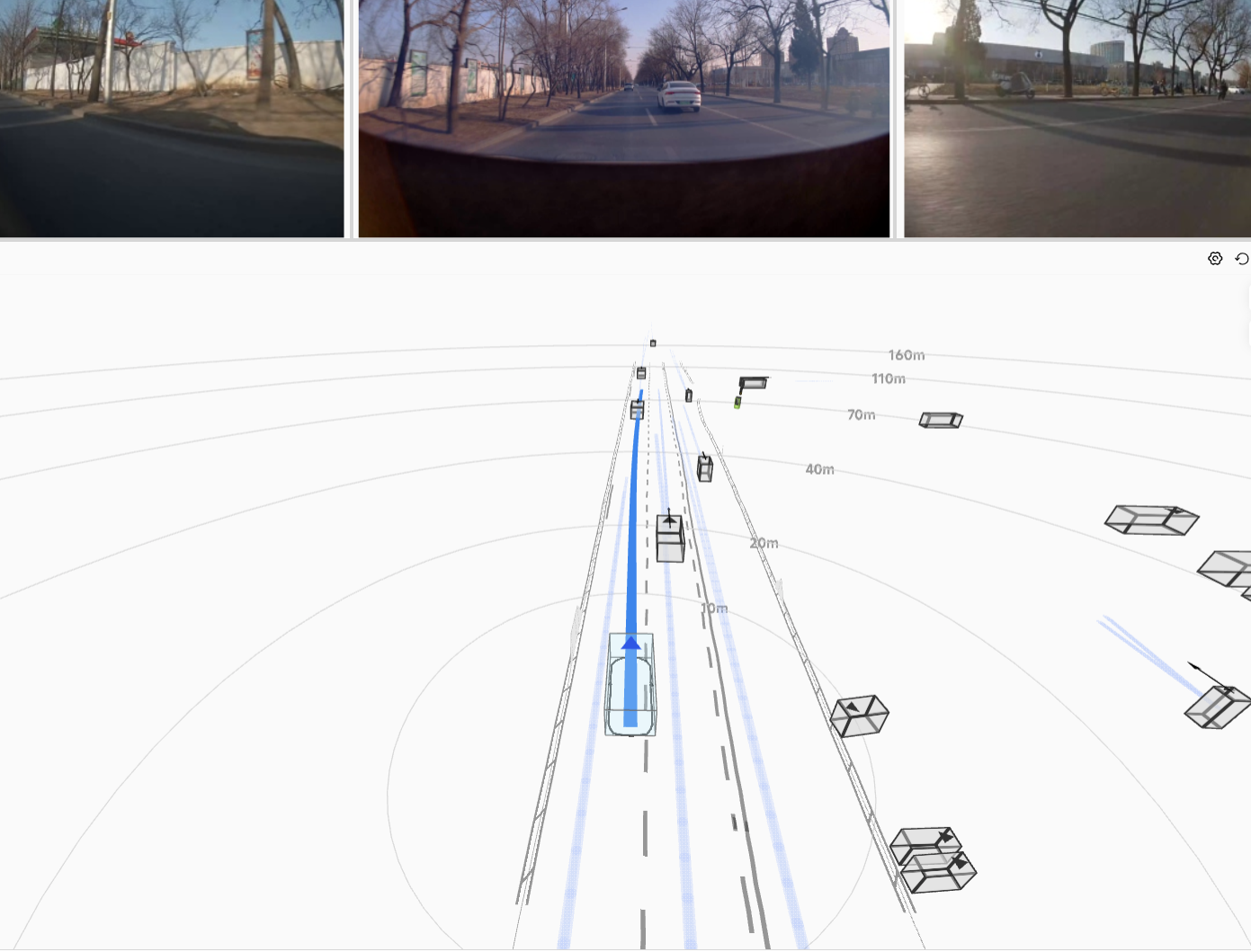}\hfill
\includegraphics[width=0.24\textwidth]{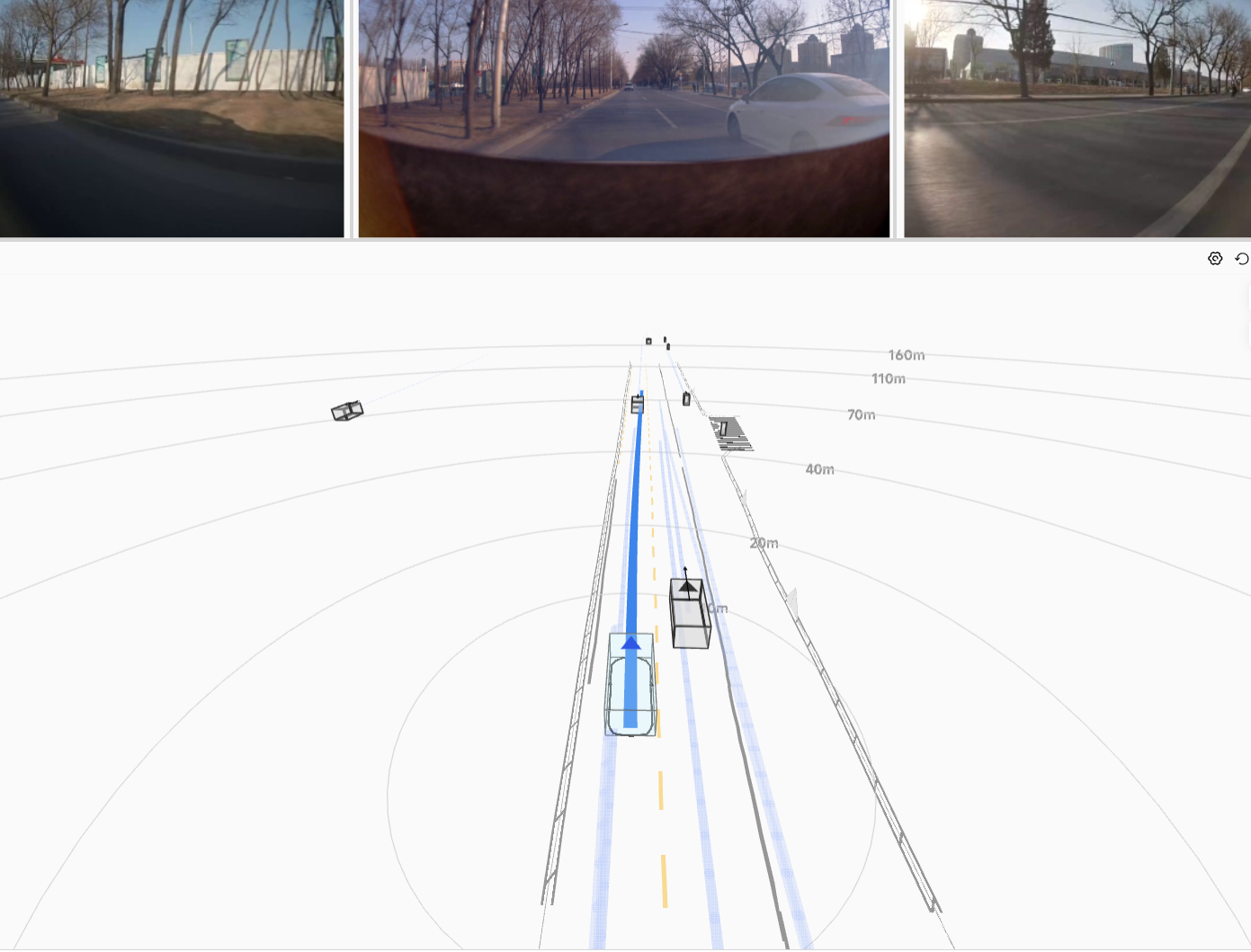}\\[2pt]
\includegraphics[width=0.24\textwidth]{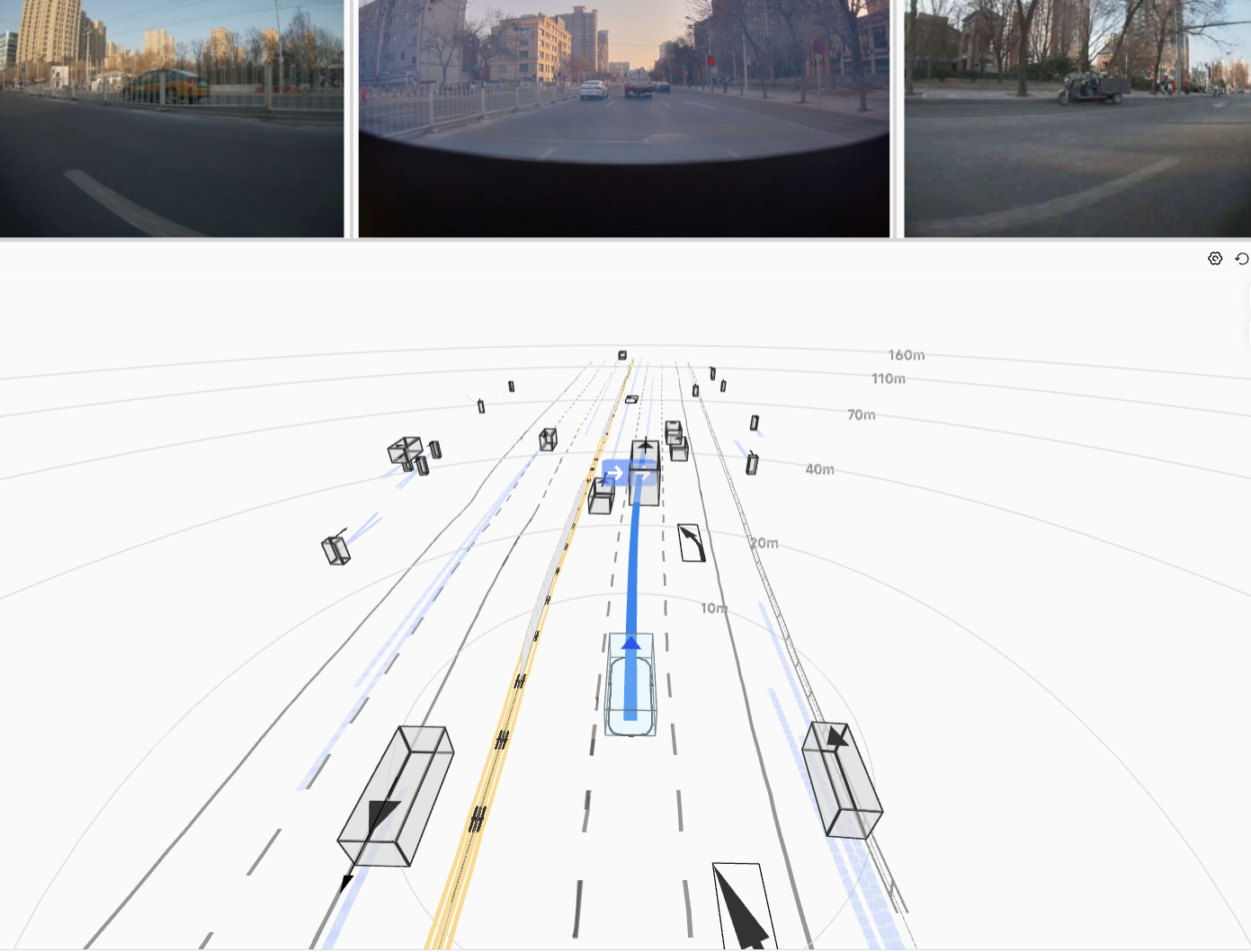}\hfill
\includegraphics[width=0.24\textwidth]{assets/real/real_4b.png}\hfill
\includegraphics[width=0.24\textwidth]{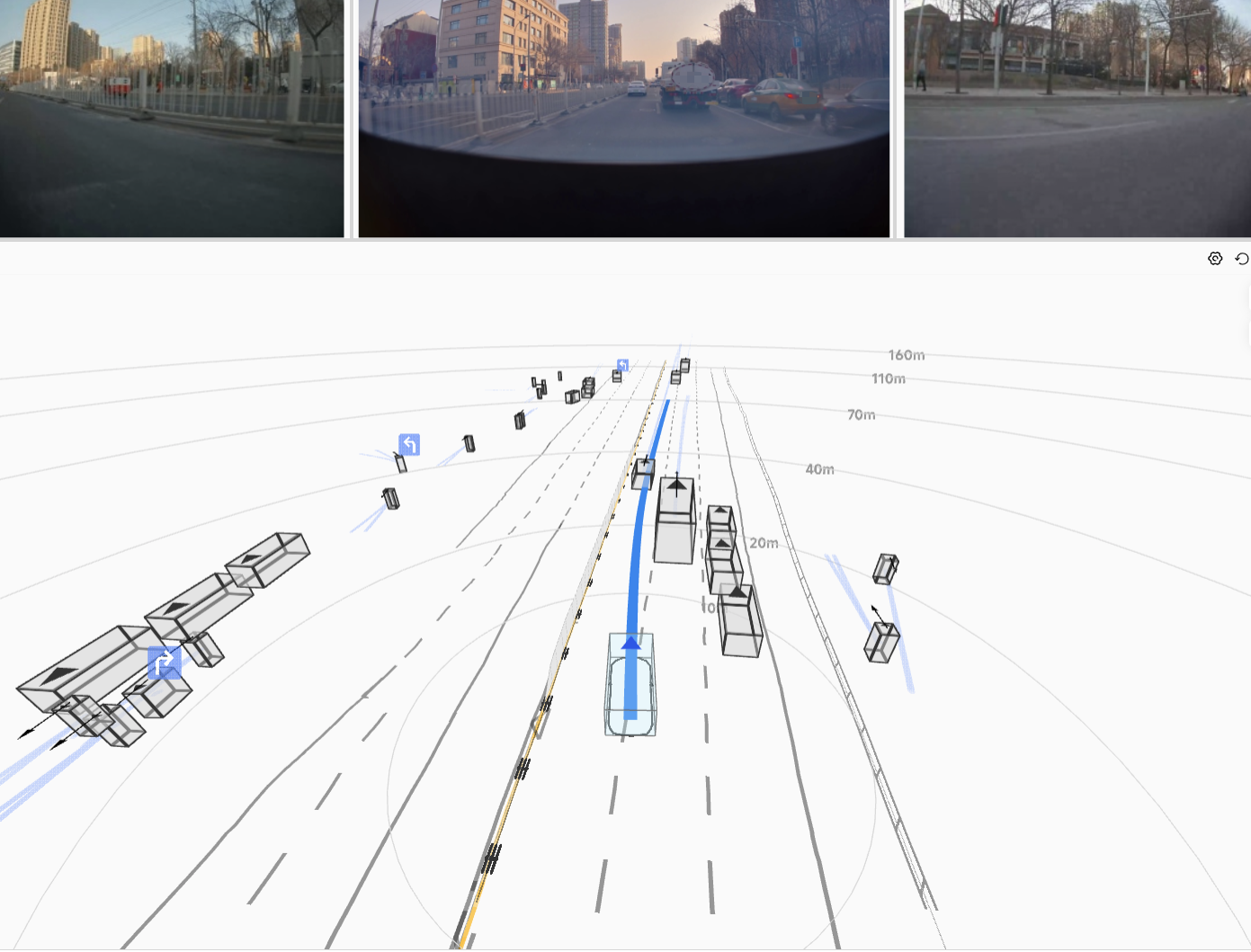}\hfill
\includegraphics[width=0.24\textwidth]{assets/real/real_4d.png}\\[2pt]
\includegraphics[width=0.24\textwidth]{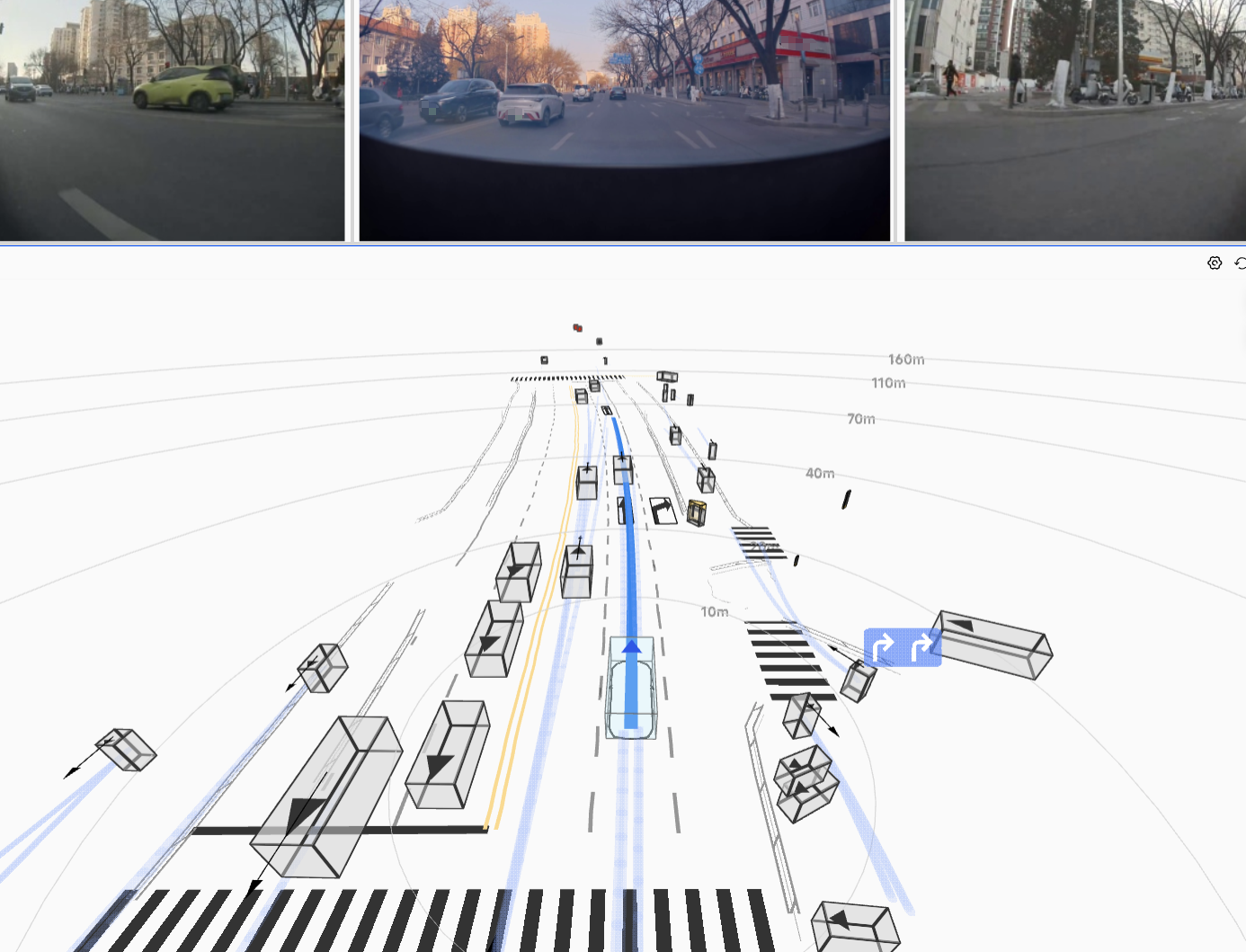}\hfill
\includegraphics[width=0.24\textwidth]{assets/real/real_5b.png}\hfill
\includegraphics[width=0.24\textwidth]{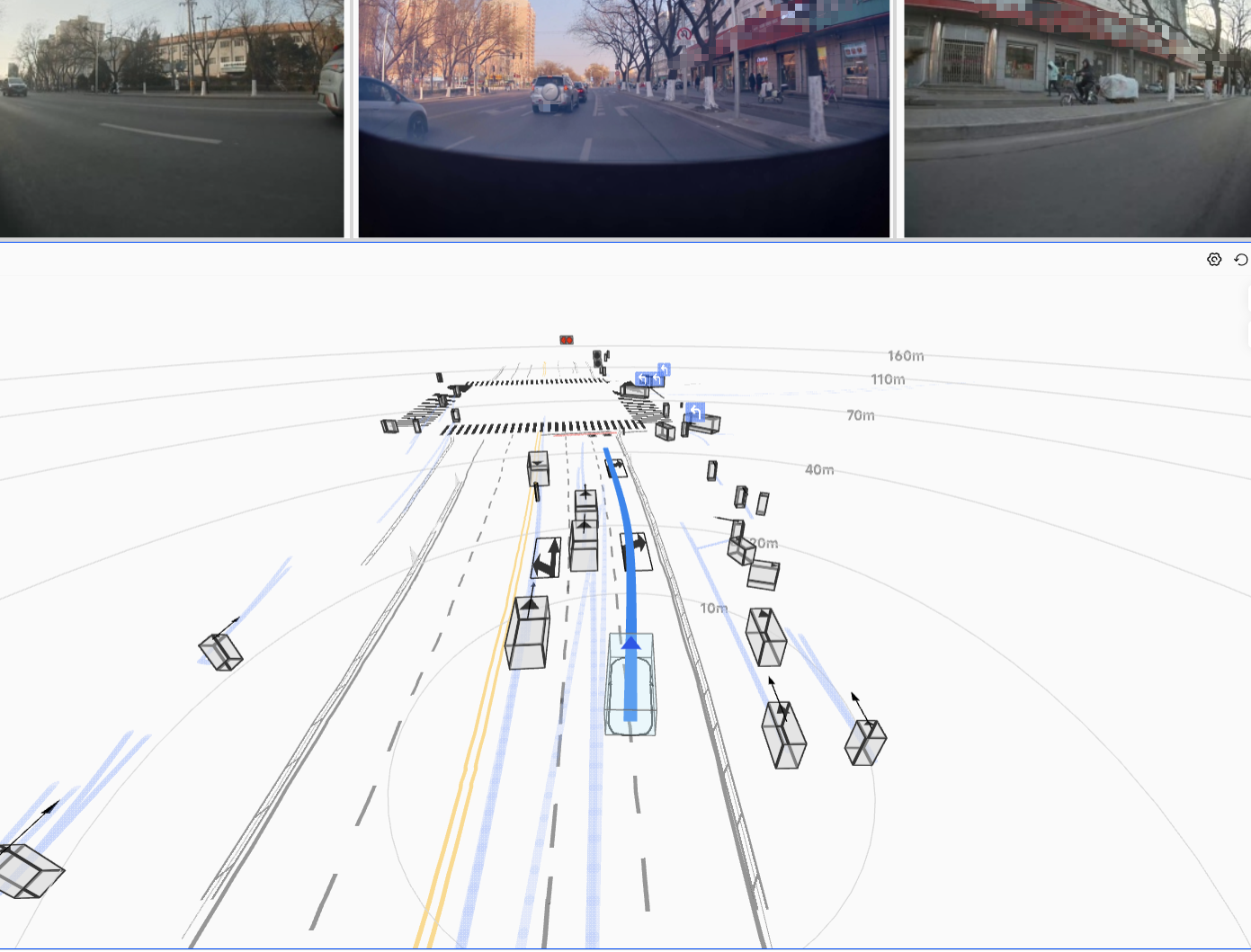}\hfill
\includegraphics[width=0.24\textwidth]{assets/real/real_5d.png}\\[2pt]
\includegraphics[width=0.24\textwidth]{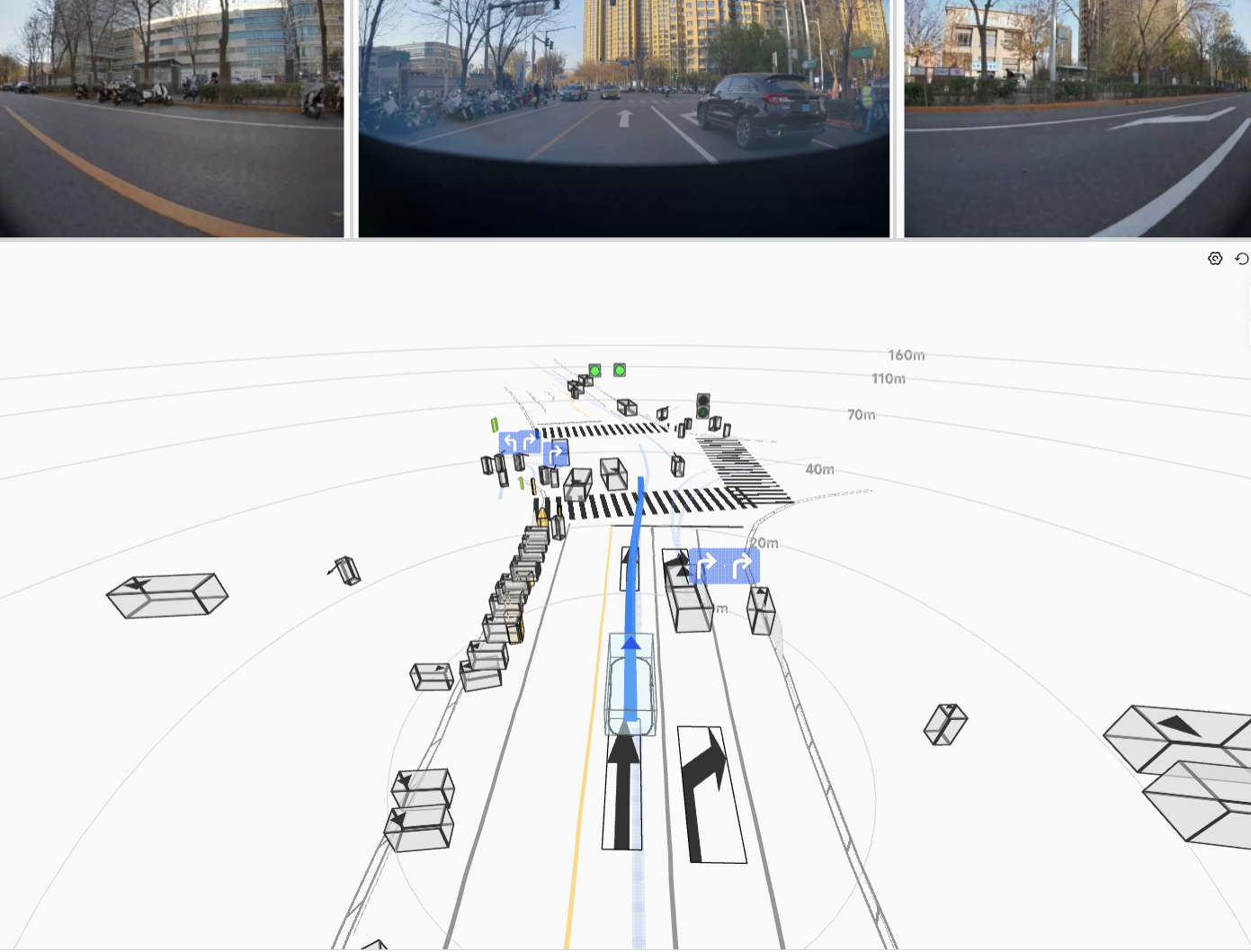}\hfill
\includegraphics[width=0.24\textwidth]{assets/real/real_10b.png}\hfill
\includegraphics[width=0.24\textwidth]{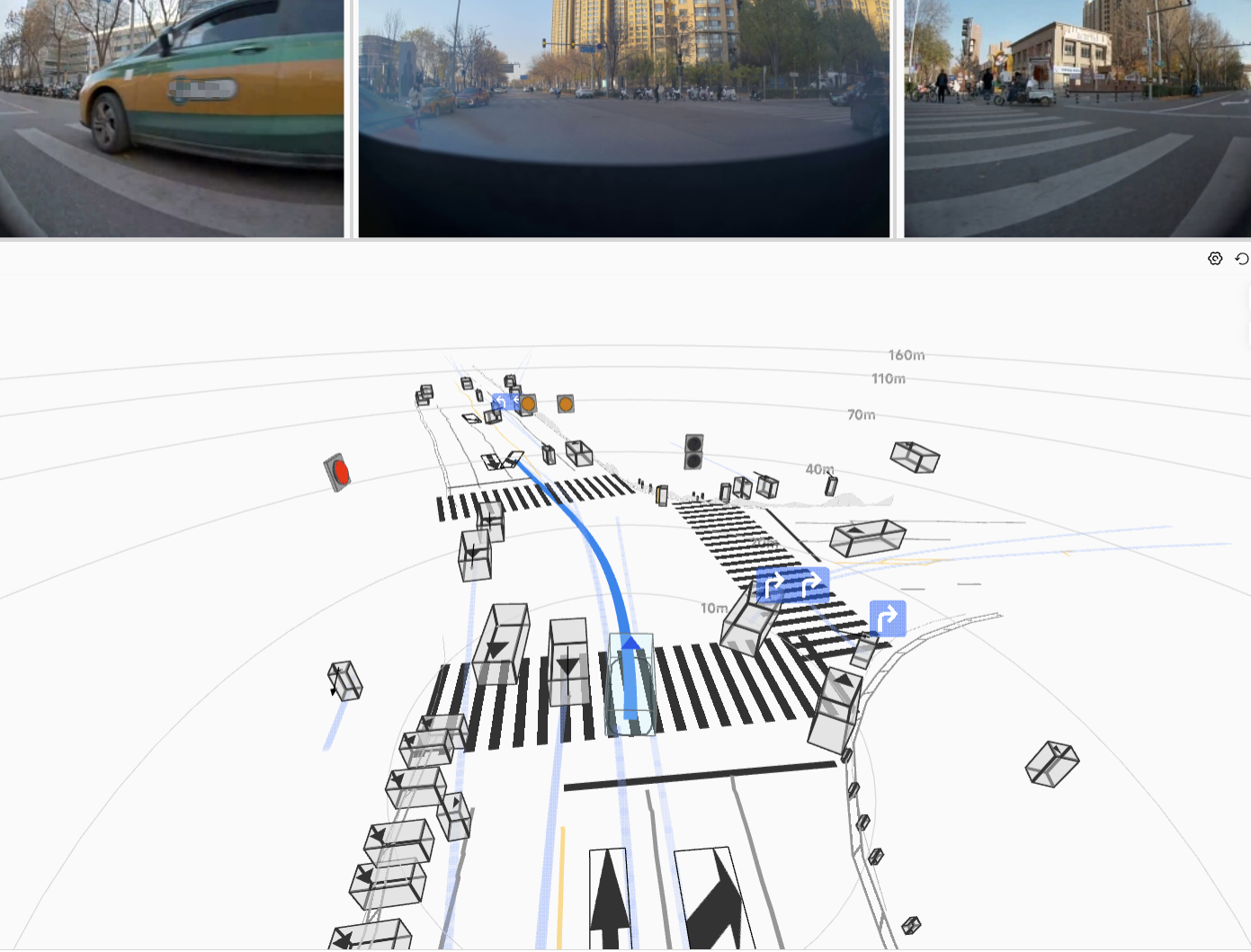}\hfill
\includegraphics[width=0.24\textwidth]{assets/real/real_10d.png}\\[2pt]
\includegraphics[width=0.24\textwidth]{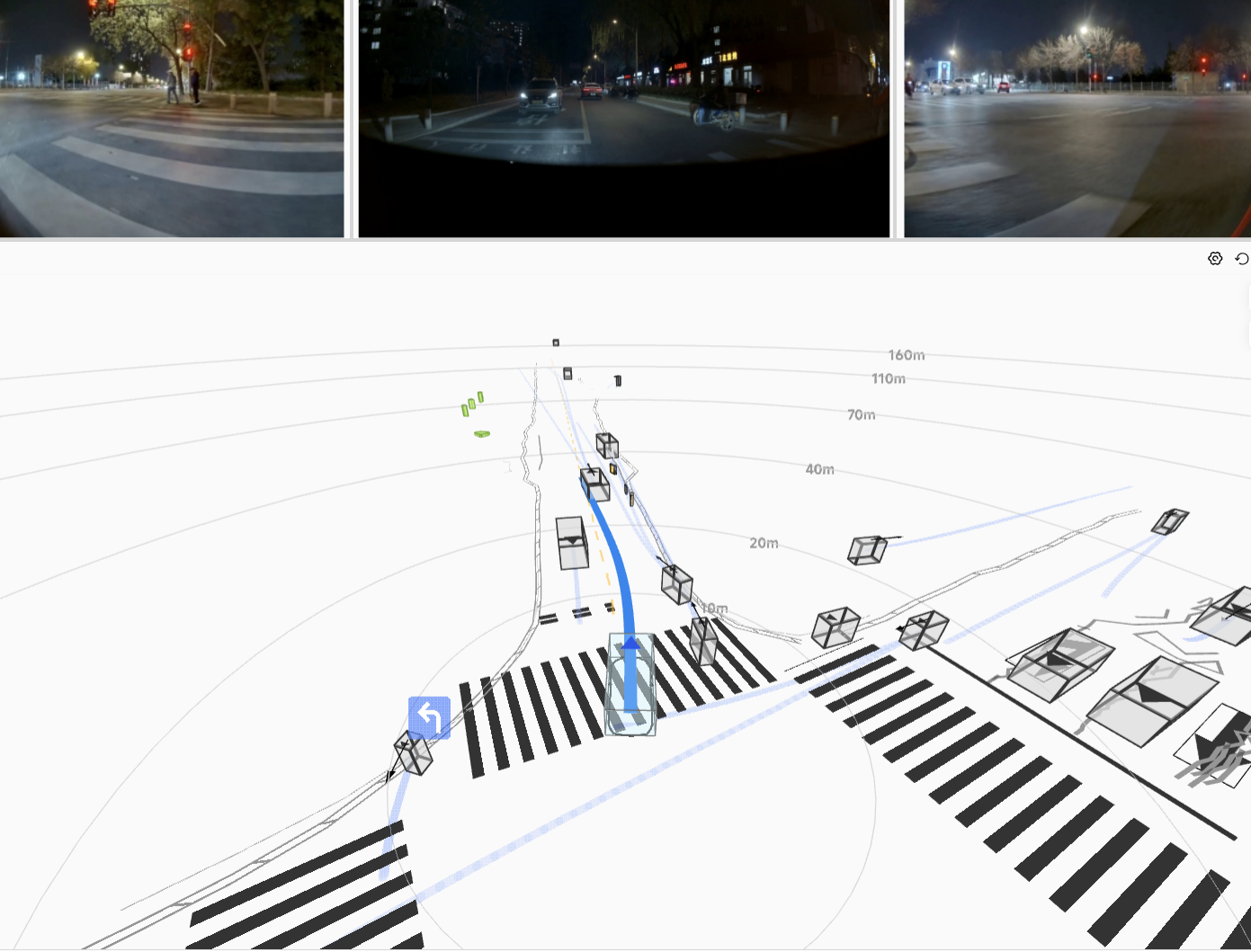}\hfill
\includegraphics[width=0.24\textwidth]{assets/real/real_11b.png}\hfill
\includegraphics[width=0.24\textwidth]{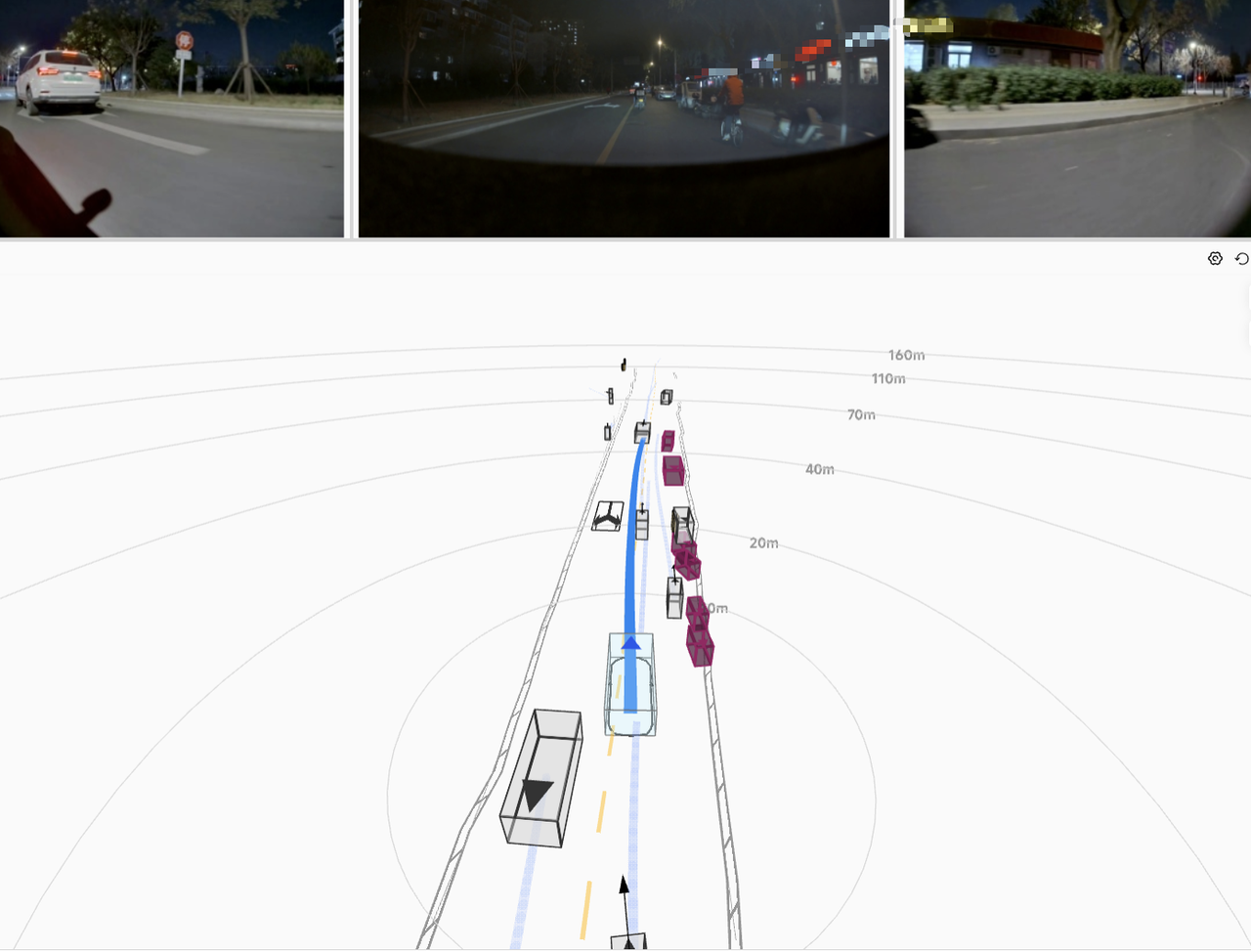}\hfill
\includegraphics[width=0.24\textwidth]{assets/real/real_11d.png}
\caption{\small Closed-loop real-vehicle testing results. Each row contains representative frames from the scenario.}
\label{fig:real_case_ap}
\end{figure}

\section{Related Works}

\textbf{Diffusion Model}.
Diffusion models~\citep{ho2020denoising, sohldickstein2015deep} learn 
complex data distributions through a forward noising process and a 
learned reverse denoising process. They have achieved remarkable 
performance in image and video generation~\citep{betker2023improving, 
croitoru2023diffusion, liu2024sora, rombach2022high}, and have more 
recently been adopted as expressive policy classes in robotics 
control~\citep{chi2023diffusion, black2024pi0visionlanguageactionflowmodel, 
intelligence2025pi_, liurdt}. Their ability to capture multimodal action 
distributions makes them a natural fit for autonomous driving (AD), where 
human driving behavior is similarly diverse and stochastic, and has 
motivated a wave of recent attempts to apply diffusion to AD planning.

\textbf{End-to-End Autonomous Driving}. End-to-end autonomous driving (E2E AD) directly maps raw sensor inputs to 
trajectory or control outputs through deep neural 
networks~\citep{bojarski2016end, chen2023end}. Modern E2E systems 
integrate perception, prediction, and planning into a single 
differentiable model~\citep{hu2023planning, jiang2023vad}, and frequently 
rely on auxiliary supervision on intermediate representations or 
trajectories to stabilize training~\citep{bansal2018chauffeurnet}. Most existing approaches are evaluated on open-loop datasets~\citep{nuscenes2019} or simulation-based benchmarks~\citep{caesar2021nuplan, Dauner2024NEURIPS, liuskill}. However, open-loop metrics are widely acknowledged to be poor proxies for true closed-loop performance on the road~\citep{Dauner2023CORL, li2024ego, zheng2025diffusionplanner}, while simulation-based settings exhibit a non-trivial sim-to-real gap. Conclusions drawn from these benchmarks 
therefore offer limited insight into real-world deployment, which 
motivates the on-road evaluation pursued in this work.

\textbf{Diffusion-Based Planning for Autonomous Driving}. \citet{zheng2025diffusionplanner} make a pioneering attempt by applying diffusion models to AD planning, although their formulation still relies on vectorized scene representations. \citet{diffusiondrive} introduce a truncated denoising process; however, this modification disrupts the original diffusion mechanism, and their model continues to depend heavily on pre-defined trajectory anchors for trajectory generation. \citet{wang2025alpamayo} develop a vision-language-action (VLA) model with a diffusion-based decoder, leveraging large language priors to generate trajectories with strong reasoning capability. \citet{tan2025flow} and \citet{li2025discrete} explore alternative parameterizations such as flow matching and discrete diffusion in the AD context. Despite strong benchmark numbers, many of these methods still rely on hand-crafted priors—rule-based post-processing~\citep{fan2018baidu}, pre-defined anchor trajectories~\citep{li2024hydra}, or explicit goal conditioning~\citep{albrecht2021interpretable, gu2021densetnt}—to remain competitive, which makes it difficult to attribute the reported gains to diffusion modeling itself and raises concerns about transfer to real-world deployment.

\textbf{Diffusion-Based Reinforcement Learning}.
Building upon imitation-pretrained diffusion models, several lines of work 
apply RL to further align generated outputs with task-specific objectives. 
One direction optimizes a differentiable reward or value function by 
back-propagating gradients through the denoising 
process~\citep{xu2023imagereward, clark2023directly}; however, these 
gradients are typically noisy and lead to unstable training. A second 
direction~\citep{ren2024diffusion} treats each denoising step as a 
Gaussian transition and applies policy-gradient methods such as 
PPO~\citep{schulman2017proximal} to the diffusion process, but this 
incurs substantial computational overhead. \citet{li2025recogdrive} 
extend this idea from~\citet{black2023training} to AD planning, while 
weighted-regression-based approaches~\citep{lee2023aligning, 
kang2023efficient, zheng2024safe, zhengtowards} offer simpler 
alternatives that retain the supervised structure of diffusion training. 
More recently, \citet{liang2026dipole} propose a dichotomous policy 
optimization method and fine-tune a one-billion-parameter diffusion-based 
VLA model for AD with stable training. Despite this rapid progress, most existing methods are validated only in simulation, and their real-vehicle effectiveness under safety constraints and out-of-distribution scenarios remains largely unexplored. Our work fills this gap, showing that an RL post-trained diffusion planner delivers measurable closed-loop gains in real-vehicle deployment.

\section{Theoretical Analysis}

\label{app:b}
In this section, we provide details on the conversion between different types of diffusion losses and predictions, as well as the proofs of the theorems.

\subsection{Diffusion Loss Space}
\label{sec:diffusion_loss_space}

Following the standard diffusion formulation~\citep{ho2020denoising, song2021scorebased}, the forward process diffuses a clean trajectory $\tau_0$ into a noised sample
\begin{equation}
    \tau_t = \alpha_t \tau_0 + \sigma_t \epsilon, \qquad \epsilon \sim \mathcal{N}(0, I),
\end{equation}
where $\{\alpha_t, \sigma_t\}$ is a predefined noise schedule and $t \in [0, T]$ indexes the diffusion timestep. A diffusion model is then trained to predict one of three equivalent quantities: the clean trajectory $\tau_0$, the injected noise $\epsilon$, or the velocity $v_t \!=\! \alpha_t \epsilon - \sigma_t \tau_0$. Given $(t, \alpha_t, \sigma_t, \tau_t)$, these three quantities are linearly inter-convertible, so any one of them is sufficient to recover the other two.

This convertibility decouples what the network predicts (the \textbf{parameterization}) from the space in which the loss is computed (the \textbf{loss space}), yielding a $3 \times 3$ design grid. The diagonal entries correspond to the canonical $\tau_0$-, $v_t$-, and $\epsilon$-objectives, where the parameterization and the loss space coincide, while the off-diagonal entries first map the network output to the target space via the closed-form transforms in Table~\ref{tab:diffusion_loss} and then compute the squared error. As an example, parameterizing the network as a clean-trajectory predictor $\tau_\theta$ but computing the loss in noise space gives
\begin{equation}
    \label{eq:diffusion_loss_instance}
    \mathcal{L} = \mathbb{E}_{\tau_0, t, \epsilon}\left\lVert \frac{\tau_t - \alpha_t \tau_\theta}{\sigma_t} - \epsilon \right\rVert^2.
\end{equation}

Although the nine combinations share the same underlying score-matching objective, they are not numerically equivalent during training: each combination induces a different implicit weighting over noise levels. As a result, the choice of (parameterization, loss space) materially affects optimization stability, sample quality, and the relative emphasis the model places on coarse vs.\ fine-grained trajectory details. We therefore treat parameterization and loss space as two orthogonal design axes and study them jointly in our experiments.
\begin{table*}[t]
    \centering
    \normalsize
    \renewcommand{\arraystretch}{1.2}
    \setlength{\tabcolsep}{6pt}
    \caption{\small Mutual conversions among the three diffusion quantities. Columns index the model parameterization and rows index the loss space. Predicted quantities are denoted with $\hat{\,\cdot\,}$ and the network is parameterized by $\theta$. Diagonal entries are the canonical objectives; off-diagonal entries first transform the network output into the target space and then compute the squared error.}
    \label{tab:diffusion_loss}
    \scalebox{1.0}{
        \begin{tabular}{c | c c c}
        \toprule
        & \textbf{${\tau}_0$-pred.}
        & \textbf{${v}_t$-pred.}
        & \textbf{${\epsilon}$-pred.} \\
        \midrule
        \textbf{$\tau_0$-loss: $\mathbb{E}\|\hat{\tau}_0 - \tau_0\|^2$}
        & $\hat{\tau}_0 = \tau_\theta$
        & $\hat{\tau}_0 = \alpha_t \tau_t - \sigma_t v_{\theta;t}$
        & $\hat{\tau}_0 = (\tau_t - \sigma_t \epsilon_\theta)/\alpha_t$ \\
        \textbf{$v_t$-loss: $\mathbb{E}\|\hat{v}_t - v_t\|^2$}
        & $\hat{v}_t = (\alpha_t \tau_t - \tau_\theta)/\sigma_t$
        & $\hat{v}_t = v_{\theta;t}$
        & $\hat{v}_t = (\epsilon_\theta - \sigma_t \tau_t)/\alpha_t$ \\
        \textbf{$\epsilon$-loss: $\mathbb{E}\|\hat{\epsilon} - \epsilon\|^2$}
        & $\hat{\epsilon} = (\tau_t - \alpha_t \tau_\theta)/\sigma_t$
        & $\hat{\epsilon} = \sigma_t \tau_t + \alpha_t v_{\theta;t}$
        & $\hat{\epsilon} = \epsilon_\theta$ \\
        \bottomrule
        \end{tabular}
    }
\end{table*}

% equation labeling in the restated thm is erased here to avoid label chaos in the context

\subsection{Proofs of Lemma~\ref{lem:bregman} and Theorem~\ref{thm:general_score_matching}}
\label{app:proof_bregman}
\label{app:proof_general_score}

In this subsection, we prove that the hybrid loss in Eq.~(\ref{eq:hybrid_loss}) is a valid diffusion score-matching objective whose unique minimizer recovers the marginal score. We first establish a general Bregman-divergence lemma (Lemma~\ref{lem:bregman}), then apply it to identify the $P$-norm structure of the hybrid loss and its score-matching minimizer (Theorem~\ref{thm:general_score_matching}).

\BregmanLemma*

\begin{proof}
\textbf{(i) $D_{P}$ is a Bregman divergence.} The potential $\Phi_{P}(u)=u^{\top}\!P u$ is strictly convex (since $P\succ 0$) with $\nabla\Phi_{P}(u)=2Pu$, and direct expansion gives
\begin{equation*}
    \Phi_{P}(u)-\Phi_{P}(v)-\langle\nabla\Phi_{P}(v),u-v\rangle
    =(u-v)^{\top}P(u-v)
    =D_{P}(u,v).
\end{equation*}

\textbf{(ii) Minimizer.} Conditioning on $\tau^{\mathbf{v}}_{t}$ and expanding the quadratic form in $f:=\tau^{\mathbf{v}}_{\theta}(\tau^{\mathbf{v}}_{t},t)$,
\begin{equation*}
    \mathbb{E}\!\left[D_{P}(f,\tau^{\mathbf{v}}_{0})\,\big|\,\tau^{\mathbf{v}}_{t}\right]
    = f^{\top}\!Pf - 2f^{\top}\!P\,\mathbb{E}[\tau^{\mathbf{v}}_{0}\mid\tau^{\mathbf{v}}_{t}] + C(\tau^{\mathbf{v}}_{t}),
\end{equation*}
with $C$ independent of $f$. The first-order condition yields $\tau^{\mathbf{v},\,\star}_{\theta}(\tau^{\mathbf{v}}_{t},t)=\mathbb{E}[\tau^{\mathbf{v}}_{0}\mid\tau^{\mathbf{v}}_{t}]$, and the Hessian $2P\succ 0$ makes this the unique minimizer.

\textbf{(iii) Connection to the score.} By Tweedie's formula applied to $\tau^{\mathbf{v}}_{t}=\alpha_{t}\tau^{\mathbf{v}}_{0}+\sigma_{t}\epsilon$,
\begin{equation*}
    \mathbb{E}[\tau^{\mathbf{v}}_{0}\mid\tau^{\mathbf{v}}_{t}]
    =\frac{1}{\alpha_{t}}\!\left(\tau^{\mathbf{v}}_{t}+\sigma_{t}^{2}\,\nabla_{\tau^{\mathbf{v}}_{t}}\!\log q_{t}(\tau^{\mathbf{v}}_{t})\right),
\end{equation*}
which is in one-to-one correspondence with the marginal score in Eq.~(\ref{eq:reverse}).
\end{proof}

% \subsection{Proof of Theorem~\ref{thm:general_score_matching}}
% \label{app:proof_general_score}

\begingroup
\renewenvironment{equation}{$$}{$$}
\GeneralScoreMatching*
\endgroup

\begin{proof}
\textbf{(i) $\mathcal{L}_{hybrid}$ as a $\mathit{P}$-norm.} The two terms of the hybrid loss admit matrix forms (using linearity of $M$ for the second):
\begin{equation*}
    \mathcal{L}_{velocity}=\mathbb{E}\!\left[(\tau^{\mathbf{v}}_{\theta}-\tau^{\mathbf{v}}_{0})^{\top}I(\tau^{\mathbf{v}}_{\theta}-\tau^{\mathbf{v}}_{0})\right],
    \quad
    \mathcal{L}_{waypoints}=\mathbb{E}\!\left[\Delta t^{2}(\tau^{\mathbf{v}}_{\theta}-\tau^{\mathbf{v}}_{0})^{\top}M^{\top}\!M(\tau^{\mathbf{v}}_{\theta}-\tau^{\mathbf{v}}_{0})\right].
\end{equation*}
Plugging into Eq.~(\ref{eq:hybrid_loss}),
\begin{equation*}
    \mathcal{L}_{hybrid}
    = \mathbb{E}\!\left[(\tau^{\mathbf{v}}_{\theta}-\tau^{\mathbf{v}}_{0})^{\top}\!\big(I+\omega\Delta t^{2}M^{\top}\!M\big)(\tau^{\mathbf{v}}_{\theta}-\tau^{\mathbf{v}}_{0})\right]
    = \mathbb{E}\!\left[||\tau^{\mathbf{v}}_{\theta}-\tau^{\mathbf{v}}_{0}||_{\mathit{P}}^{2}\right],
\end{equation*}
with $\mathit{P}:=I+\omega\Delta t^{2}M^{\top}\!M$.

\textbf{(ii) $\mathit{P}\succ 0$.} $\mathit{P}$ is symmetric, and for any $z\neq 0$,
\begin{equation*}
    z^{\top}\mathit{P}z=||z||_{2}^{2}+\omega\Delta t^{2}||Mz||_{2}^{2}>0.
\end{equation*}
\textbf{(iii) Minimizer.} With $\mathit{P}\succ 0$, $\mathcal{L}_{hybrid}$ matches the conditional regression objective in Lemma~\ref{lem:bregman}, so its minimizer is $\tau^{\mathbf{v},\,\star}_{\theta}(\tau^{\mathbf{v}}_{t},t)=\mathbb{E}[\tau^{\mathbf{v}}_{0}\mid\tau^{\mathbf{v}}_{t}]$, which by Tweedie equals the marginal score in Eq.~(\ref{eq:reverse}).
\end{proof}

\subsection{Proof of Theorem~\ref{theorem:fisor}}
\label{app:proof_awr_hybrid}

We formalize the RL post-training as a Markov Decision Process (MDP)~\citep{sutton1998reinforcement} $\mathcal{M} = (\mathcal{S}, \mathcal{A}, \mathcal{P}, r, \gamma)$, where the state $s\in\mathcal{S}$ is the driving scene, the action $a \equiv \tau^{\mathbf{v}}_0 \in \mathcal{A}$ is the velocity-parameterized trajectory, and the policy $\pi^k(a\mid s)$ is parameterized by the diffusion model $\tau^{\mathbf{v};k}_\theta$ at iteration $k$. We update $\pi^k$ from the previous policy $\pi^{k-1}$ by solving the KL-regularized objective in Eq.~(\ref{eq:optim}):
\begin{equation*}
    \max_{\pi^k}\; \mathbb{E}_{s \sim \mathcal{D}}\!\left[\mathbb{E}_{a \sim \pi^k(\cdot \mid s)}[r(s,a)] - \tfrac{1}{\beta}\, D_{\mathrm{KL}}\!\big(\pi^k(\cdot\mid s)\,\big\|\,\pi^{k-1}(\cdot\mid s)\big)\right],
\end{equation*}
subject to the simplex constraint $\int_a \pi^k(a\mid s)\,\mathrm{d}a = 1$ for all $s$. With state density $d(s)$ from $\mathcal{D}$ and multiplier $\alpha_s$ for the simplex constraint, the Lagrangian is
\begin{equation*}
\begin{aligned}
    \mathcal{L}(\pi^k, \alpha) &= \int_s d(s)\!\left[\int_a \pi^k(a\mid s)\, r(s,a)\, \mathrm{d}a - \tfrac{1}{\beta}\!\int_a \pi^k(a\mid s)\,\log\!\tfrac{\pi^k(a\mid s)}{\pi^{k-1}(a\mid s)}\, \mathrm{d}a\right] \mathrm{d}s \\
    &\qquad + \int_s \alpha_s\!\left(\int_a \pi^k(a\mid s)\,\mathrm{d}a - 1\right) \mathrm{d}s.
\end{aligned}
\end{equation*}
Taking the derivative with respect to $\pi^k(a\mid s)$ and setting it to zero yields
\begin{equation*}
    r(s,a) - \tfrac{1}{\beta}\!\left(\log\!\tfrac{\pi^k(a\mid s)}{\pi^{k-1}(a\mid s)} + 1\right) + \tfrac{\alpha_s}{d(s)} = 0.
\end{equation*}
Solving for $\pi^k$ and absorbing the $a$-independent term $\exp(\beta\alpha_s/d(s)-1)$ into the normalization gives the closed-form optimum~\citep{nair2020awac}
\begin{equation*}
    \pi^{k^\star}(a\mid s) \;\propto\; \pi^{k-1}(a\mid s)\cdot \exp(\beta\, r(s,a)),
\end{equation*}
matching Eq.~(\ref{eq:closedform}). Next, we show that the hybrid loss can be expanded into a reward-weighted diffusion loss whose minimizer recovers this $\pi^{k^\star}$.

\OptimalPolicy*

\begin{proof}
    Here we consider the velocity of future trajectory as action $a = \tau^\mathbf{v}_0$. To prove that we can sample from the optimal policy in Eq.~\ref{eq:closedform}, we only need to show that the reward-weighted objective in Eq.~\ref{eq:awr_hybrid} is equivalent to the corresponding score matching objective of the optimal policy distribution. For simplicity, we omit the state condition in the derivation.
    \begin{equation}
        \label{eq:proof_of_diffusion_loss}
        \begin{aligned}
            &\mathbb{E}_{\tau_0^\mathbf{v}\sim\pi^{k-1},\epsilon\sim p_\epsilon(\epsilon),t\sim p_t(t)}\left[ \exp(\beta r)||\tau^{\mathbf{v};k}_\theta-\tau_0^\mathbf{v}||_P^2 \right] \\
            &=\int_{\tau^\mathbf{v}_0,\epsilon,t} \exp(\beta r)||\tau^{\mathbf{v};k}_\theta - \tau^\mathbf{v}_0||_P^2 \cdot \pi^{k-1}(\tau^\mathbf{v}_0)p_\epsilon(\epsilon) p_t(t) \mathrm{d}\epsilon \mathrm{d}t \mathrm{d}\tau^\mathbf{v}_0\\
            &\begin{split}
                &= Z \int_{\tau^\mathbf{v}_0,\epsilon,t} 
                   \|\tau^{\mathbf{v};k}_\theta - \tau^\mathbf{v}_0\|_P^2 
                   \cdot \frac{\exp(\beta r)\pi^{k-1}(\tau^\mathbf{v}_0)}{Z} \\
                &\qquad \cdot p_\epsilon(\epsilon)\, p_t(t)
                   \,\mathrm{d}\epsilon \,\mathrm{d}t \,\mathrm{d}\tau^\mathbf{v}_0
            \end{split} \\
            &=Z\int_{\tau^\mathbf{v}_0,\epsilon,t} ||\tau^{\mathbf{v};k}_\theta - \tau^\mathbf{v}_0||_P^2 \cdot \pi^{k^\star}(\tau^\mathbf{v}_0)\cdot p_\epsilon(\epsilon) p_t(t)  \mathrm{d}\epsilon \mathrm{d}t \mathrm{d}\tau^\mathbf{v}_0\\
            &=Z\mathbb{E}_{\tau^\mathbf{v}_0\sim\pi^{k^\star},\epsilon\sim p_\epsilon(\epsilon),t\sim p_t(t)}\left[||\tau^{\mathbf{v};k}_\theta - \tau^\mathbf{v}_0||_P^2 \right]
        \end{aligned}
    \end{equation}
    where $Z = \int_{\tau^\mathbf{v}_0} \exp(\beta r)\pi^{k-1}(\tau^\mathbf{v}_0) \mathrm{d}\tau^\mathbf{v}_0$ is the normalizing constant. This indicates that the reward-weighted objective is equivalent to the standard score matching objective over the optimal policy, scaled by a constant that does not change the minimizer. By Theorem~\ref{thm:general_score_matching}, this minimizer recovers the marginal score of $\pi^{k^\star}$. Consequently, the diffusion reverse process driven by the learned $\tau^{\mathbf{v};k^\star}_\theta$ samples from $\pi^{k^\star}(a\mid s)$, completing the proof.
\end{proof}

\section{Experimental Details}
\label{app:c}

In this section, we provide the experimental details, including the metrics used for open-loop and closed-loop evaluation, as well as the implementation details of imitation learning pre-training and reinforcement learning post-training.

\subsection{Evaluation Metric Design}
\label{app:metrics}
We consider two types of evaluation metrics: open-loop metrics for assessing trajectory quality, and closed-loop metrics for evaluating performance during real-vehicle testing.

\textbf{Open-Loop Metrics}. To perform a comparable open-loop evaluation, we consider widely adopted open-loop measures and compute a final score as the aggregated open-loop metric. The diffusion model exhibits multi-modal behavior during trajectory generation. We generate $6$ trajectories for evaluation. To reduce randomness, we compute the minADE (minimum average Euclidean distance between each predicted trajectory and the ground truth across all waypoints) and minFDE (minimum Euclidean distance between the final waypoint of each predicted trajectory and the ground truth) and calculate the corresponding scores as follows:
        \begin{equation}
            \label{eq:ade_fde_score}
            \begin{aligned}
                S_{ADE} &= 100 \times Clip(1 - \frac{minADE}{Thresh_{ADE}}, 0, 1) \\
                S_{FDE} &= 100 \times Clip(1 - \frac{minFDE}{Thresh_{FDE}}, 0, 1) \\
            \end{aligned}
        \end{equation}
        in which we clip and scale the corresponding scores into $[0, 100]$. To evaluate the comfort and smoothness of the model's generated trajectories, we computed a comfort score as a combination of average acceleration (\textit{Acc}) and jerk (\textit{Jerk}), and calculated the average score over the $6$ trajectories:
        \begin{equation}
            \label{eq:comfort_score}
            \begin{aligned}
            &Cost = \frac{1}{N_1} \sum\limits_{i=1}^{N_1} \left(Cost_{Acc} \times Acc + Cost_{Jerk} \times Jerk\right) \\
                &S_{Comfort} = 100 \times Clip(1 - \frac{Cost}{Thresh_{Comfort}}, 0, 1) 
        \end{aligned}
        \end{equation}
        The final aggregated open-loop score is computed as a weighted sum of the previous metrics scaled by the average collision rate (\textit{CR}):
        \begin{equation}
            \label{eq:aggregated_score}
            \begin{aligned}
                S_{Open-Loop} &= (1 - CR) \times (\sum\limits_{m \in \mathcal{M}} \omega_m S_m) \\
                \mathcal{M} &= \{ADE,~FDE,~Comfort\} 
            \end{aligned}
        \end{equation}
        In addition, to evaluate the multi-modal generation ability of the model, we consider the divergence of each rollout with $64$ generations, measured by the average distance of trajectory endpoints to their geometrical center:
        \begin{equation}
            \begin{aligned}
                Divergence~Score = \frac{1}{N_2} \sum\limits_{i=1}^{N_2} \lVert P^{L}_i - \frac{1}{N_2} \sum\limits_{i=1}^{N_2} P_i^L \lVert_2
            \end{aligned}
        \end{equation}
        where $P^L_i$ is the endpoint of the $i$-th trajectory. The choice of hyperparameters can be found in Table~\ref{tab:metric_hyperparameter}.

\begin{figure}[htbp]
    \centering
    \renewcommand{\arraystretch}{1.2}
    \begin{minipage}[c]{0.55\linewidth}
        \centering
        \captionof{table}{\small Hyperparameters for the open-loop metrics.}
        \label{tab:metric_hyperparameter}
        \begin{tabular}{l c @{\hspace{0.3cm}} l c}
            \toprule
            \textbf{Hyperparameter} & \textbf{Value} & \textbf{Hyperparameter} & \textbf{Value} \\
            \midrule
            $Thresh_{ADE}$     & 4    & $Cost_{jerk}$      & 0.5  \\
            $Thresh_{FDE}$     & 8    & $\omega_{ADE}$     & 0.35 \\
            $Thresh_{Comfort}$ & 200  & $\omega_{FDE}$     & 0.25 \\
            $Cost_{Acc}$       & 1.0  & $\omega_{Comfort}$ & 0.40 \\
            \bottomrule
        \end{tabular}

        \vspace{1em}

        \captionof{table}{\small Hyperparameters for the closed-loop metrics.}
        \label{tab:metric_hyperparameter_closed}
        \begin{tabular}{l c @{\hspace{0.3cm}} l c}
            \toprule
            \textbf{Hyperparameter} & \textbf{Value} & \textbf{Hyperparameter} & \textbf{Value} \\
            \midrule
            $w_1$ & 0.1  & $w_5$             & 0.1 \\
            $w_2$ & 0.25 & $w_6$             & 0.2 \\
            $w_3$ & 0.25 & $Thresh_{center}$ & 40  \\
            $w_4$ & 0.1  & $Thresh_{speed}$  & 40  \\
            \bottomrule
        \end{tabular}
    \end{minipage}
    \hfill
    \begin{minipage}[c]{0.40\linewidth}
        \centering
        \includegraphics[width=\linewidth, angle=90]{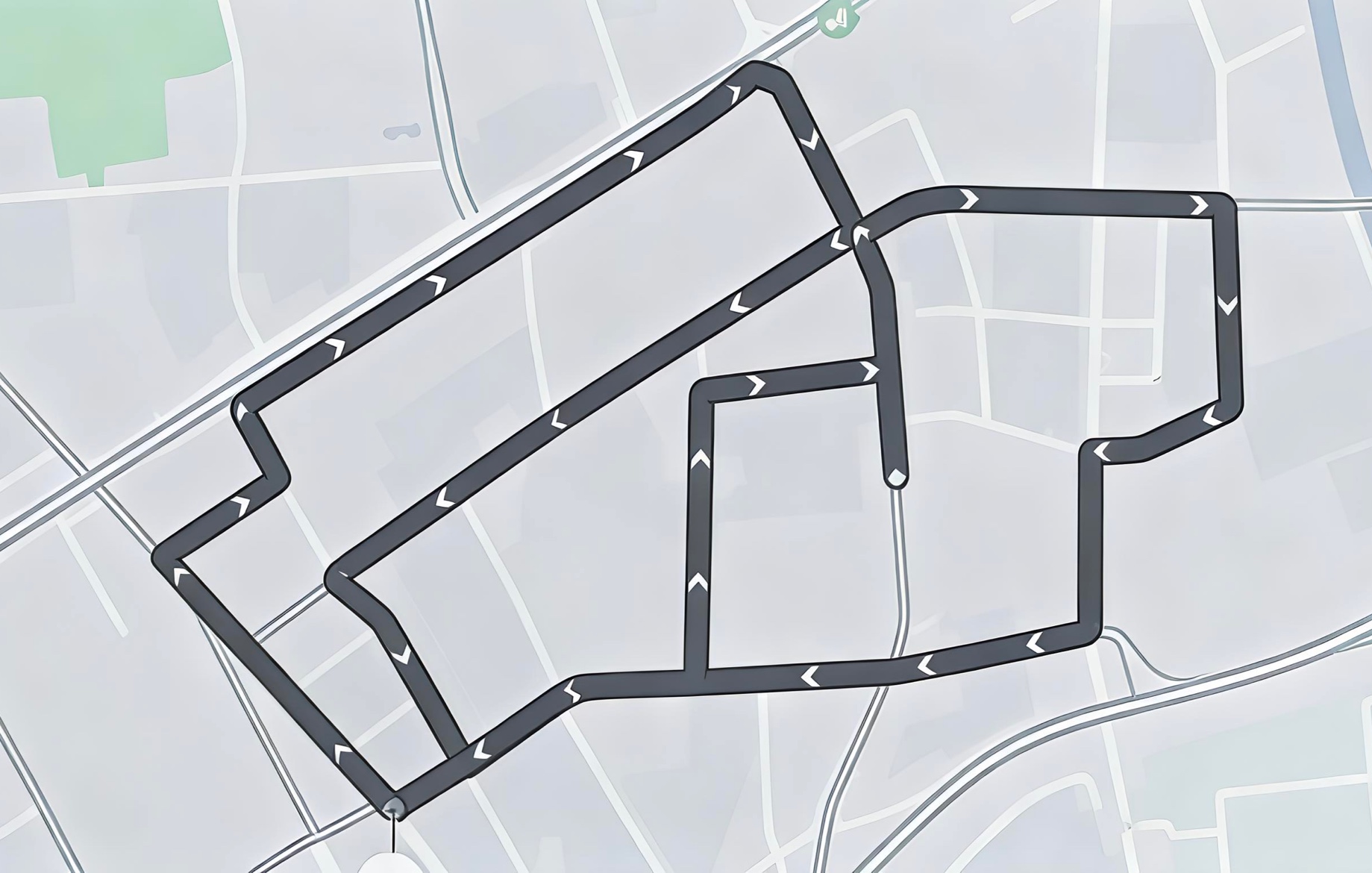}
        \captionof{figure}{\small Real vehicle test route.}
        \label{fig:map}
    \end{minipage}
\end{figure}

        \textbf{Closed-Loop Metrics}. We provide two types of closed-loop metrics: the success rate and the stability score. To ensure a fair comparison, we use a fixed route, as shown in Figure~(\ref{fig:map}), for each model, and each model performs two loops. During the test, we mark specific scenarios, including starting maneuvers ($s_1$), car-following with stopping ($s_2$), navigational lane changes ($s_3$), yielding to VRUs ($s_4$), yielding to cross traffic at intersections ($s_5$), and left and right turns ($s_6$). For each task, a trial is considered a failure if a human takeover occurs; otherwise, it is considered a success. The success rate for each scenario is then calculated. To obtain an overall success rate, we compute a weighted mean across all six scenarios, assigning higher weights to more frequent scenarios for a more accurate evaluation.
        \begin{equation}
            \begin{aligned}
                Success~Rate = &w_1s_1 + w_2 s_2 + w_3 s_3 \\
                &+ w_4  s_4 + w_5  s_5 + w_6  s_6
            \end{aligned}
        \end{equation}

Moreover, we also consider the stability score. Unlike the success rate, this metric is not constrained to specific scenarios. Instead, we evaluate abnormal centering behavior and abnormal speeds, such as driving too slowly or too fast. We record the occurrences of these abnormal behaviors and normalize the counts per 100 km ($k_{center}, k_{speed}$). Afterward, we calculate the scores for centering performance and speed compliance:
\begin{equation}
    \label{eq:centeringvelocity}
    \begin{aligned}
        S_{center} &= 100 \times Clip(1 - \frac{k_{center}}{Thresh_{center}}, 0, 1) \\
        S_{speed} &= 100 \times Clip(1 - \frac{k_{speed}}{Thresh_{speed}}, 0, 1) \\
    \end{aligned}
\end{equation}
Finally, we calculate the overall score using the average of the centering performance and speed compliance scores. The choice of hyperparameters can be found in Table~\ref{tab:metric_hyperparameter}.

\subsection{Scene Encoder Details}
\label{app:scene_encoder}

We consider an E2E AD system that directly processes multi-modal sensory inputs (camera images and LiDAR point clouds) to generate planning trajectories~\citep{chen2023end, hu2023planning}. Since our contribution centers on diffusion-based planning, we adopt an in-house validated perception backbone as the scene encoder $\mathrm{Enc}$ rather than re-engineering perception from scratch. Heterogeneous sensory streams are first compressed into a unified Bird's Eye View (BEV) feature map~\citep{li2024bevformer}, exposing a common spatial frame for downstream perception and planning. On top of the BEV features, we employ two distinct sets of transformer-based queries: Object Detection (OD) tokens for vehicle and pedestrian localization, and Lane Detection (LD) tokens for road structure understanding; the encoder is pretrained on these perception tasks to provide a solid initialization before being plugged into the planner. At planning time, the OD and LD tokens are concatenated with Navi tokens (encoding high-level navigation commands) and further fused via self-attention blocks, yielding the latent condition $C := \mathrm{Enc}(s)$ consumed by the diffusion decoder.

\subsection{Hybrid Loss Implementations}
\label{app:hybrid loss}

The pseudocode for hybrid loss with detach is shown in Algorithm~\ref{alg:hybrid_loss}, implemented in torch. In practice, we set $W=L-1$. We further illustrate the stop-gradient operation in hybrid loss with a small example: $L=6$ time steps and detach window $W=3$. The matrix-form expression 

\begin{equation}
\label{eq:detached_integral}
\begin{aligned}\hat{\tau}^{\mathbf{x}}_{\theta} = M_W \tau^{\mathbf{v}}_{\theta}\Delta t + \mathrm{sg}((M-M_W)\tau^{\mathbf{v}}_{\theta}\Delta t)
    \end{aligned}
\end{equation}

\input{code}

\input{code_rl}

The integration matrix $M$ and its $W$-banded version $M_{W}$ are
\begin{equation*}
    M = \begin{pmatrix}
        1 & 0 & 0 & 0 & 0 & 0 \\
        1 & 1 & 0 & 0 & 0 & 0 \\
        1 & 1 & 1 & 0 & 0 & 0 \\
        1 & 1 & 1 & 1 & 0 & 0 \\
        1 & 1 & 1 & 1 & 1 & 0 \\
        1 & 1 & 1 & 1 & 1 & 1
    \end{pmatrix},\qquad
    M_{W} = \begin{pmatrix}
        1 & 0 & 0 & 0 & 0 & 0 \\
        1 & 1 & 0 & 0 & 0 & 0 \\
        1 & 1 & 1 & 0 & 0 & 0 \\
        0 & 1 & 1 & 1 & 0 & 0 \\
        0 & 0 & 1 & 1 & 1 & 0 \\
        0 & 0 & 0 & 1 & 1 & 1
    \end{pmatrix}.
\end{equation*}
The first $W=3$ rows of $M_{W}$ coincide with $M$ (early waypoints fall entirely within the $W$-step horizon), whereas later rows retain only the most recent $W$ entries.

\paragraph{Forward pass.}
Write the predicted velocities as $\tau^{\mathbf{v}}_{\theta}=(v^{1},v^{2},\ldots,v^{6})^{\top}$. The two branches in Eq.~(\ref{eq:detached_integral}) decompose as
\begin{equation*}
    M_{W}\tau^{\mathbf{v}}_{\theta}\Delta t = \begin{pmatrix}
        v^{1} \\ v^{1}+v^{2} \\ v^{1}+v^{2}+v^{3} \\ v^{2}+v^{3}+v^{4} \\ v^{3}+v^{4}+v^{5} \\ v^{4}+v^{5}+v^{6}
    \end{pmatrix}\Delta t,\qquad
    (M-M_{W})\tau^{\mathbf{v}}_{\theta}\Delta t = \begin{pmatrix}
        0 \\ 0 \\ 0 \\ v^{1} \\ v^{1}+v^{2} \\ v^{1}+v^{2}+v^{3}
    \end{pmatrix}\Delta t.
\end{equation*}
Their sum is the standard waypoint integration $M\tau^{\mathbf{v}}_{\theta}\Delta t = \big(v^{1},\,v^{1}{+}v^{2},\,\ldots,\,\textstyle\sum_{j=1}^{6}v^{j}\big)^{\!\top}\!\Delta t$. Since $\mathrm{sg}(\cdot)$ is the identity in forward, $\hat{\tau}^{\mathbf{x}}_{\theta} = M\tau^{\mathbf{v}}_{\theta}\Delta t = \tau^{\mathbf{x}}_{\theta}$ numerically: \emph{the loss value is the same as without detach}.

\paragraph{Backward pass.}
Only the $M_{W}\tau^{\mathbf{v}}_{\theta}\Delta t$ branch contributes to the gradient (the second branch is wrapped in $\mathrm{sg}$). The waypoint loss $\mathcal{L}_{waypoints}=\sum_{i=1}^{L}(\hat{x}^{i}_{\theta}-x^{i}_{0})^{2}$ therefore yields
\begin{equation}
    \frac{\partial\mathcal{L}_{waypoints}}{\partial v^{k}}
    \;=\; 2\Delta t\!\!\sum_{i\,:\,(M_{W})_{ik}=1}\!\!(\hat{x}^{i}_{\theta}-x^{i}_{0})
    \;=\; 2\Delta t\!\!\sum_{i=k}^{\min(k+W-1,\,L)}\!\!(\hat{x}^{i}_{\theta}-x^{i}_{0}).
    \label{eq:detached_grad}
\end{equation}
Concretely, for $L=6, W=3$ the gradient of $\mathcal{L}_{waypoints}$ w.r.t.\ each velocity component picks up residuals only from the listed waypoints:
\begin{center}
\begin{tabular}{c|c|c}
\toprule
component & full integration ($M$) & detached ($M_{W}$, $W=3$) \\
\midrule
$v^{1}$ & $\{1,2,3,4,5,6\}$ & $\{1,2,3\}$ \\
$v^{2}$ & $\{2,3,4,5,6\}$   & $\{2,3,4\}$ \\
$v^{3}$ & $\{3,4,5,6\}$     & $\{3,4,5\}$ \\
$v^{4}$ & $\{4,5,6\}$       & $\{4,5,6\}$ \\
$v^{5}$ & $\{5,6\}$         & $\{5,6\}$ \\
$v^{6}$ & $\{6\}$           & $\{6\}$ \\
\bottomrule
\end{tabular}
\end{center}
Without detach, $v^{1}$ accumulates residuals from \emph{all six} future waypoints, whereas $v^{6}$ only sees one---the imbalance grows with $L$. With detach, each $v^{k}$ accumulates at most $W$ residuals, evening out the gradient magnitudes across time steps.

\paragraph{Theoretical compatibility.}
The decomposition in Eq.~(\ref{eq:detached_integral}) only modifies the backward pass; the forward value of $\hat{\tau}^{\mathbf{x}}_{\theta}$, and hence of $\mathcal{L}_{hybrid}$, is identical to the un-detached case. Consequently, the minimizer characterized by Theorem~\ref{thm:general_score_matching} is unchanged: the global minimizer remains the marginal score function. The detach merely reshapes the gradient field for training stabilization.

\subsection{Implementation Details} 
\label{ap:implementation}

\textbf{Imitation Learning}. Based on the content in Section~\ref{sec:exp}, we provide the following details of the experimental setup. We adopt the variance-preserving(VP) noise schedule following \cite{zheng2025diffusionplanner} and use 6 sampling steps for efficient generation. Training was conducted using 64 NVIDIA H20 GPUs, with a batch size of $160$ per GPU over $10$ epochs, with a warmup phase. We use AdamW optimizer with a learning rate of $5\times10^{-4}$, weight decay of $0.01$. We report the other detailed setup in Table \ref{tab:param}. After being well trained, our models are deployed on a real vehicle platform for real-world closed-loop testing. Specifically, the model is first converted to the ONNX format and then optimized using TensorRT's inference compiler to enable hardware-accelerated execution. Furthermore, for multi-step inference, we follow the approach used by \citet{zheng2025diffusionplanner}, which employs the DPM-Solver~\citep{lu2022dpm} to accelerate the sampling process, achieving a final inference speed that easily meets the 10Hz requirement. It is worth noting that we apply only a light post-processing smoothing step after the model output, ensuring that the evaluation accurately reflects the model's inherent performance.

\textbf{Reinforcement Learning}. Given the safety risks of online RL on real vehicles~\citep{kendall2019learning} and the high computational cost of world-model-based simulators~\citep{agarwal2025cosmos, hu2023gaia}, we train with a non-reactive pseudo-closed-loop simulator~\citep{Dauner2024NEURIPS} built on real-world datasets, while still evaluating the trained model on real vehicles. In this setup, neighboring vehicles replay logged behaviors, while our model generates planning trajectories.  To achieve stable training using Eq.~(\ref{eq:awr_hybrid}), we apply reward group normalization~\citep{shao2024deepseekmath} to obtain an appropriate numerical range for weighting. Additionally, we discard samples in which all actions receive identical rewards to improve learning effectiveness. Finally, we employ Exponential Moving Average (EMA) for policy updates to further enhance stability. More details about rewards design are shown in Appendix~\ref{app:rewards}.

%%%%%%%%%%%%%%%%%%%%%%%%%%%%%%%%%%%%%%%%%%%%%%%%%%%%%%%%%%%%

\subsection{Reward Function Details}
\label{app:rewards}

We provide concrete forms of the four rewards used in Section~\ref{sec:rl}. All rewards return a scalar in $[0, 1]$, evaluated on the candidate trajectory over the planning horizon of $L$ steps. The single-reward baseline uses $r_\text{safety}$ alone; the multi-reward setting replaces it with $r_\text{risk}$ and adds $r_\text{follow}$ and $r_\text{lane}$.

\paragraph{Safety Reward $r_\text{safety}$.} Collisions between the ego and each neighboring vehicle are detected via the Separating Axis Theorem (SAT) applied to their oriented bounding boxes at every future timestep $l$. The per-step collision indicator is $c_l = 1.0$ for active (head-on or lateral) collisions, $c_l = 0.3$ for rear-end collisions (attenuated to mitigate non-reactive simulator artifacts), and $c_l = 0$ otherwise. The trajectory-level reward is
\begin{equation}
r_\text{safety} \;=\; 1 - \max_{l=1,\ldots,L} c_l.
\end{equation}

\paragraph{Risk Reward $r_\text{risk}$.} A continuous, near-miss-aware refinement of $r_\text{safety}$. At each timestep $l$ we compute three sub-scores in $[0, 1]$ from time-to-collision (TTC), time-headway (THW), and occupancy distance to static/uncertain regions. Each sub-score is obtained by passing the raw quantity through a speed-adaptive shaping function: $1$ when the geometry is comfortably safe, decaying to $0$ as the configuration becomes critical. We aggregate conservatively---taking the most pessimistic value across both time and signals---yielding
\begin{equation}
r_\text{risk} \;=\; \min_{l=1,\ldots,L,\; m\in\{\text{TTC},\,\text{THW},\,\text{OCC}\}} s^m_l.
\end{equation}
The same rear-end attenuation as in $r_\text{safety}$ is applied so that artifacts of the non-reactive simulator are not over-penalized.
\begin{table}[t]
    \centering
    \caption{\small Hyperparameters of \textit{\name{}} / \textit{\name{}-RL}}
    \begin{tabular}{llcc}
    \toprule
     \textbf{Type}&   \textbf{Parameter}&\textbf{Symbol}  &\textbf{Value}\\
    \midrule
    \multirow{4}{*}[-0.ex]{IL}
        &Num. block &-& 6 \\
        &Dim. hidden layer &-& 256 \\
        &Num. multi-head &-& 8 \\
        &Hybrid loss weight & $\omega$ & 0.1 \\
    \midrule
    \multirow{6}{*}[-0.ex]{RL}
        &Group size &-& 32 \\
        &Temperature & $\beta$ & 1.0 \\
        &EMA &-& 0.05 \\
        &Risk reward weight   & $\lambda_\text{risk}$   & 1.0 \\
        &Car-following reward weight & $\lambda_\text{follow}$ & 3.0 \\
        &Lane-keeping reward weight & $\lambda_\text{lane}$ & 2.5 \\
    \bottomrule
    \end{tabular}
    \label{tab:param}
    % \vspace{-15pt}
\end{table}

\paragraph{Car-Following Reward $r_\text{follow}$.}
When a leader vehicle is present, this reward encourages human-like car-following, following the spirit of classical ACC heuristics. At each timestep $l$ we score the ego on four intuitive aspects, each normalized to $[0,1]$:
\begin{itemize}[leftmargin=*]
    \item \emph{time gap} $s^\text{gap}_l$: how close the current headway (in seconds) is to a speed-dependent ideal -- shorter at low speed, longer at high speed;
    \item \emph{spacing} $s^\text{dist}_l$: how close to the target vehicle, clipped to a safe min/max range;
    \item \emph{speed match} $s^\text{spd}_l$: how well the ego speed tracks the leader's speed;
    \item \emph{comfort} $s^\text{cmf}_l$: a penalty on longitudinal acceleration or braking that exceeds comfortable bounds.
\end{itemize}
The four scores are averaged over time and across aspects:
\begin{equation}
r_\text{follow} \;=\; \frac{1}{4L}\sum_{l=1}^{L}\Big(s^\text{gap}_l + s^\text{dist}_l + s^\text{spd}_l + s^\text{cmf}_l\Big).
\end{equation}

\paragraph{Lane-Keeping Reward $r_\text{lane}$.}
This reward encourages trajectories to stay close to the lane centerline during lane-keeping, improving lateral robustness. For each predicted waypoint we measure its perpendicular distance to the nearest lane centerline and convert it into a score $s^\text{ctr}_l \in [0,1]$ that is $1$ on the centerline and decays linearly to $0$ once the offset exceeds roughly half a lane width. The reward is the temporal average:
\begin{equation}
r_\text{center} \;=\; \frac{1}{L}\sum_{l=1}^{L} s^\text{ctr}_l.
\end{equation}
To avoid penalizing intentional lateral maneuvers, we mask out samples whose expert reference is off-lane or exhibits lane-change behavior.

\paragraph{Total Training Reward.} The reward $r$ used in Eq.~(\ref{eq:awr_hybrid}) is
\begin{equation}
r \;=\;
\begin{cases}
r_\text{safety}, & \text{single-reward baseline,} \\
\lambda_\text{risk}\, r_\text{risk} + \lambda_\text{follow}\, r_\text{follow} + \lambda_\text{lane}\, r_\text{lane}, & \text{multi-reward setting.}
\end{cases}
\end{equation}

\section{Baseline Comparison Results}
\label{app:moreresults}

\subsection{IL Pre-Training Results}
\label{app:ilresults}

On top of the results in Section~\ref{sec:exp}, we select two baselines for comparison: (1) a regression-adaptation of \textit{HDP} that shares the same architecture but is trained with regression loss; (2) an anchor-based planning algorithm with truncated diffusion~\citep{diffusiondrive}. To ensure the maximum capability gain from training, we use the largest dataset (70M frames) for training the two baselines.

\begin{table*}[t]
\centering
\caption{\small Compare the open-loop and closed-loop scores with baselines.}
\vspace{-5pt}
\resizebox{0.95\linewidth}{!}{\scriptsize
\begin{tabular}{l c c c c c} \toprule 
\multirow{2}{*}[-0.8ex]{\makecell[l]{\textbf{Model Name}}} & \multirow{2}{*}[-0.8ex]{\makecell[l]{\textbf{Data Size}}} & \multirow{2}{*}[-0.8ex]{\makecell[l]{\textbf{Open-Loop Score}}} & \multicolumn{3}{c}{\textbf{Closed-Loop Score}} \\ \cmidrule(lr){4-6} 
& & & \textbf{Success Rate} & \textbf{Stability Score} & \textbf{Overall Score} \\ \midrule
Regression Model & XL & 73.98 & 21.81 & 0.00 & 10.91 \\
Base Model & M& 51.07 & 15.67 & 0.00 & 7.83 \\
with $\tau_0$-loss \& $\tau_0$-pred & M& 75.27 & 22.84 & 0.00 & 11.42 \\
HDP & XL& \colorbox{mine}{88.94} & 71.24 & {79.53} & 75.38 \\ \midrule
HDP-RL  & -& - & \colorbox{mine}{83.49} & \colorbox{mine}{84.65} & \colorbox{mine}{84.07}\\
\bottomrule \end{tabular}}
\label{tab:ilbaseline}
\end{table*}

\begin{wrapfigure}{r}{0.45\linewidth} 
\includegraphics[width=\linewidth]{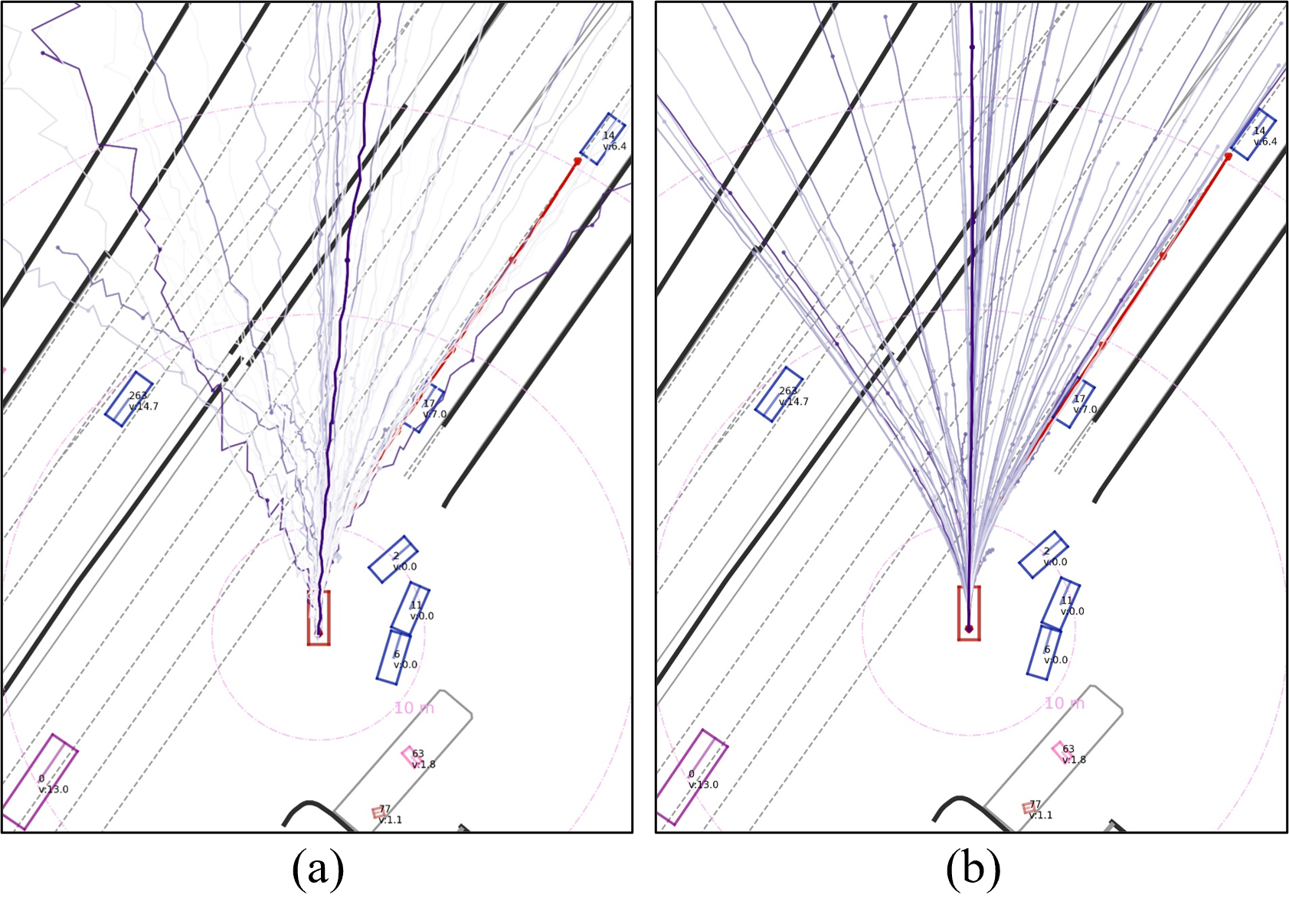}
\vspace{-2pt}
\caption{\small Visualization of anchor-based truncated diffusion using (a) offset prediction and (b) direct prediction. The ground truth is in red and model predictions in purple. The deeper the trajectory is, the higher the predicted score.}
\label{fig:diffusiondrive_openloop}
\vspace{-15pt}
\end{wrapfigure}
\textbf{Regression Adaptation.} The regression model has been used as a baseline method for autonomous driving with compatible benchmark performance~\citep{Chitta2023transfuser, diffusiondrive}. In our settings, the regression adaptation shares the same scene encoder with \textit{HDP}. For the decoder, we replace the embeddings of noised trajectory with a group of learnable queries to extract information from the scene features and to predict the final trajectory. For the model architecture, we retain the DiT-like architecture that uses stacked layers with attention blocks. The model is supervised with $L_2$ loss to predict the future trajectory in waypoint representation. For inference, the model directly outputs one trajectory for each rollout.

We evaluate the regression adaptation model in both open-loop and closed-loop settings, as shown in Table~\ref{tab:ilbaseline}. The model demonstrates compatible performance in terms of open-loop score, benefiting from large-scale training with sufficient data coverage. However, the model does not perform well in real-vehicle closed-loop evaluation, achieving an overall score (10.91) much lower than our \textit{HDP} (75.38). This indicates the inability of simple regression to scale up in real-world closed-loop scenarios, calling for more power generation capability of diffusion models with our proposed improvements.
% indicating the generalization capability in real-world scenarios of diffusion model with our proposed improvements.

\textbf{Anchor-Based Truncated Diffusion.} The model shares the scene encoder with \textit{HDP} as well. For the trajectory decoder, we followed DiffusionDrive~\citep{diffusiondrive} for implementation. Specifically, we first generate $N_{anchor}$ anchors using the training data with k-means algorithm. We select $N_{anchor}=64$ as a balance between mode coverage and compute efficiency compared with $N_{anchor}=20$ used in~\citep{diffusiondrive} for NavSim benchmark, given the larger data volume of ours. The decoder is composed of $N_{layer}=6$ repeated layers, each of which comprises separate attention blocks for scene features and ego-vehicle features, and two MLP-based heads for trajectory prediction and scoring respectively. For the diffusion part, we use DDIM~\citep{songdenoising} with $T=1000$. The model takes in slightly-noised anchors ($t \leq 50$) as well as scene embeddings as input, and predicts a group of trajectory candidates and classification scores in each layer. The output candidates of preceding layers are used as input for subsequent layers to enable multiple forward passes for refinement. During training, we select one positive sample from each of the $N_{layer}$ candidate groups. This sample is chosen based on the anchor closest to the ground truth. For supervision, $L_1$ loss is applied to the $N_{layer}$ trajectories. Additionally, focal loss is applied to the $N_{layer}$ classification scores, using the one-hot encoded index of the closest anchor as the target label. For inference, the model starts from noisy anchors with $t=8$, and the final prediction is selected from the last layer output as the trajectory with the highest score.

In practice, we find a severe performance mismatch compared with benchmark results from DiffusionDrive~\citep{diffusiondrive} as shown in Figure~\ref{fig:diffusiondrive_openloop}. The model's generations exhibit obvious irregular jitter, leading to low quality trajectories. We credit this to the offset prediction in the original implementation, which predicts the offset of noisy anchors from the ground truth:
$\text{offset} = \tau_0 - (\alpha_t \tau_0 + \sigma_t \epsilon)
= (1-\alpha_t)\tau_0 - \sigma_t \epsilon
\approx -\sigma_t \epsilon,~~t \ll T$.
This resembles the $\epsilon$-prediction as discussed in Section~\ref{sec:diffusion_loss}. Instead, we supervise the direct output of the model with the ground truth. This results in better smoothness in the generated trajectories, as shown in Figure~\ref{fig:diffusiondrive_openloop}. However, we still find that the model's predictions are strongly influenced by the anchor and fail to pick the correct candidate, and are unfortunately infeasible for real-vehicle deployment.

\subsection{RL Post-Training Results}
\label{app:rlresults}

% Figure~\ref{fig:rlcase} further 
% illustrates this: compared with \textit{\name{}}, \textit{\name{}-RL$^\dagger$} 
% avoids surrounding vehicles more effectively, confirming the effectiveness 
% of our RL algorithm.

To benchmark our RL-hybrid loss in Eq.~(\ref{eq:awr_hybrid}), we additionally implement a baseline following~\citep{black2023training, ren2024diffusion} (DPPO-style): the denoising chain $\tau_T \!\to\! \cdots \!\to\! \tau_0$ is treated as a $T$-step MDP with per-step Gaussian transitions $\pi_\theta(\tau_{t-1}\!\mid\!\tau_t, s) = \mathcal{N}(\mu_\theta(\tau_t, t, s),\, \Sigma_t)$, and a clipped PPO surrogate is applied at each transition:
\begin{equation}
\label{eq:dppo_loss}
    \mathcal{L} \;=\; -\,\mathbb{E}_{s,\,k,\,t}\!\left[\min\!\big(\rho^{k}_{t} A^{k}_{t},\; \mathrm{clip}(\rho^{k}_{t},\, 1\!\pm\!\epsilon)\, A^{k}_{t}\big)\right],
\end{equation}
where $\rho^{k}_{t}$ is the per-step current/old log-prob ratio and $A^{k}_{t}$ is a within-state advantage normalized over $K$ trajectory candidates per state, with rollouts collected under an EMA-updated old policy. In practice this baseline was fragile to train and required a battery of numerical safeguards. The results are shown in Table~\ref{tab:rlbaseline}. DPPO-style methods exhibit signs of reward hacking: although the stability score improves substantially, the low success rate indicates frequent takeovers, resulting in an unreliable planner. HDP-RL, in contrast, improves both metrics, and the multi-reward setting further enhances performance.

\begin{table*}[h]
\centering
\caption{\small Compare the closed-loop scores with DPPO-style RL methods~\citep{black2023training, ren2024diffusion}.}
\vspace{-5pt}
\resizebox{0.85\linewidth}{!}{\scriptsize
\begin{tabular}{l c c c} \toprule 
\multirow{2}{*}[-0.8ex]{\makecell[l]{\textbf{Model Name}}} & \multicolumn{3}{c}{\textbf{Closed-Loop Score}} \\ \cmidrule(lr){2-4} 
& \textbf{Success Rate} & \textbf{Stability Score} & \textbf{Overall Score} \\ \midrule
HDP-RL$^\dagger$ (only safety reward) & {72.89} & {79.53} & {76.20}\\ \midrule
DPPO-style methods (multi-rewards) & {65.50} & \colorbox{mine}{89.76} & {77.63}\\
HDP-RL (multi-rewards) & \colorbox{mine}{83.49} & {84.65} & \colorbox{mine}{84.07}\\
\bottomrule \end{tabular}}
\label{tab:rlbaseline}
\end{table*}

Let $T$ denote the number of denoising steps, $K$ the number of candidates per state. The per-state cost of one RL update is:
\begin{itemize}[leftmargin=*]
    \item \textbf{HDP-RL (ours):} the buffer holds $K$ trajectory tuples $(s, \tau_0, r)$; each is processed by one forward-backward at a single random $t$. Rollout cost: $K\,T$ no-grad forwards. Optimization cost: $K$ forward-backwards.
    \item \textbf{DPPO-style methods:} the buffer holds $K\,T$ transition tuples (the entire denoising chain for each candidate); each transition contributes one forward-backward to evaluate $\mathcal{L}_{DPPO}$. Rollout cost: $K\,T$ no-grad forwards. Optimization cost: $K\,T$ forward-backwards---a $T\times$ overhead over RL-hybrid.
\end{itemize}

\section{Limitations \& Discussions \& Future Work}
\label{ap:limitations}

Here, we discuss the limitations, potential solutions, and interesting directions for future research.

\begin{itemize}[leftmargin=*]

    \item \textbf{Lack of Trajectory Selection Mechanism}. Our framework directly adopts the trajectory generated by the diffusion model as the final planning output, without an additional selection or scoring stage. The strong closed-loop performance on real vehicles already demonstrates that the generated trajectories are of high quality. Nevertheless, from a practical deployment perspective, an explicit trajectory selection mechanism is often beneficial for filtering out occasional low-quality samples and further improving robustness.

    \textit{Solution and future work:} A common solution is to introduce an auxiliary scoring head, e.g., a learned scorer or a rule-based filter that ranks multiple candidate trajectories and selects the best one~\citep{li2024hydra}. However, the central goal of this work is to thoroughly investigate the inherent capability of diffusion models for E2E AD, and we therefore deliberately avoid introducing additional engineering components that may obscure the contribution of the diffusion model itself. Combining HDP with a learned trajectory scorer is a natural and promising direction for production-oriented systems, and is fully complementary to our findings.

    \item \textbf{Pseudo Closed-Loop Reinforcement Learning}. Our RL post-training is conducted in a non-reactive pseudo closed-loop simulator built upon real-world logged data, where neighboring vehicles replay their recorded behaviors and do not react to the ego vehicle. This inevitably introduces a gap with truly interactive closed-loop training, and may underestimate certain reactive behaviors of the surrounding agents.

    \textit{Solution and future work:} A natural alternative is to perform reactive closed-loop RL within a learned driving world model~\citep{hu2023gaia,agarwal2025cosmos}, so that surrounding agents can respond to the ego vehicle. However, current world models are still extremely expensive to train and run at scale, and their generative quality on truly out-of-distribution scenarios remains limited, making them more suitable as offline evaluators rather than as scalable RL environments at this stage. Despite this limitation, we observe that even with a simple non-reactive pseudo closed-loop simulator, our reward-weighted diffusion RL already yields substantial improvements on real-vehicle tests, validating the effectiveness of the proposed RL post-training framework. Combining HDP-RL with reactive world models is an exciting future direction that could further amplify these gains.

    \item \textbf{Compatibility with VLM/VLA Backbones}. In this work, we instantiate the diffusion planner on top of an in-house perception backbone, and have not investigated how vision--language pretraining or large-scale vision--language--action (VLA) models~\citep{wang2025alpamayo, liang2026dipole} can be combined with our framework. As a result, the world knowledge and reasoning capabilities accumulated by recent VLM/VLA foundations are not yet fully exploited in HDP.

    \textit{Solution and future work:} An important observation is that the contributions of this paper---diffusion loss space, hybrid trajectory representation, data scaling, and the RL-hybrid post-training---are largely orthogonal to the choice of upstream encoder. HDP can therefore be naturally used as a \emph{plug-in diffusion action head} on top of a VLM/VLA backbone, where the planner inherits the rich semantic and commonsense priors from large-scale vision--language pretraining, while still benefiting from our findings on stable, scalable diffusion-based action generation. We view this combination as a particularly promising path towards more generalizable and reasoning-aware E2E driving policies, and leave a thorough investigation as future work.

    \item \textbf{Limited Simulation-Based Evaluation}. Our evaluation primarily relies on offline open-loop replay metrics and large-scale closed-loop real-vehicle road tests. This makes our results harder to compare side-by-side with prior diffusion-based planners that are mainly evaluated in simulation.

    \textit{Solution and future work:} We argue that this is in fact a deliberate choice rather than a pure shortcoming. Existing simulation benchmarks are known to suffer from non-trivial biases, including limited scenario diversity, non-reactive log-replay neighbors, and a notable gap between simulator metrics and on-road performance~\citep{dauner2023parting}. Instead, we validate HDP on a real vehicle across 6 urban scenarios and 200\,km of road testing, and observe that our key design choices consistently translate into substantial closed-loop improvements, leading to a planner that performs reliably in real-world driving. Nevertheless, we acknowledge that real-vehicle testing is costly and not easily reproducible by the broader community, and we believe that building more realistic, reactive, and diverse closed-loop driving benchmarks is an important and pressing direction for future research.

\end{itemize}

Overall, although certain limitations exist, our work provides the systematic and large-scale investigation of diffusion-based planning for real-world E2E AD. Through carefully controlled studies, we identify a set of clean and principled design choices---diffusion loss space, hybrid trajectory representation, data scaling, and the RL-hybrid post-training---that together unleash the potential of diffusion models as a scalable and deployable E2E AD planner. The 10x closed-loop improvement achieved on real vehicles, with only a light smoothness post-processing and without heavy rule-based engineering, demonstrates that diffusion models can serve as a strong foundation for the next generation of end-to-end autonomous driving systems.

\end{document}

%% file: code.tex
\definecolor{codeblue}{rgb}{0.25,0.5,0.5}
\definecolor{codekw}{rgb}{0.85, 0.18, 0.50}

\definecolor{codesign}{RGB}{0, 0, 255}
\definecolor{codefunc}{rgb}{0.85, 0.18, 0.50}

\lstdefinelanguage{PythonFuncColor}{
  language=Python,
  keywordstyle=\color{blue}\bfseries,
  commentstyle=\color{codeblue},  % for lines starting with "#"
  stringstyle=\color{orange},
  showstringspaces=false,
  basicstyle=\ttfamily\small,
  % Match function calls and color them
  literate=
    % functions with one arg
    {*}{{\color{codesign}* }}{1}
    {-}{{\color{codesign}- }}{1}
    % {=}{{\color{codesign}= }}{1}
    {+}{{\color{codesign}+ }}{1}
    % function call pattern (common names)
    {dataloader}{{\color{codefunc}dataloader}}{1}
    {sample_t_r}{{\color{codefunc}sample\_t\_r}}{1}
    {randn}{{\color{codefunc}randn}}{1}
    {randn_like}{{\color{codefunc}randn\_like}}{1}
    {jvp}{{\color{codefunc}jvp}}{1}
    {stopgrad}{{\color{codefunc}stopgrad}}{1}
    {metric}{{\color{codefunc}metric}}{1}
    {torch}{{\color{codefunc}torch}}{1}
    {cumsum}{{\color{codefunc}cumsum}}{1}
    {roll}{{\color{codefunc}roll}}{1}
    {detached\_integral}{{\color{codefunc}detached\_integral}}{1}
    {hybrid\_loss}{{\color{codefunc}hybrid\_loss}}{1}
}

\lstset{
  language=PythonFuncColor,
  backgroundcolor=\color{white},
  basicstyle=\fontsize{9pt}{9.9pt}\ttfamily\selectfont,
  columns=fullflexible,
  breaklines=true,
  captionpos=b,
}

\begin{algorithm}[h]
\caption{Hybrid Loss with Detach}
\label{alg:hybrid_loss}
\begin{lstlisting}
def detached_integral(v, W, dt):
    # v: velocity of future trajectory
    # W: gradient detach window size
    # dt: time interval
    wpt_sg = torch.cumsum(v.detach()) * dt
    shift_sg = torch.roll(wpt_sg, shifts=W)
    shift_sg[:W] = 0
    
    wpt = torch.cumsum(v) * dt
    shift = torch.roll(wpt, shifts=W)
    shift[:W] = 0

    return wpt + shift_sg - shift

def hybrid_loss(pred_v, gt_v, W, omega):
    # omega: loss balancing weight
    # pred_v: predicted future velocity
    # gt_v: ground truth future velocity
    l_v = (pred_v - gt_v) ** 2
    l_wpt = (detached_integral(pred_v, W) - torch.cumsum(gt_v)) ** 2
    return l_v + omega * l_wpt
    
\end{lstlisting}
\end{algorithm}

%% file: code_rl.tex
% --- Add to preamble, right after existing \lstdefinelanguage{PythonFuncColor}{...} ---
\lstdefinelanguage{PythonFuncColorRL}{
  language=Python,
  keywordstyle=\color{blue}\bfseries,
  commentstyle=\color{codeblue},
  stringstyle=\color{orange},
  showstringspaces=false,
  basicstyle=\ttfamily\small,
  literate=
    % operators
    {*}{{\color{codesign}* }}{1}
    {-}{{\color{codesign}- }}{1}
    {+}{{\color{codesign}+ }}{1}
    % functions defined in this algorithm
    {reward_weight}{{\color{codefunc}reward\_weight}}{1}
    {rl_hybrid_loss}{{\color{codefunc}rl\_hybrid\_loss}}{1}
    % cross-reference: function reused from Algorithm 1
    {hybrid_loss}{{\color{codefunc}hybrid\_loss}}{1}
    % torch namespace and tensor methods
    {torch}{{\color{codefunc}torch}}{1}
    % {mean}{{\color{codefunc}mean}}{1}
    % {std}{{\color{codefunc}std}}{1}
    {exp}{{\color{codefunc}exp}}{1}
    % {detach}{{\color{codefunc}detach}}{1}
    {float}{{\color{codefunc}float}}{1}
}

% --- Replace your existing Algorithm 2 block ---
\begin{algorithm}[h]
\caption{RL-Hybrid Loss}
\label{alg:rl_hybrid_loss}
\begin{lstlisting}[language=PythonFuncColorRL]

def rl_hybrid_loss(r, beta, pred_v, gt_v, W, omega):
    # r:    per-candidate reward
    # beta: temperature
    # pred_v, gt_v, W, omega: same as Algorithm 1

    r_n = (r - r.mean() / (r.std() + 1e-6)
    weight = torch.exp(beta * r_n).detach()

    return weight * hybrid_loss(pred_v, gt_v, W, omega)
\end{lstlisting}
\end{algorithm}